\pdfoutput=1
\documentclass{article} %

\usepackage{iclr2027_conference,times}

\usepackage{amsmath,amsfonts,bm}
 \usepackage[most]{tcolorbox}

\def\eqref#1{equation~\ref{#1}}

\def\1{\bm{1}}

\DeclareMathAlphabet{\mathsfit}{\encodingdefault}{\sfdefault}{m}{sl}
\SetMathAlphabet{\mathsfit}{bold}{\encodingdefault}{\sfdefault}{bx}{n}

\DeclareMathOperator*{\argmax}{arg\,max}
\DeclareMathOperator*{\argmin}{arg\,min}

\usepackage{hyperref}

\hypersetup{
  pdftitle={On the SPOT: Enhancing Speculative Decoding for Diffusion Language Models via Commit-Order Prediction},
  pdfauthor={Eliezer Shlomi, Eliron Rahimi, Avi Mendelson, Chaim Baskin},
  pdfsubject={Speculative decoding for diffusion language models: a 0.62M-parameter drafter predicts the commit order, and coupled-noise verification preserves the target distribution at any temperature},
  pdfkeywords={speculative decoding, diffusion language models, block diffusion, commit order, drafter, draft tree, coupled-noise verification, lossless acceleration, SDAR, LLaDA, Dream},
  pdflang={en-US}
}

\usepackage{url}
\usepackage{amsmath,amssymb,amsthm}
\usepackage{booktabs}
\usepackage{multirow}
\usepackage{graphicx}
\usepackage{xcolor}
\usepackage{colortbl}
\usepackage{pifont}
\graphicspath{{figures/}}
\usepackage{algorithm}
\usepackage{algpseudocode}
\usepackage{caption}
\usepackage{xcolor}

\definecolor{algcomment}{RGB}{72,100,135}

\usepackage{enumitem}
\newtheorem{proposition}{Proposition}

\usepackage[most]{tcolorbox}
\usepackage{xcolor}
\usepackage{enumitem}

\definecolor{takeawayblue}{RGB}{39,92,160}
\definecolor{takeawaybg}{RGB}{247,250,255}

\newtcolorbox{keytakeaway}{
    enhanced,
    breakable,
    colback=takeawaybg,
    colframe=takeawayblue,
    boxrule=0pt,
    leftrule=3pt,
    arc=2mm,
    left=3mm,
    right=3mm,
    top=2mm,
    bottom=2mm,
    before skip=7pt,
    after skip=7pt,
    title={Key Takeaway},
    coltitle=takeawayblue,
    fonttitle=\bfseries,
    attach title to upper={\quad},
}

\newtcolorbox{keytakeaways}{
    enhanced,
    breakable,
    colback=takeawaybg,
    colframe=takeawayblue,
    boxrule=0pt,
    leftrule=3pt,
    arc=2mm,
    left=3mm,
    right=3mm,
    top=2mm,
    bottom=2mm,
    before skip=7pt,
    after skip=7pt,
    title={Key Takeaways},
    coltitle=takeawayblue,
    fonttitle=\bfseries,
    attach title to upper={\quad},
}

\newcommand{\sys}{SPOT}

\newcommand{\sysB}{\textbf{SPOT}}
\newcommand{\sysfullB}{\emph{\textbf{S}peculating \textbf{P}ositions
\textbf{O}rder \textbf{T}rajectory}}
\newcommand{\sysSets}{SPOT-Sets}

\newcommand{\SZ}{}

\newcommand{\THRN}[1]{\csname THRN@#1\endcsname}  %
\newcommand{\SEEDN}[1]{\csname SEEDN@#1\endcsname}  %
\expandafter\def\csname THRN@main_sdar8b_both\endcsname{$n{=}350$ per task (HumanEval $164$)}
\expandafter\def\csname THRN@main_sdar4b\endcsname{$n{=}350$ per task (HumanEval $164$)}
\expandafter\def\csname THRN@main_sdar4b_t1\endcsname{$n{=}350$ per task (HumanEval $164$)}
\expandafter\def\csname THRN@main_sdar17\endcsname{$n{=}350$ per task (HumanEval $164$)}
\expandafter\def\csname THRN@main_sdar17_t1\endcsname{$n{=}350$ per task (HumanEval $164$)}
\expandafter\def\csname THRN@main_mini\endcsname{$n{=}100$ per task}
\expandafter\def\csname THRN@main_mini_t1\endcsname{$n{=}100$ per task}
\expandafter\def\csname THRN@transfer\endcsname{$n{=}350$ per task (HumanEval $164$)}
\expandafter\def\csname THRN@sets_a_sdar8b\endcsname{$n{=}350$ per task (HumanEval $164$)}
\expandafter\def\csname THRN@final_sets_ladder\endcsname{$n{=}350$ per task (HumanEval $164$)}
\expandafter\def\csname THRN@final_sets_maxcommit\endcsname{$n{=}350$ per task (HumanEval $164$)}
\expandafter\def\csname THRN@final_sets_transfer_full\endcsname{$n{=}350$ per task (HumanEval $164$)}
\expandafter\def\csname THRN@final_sets_transfer_full_d\endcsname{$n{=}350$ per task (HumanEval $164$)}
\expandafter\def\csname SEEDN@main_sdar8b_both\endcsname{Mean of seeds $0$--$2$}
\expandafter\def\csname SEEDN@tsweep\endcsname{Mean of seeds $0$--$2$ (each seed samples different text; sd of a throughput ratio over seeds at most $0.14$)}
\expandafter\def\csname SEEDN@sets_a_sdar8b\endcsname{Mean of seeds $0$--$2$}

\usetikzlibrary{arrows.meta}

\definecolor{nfecol}{HTML}{2F4F7F}
\definecolor{tpscol}{HTML}{2E7D4F}
\definecolor{oursrow}{RGB}{237,243,250}
\newcommand{\nfe}[1]{\textcolor{nfecol}{#1}}
\newcommand{\tps}[1]{\textcolor{tpscol}{#1}}
\newcommand{\cmark}{\ding{51}}
\newcommand{\xmark}{\ding{55}}

\iclrfinalcopy
\title{On the SPOT: Enhancing Speculative Decoding\\ for Diffusion Language Models\\ via Commit-Order Prediction}
\author{\hspace*{-\tabcolsep}Eliezer Shlomi$^{1,2*}$, Eliron Rahimi$^{1,2*}$, Avi Mendelson$^{1}$, Chaim Baskin$^{2}$\\[3pt]
\hspace*{-\tabcolsep}\normalfont\small $^{1}$Technion -- Israel Institute of Technology\\
\hspace*{-\tabcolsep}\normalfont\small $^{2}$INSIGHT Lab, School of Electrical and Computer Engineering, Ben-Gurion University of the Negev}

\begin{document}
\maketitle

\lhead{Preprint}
\begingroup\renewcommand{\thefootnote}{\fnsymbol{footnote}}
\footnotetext[1]{Equal contribution. Correspondence: \texttt{\{eliezer, elironrahimi\}@campus.technion.ac.il}.}
\endgroup

\begin{abstract}
Diffusion language models (DLMs) promise parallel text generation, but in practice face an unresolved speed-quality trade-off. Exact decoding fills in only one masked position per model pass, while faster multi-token decoding schemes are lossy. Speculative decoding is the standard lossless acceleration for autoregressive models. However, it does not transfer directly to DLMs, which must choose both what to generate and where. We show that this order is the missing piece:
the model's own predictions from the previous pass already give the right token at
over $90\%$ of positions, but not which position comes next. We introduce \sysfullB~(\sysB{}), a $0.62$M-parameter model that reads only the target's output distributions, predicts this order and proposes a tree of candidate continuations, which the target verifies in a single pass by state equality. Coupled-noise match, which shares Gumbel noise between drafter and target, keeps this verification exact at any temperature. This approach also inspired us to introduce \sysSets{} for dynamic threshold decoding, where a pass commits several positions: \sysSets{} drafts whole commit sets together with their tokens. Our comprehensive experiments show that \sys{} cuts target passes by up to $6.8\times$ and reaches up to
$3.4\times$ the throughput of standard decoding on SDAR-8B, matching or beating the calibrated and
training-free drafting baselines in every task and budget at $T{=}0$ and $T{=}1$.
In addition, we show that \sys{}'s learned drafter transfers across models, demonstrating successful transfer across six different DLMs; with a light calibration, the \sysSets{} drafter transfers as well. On top of dynamic threshold decoding, \sysSets{} improves the throughput of SDAR-8B by $59\%$ on average over the threshold decoder alone.

\par\smallskip
{\centering\hypersetup{pdfborder={0 0 0}}\href{https://eliezer318.github.io/SPOT-dLLM/}{\raisebox{-0.14em}{\includegraphics[height=0.95em]{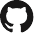}}\,\textcolor[RGB]{54,125,189}{Project Page}}\par}

\end{abstract}
\section{Introduction}
\label{sec:intro}

Diffusion language models (DLMs)~\citep{nie_large_2025,arriola_block_2025}
generate text by iteratively refining masked token sequences rather than predicting tokens left-to-right.
Block-diffusion models make this particularly practical by combining bidirectional denoising
within each block with causal generation across blocks~\citep{arriola_block_2025,cheng_sdar_2026}.
The diffusion paradigm enables new capabilities such as bidirectional attention, iterative refinement, enhanced safety robustness,
and the potential to substantially accelerate language generation via parallel decoding
~\citep{ye_dream_2025,rahimi_step-wise_2026, wu_fast-dllm_2026}.
Yet this parallelism comes with a fundamental speed--quality trade-off.
High-quality diffusion decoding typically uses conservative unmasking and
therefore requires many denoising steps, while faster decoding schemes reveal
multiple positions per step using confidence-based, top-$k$, or dynamic
policies, potentially changing the decoding trajectory and degrading generation
quality~\citep{xu_multi-token_2026,pan_blockspec_2025,sun_psd_2026}. Our main setting is exact
decoding with one commit per forward; we also address dynamic threshold
decoding~\citep{wu_fast-dllm_2026}, which commits every position whose confidence clears a bar.

\begin{figure}[t]
\centering
\includegraphics[width=\linewidth]{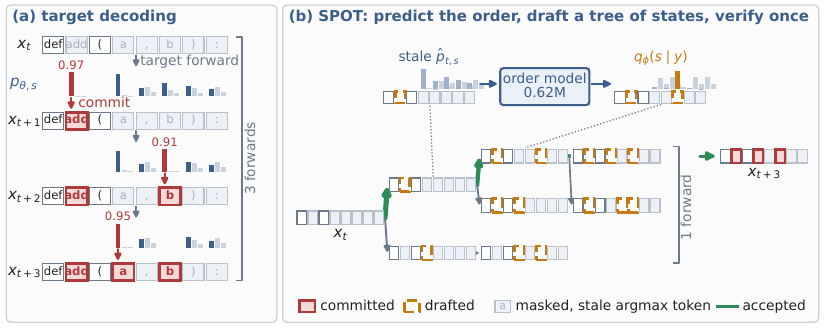}
\caption{\sys{} overview (red: committed; amber dashed: drafted; grey: masked, with its
stale argmax token). \textbf{(a)} The target commits its most confident position per
forward: the tokens are largely known in advance, the \emph{order} is not.
\textbf{(b)} \sys{} reads a hypothetical prefix's stale distributions, predicts the next
position with a $0.62$M order model, grows a tree of states by prefix mass (edge width),
and verifies all of them in one batched forward, committing three positions.}
\label{fig:method}
\end{figure}

For autoregressive language models (AR-LMs), speculative decoding provides an elegant lossless
solution.  A cheap drafter proposes several future tokens, and the target
verifies them in parallel while preserving exactly the distribution of
draft-free decoding~\citep{leviathan_fast_2023,chen_accelerating_2023}.
Recent work has begun bringing this idea to DLMs through self-drafting,
draft graphs, learned residual predictors, and auxiliary draft models
~\citep{gao_self_2025,agrawal_structuring_2026,pan_blockspec_2025,
cui_simsd_2026,xu_multi-token_2026}.  However, speculation changes
fundamentally once attention inside a block is bidirectional. In an autoregressive model, the next position is fixed, so the drafter only
needs to predict \emph{what token comes next}.  In a DLM, committing one
position changes the representations and output distributions of all remaining
masked positions in the block.  The target must therefore decide both
\emph{what token to write} and \emph{where to write it next}.  Its future is
not a fixed sequence of token positions, but a trajectory through partially
filled block states.  Several recent DLM speculative-decoding methods exploit
this state-based view by proposing hypothetical future states and checking
whether the target follows the same transitions
~\citep{gao_self_2025,agrawal_structuring_2026,pan_blockspec_2025,
xu_multi-token_2026}.  This makes verification possible, but leaves two central
questions unresolved:
\textbf{(Q1)} what information should a drafter predict to speculate several
transitions ahead?
and \textbf{(Q2)} how can a tree of speculative states be verified exactly
when the target samples rather than decodes greedily?

Our key finding is that the main difficulty is not token prediction, but
\emph{commit-order prediction}. Over the speculative horizons relevant for
acceleration, the target's stale distributions already provide strong token
hypotheses: their top prediction matches the eventual committed token at
$92$--$100\%$ of measured positions. Predicting which position commits next is
considerably harder. This suggests reusing the target's stale token predictions
while focusing the drafter on commit order, retaining additional high-probability
token hypotheses when confidence is low. Based on this observation, we introduce \sysfullB~(\sysB{}), a lightweight
$0.62$M-parameter drafter that predicts future commit order from the target's
stale output distributions (Figure~\ref{fig:method}). \sys{} is
prefix-conditional: after each hypothetical commitment, it predicts the next
position again and organizes these predictions into a tree of candidate block
states, prioritizing high-probability branches. The target then evaluates the
current and drafted states in a single batched forward and accepts the longest
prefix consistent with its trajectory. The second challenge is verifying speculative trees exactly when the target samples.
Under greedy decoding, a drafted transition is verified by comparing it with the
target's next state. Under sampling, independent draws may disagree even when
they come from the same distribution, causing state-equality verification to
reject valid proposals. We address this with \emph{coupled-noise verification}:
the drafter and target share per-position Gumbel randomness, making the target
transition deterministic conditioned on the shared noise. State equality can
therefore verify an entire speculative tree, while marginalizing over the noise
exactly recovers the target's sampling distribution. This enables exact
tree-based verification at arbitrary temperatures without the sequential
proposal interface required by classical rejection sampling. We draw
inspiration from \sys{} for dynamic threshold decoding, where a target step commits a set of
positions. There the drafter must also choose the tokens, since the stale predictions miss more of
the tokens of later steps ($83\%$ right one step ahead, $44\%$ three steps ahead). \sysSets{} still
starts from the stale distributions, but drafts whole commit sets together with their tokens, also
reading the target's hidden state and the candidates' token embeddings, and grows a tree whose
nodes are whole target steps.

The resulting method substantially reduces both model evaluations and
wall-clock decoding time.  On SDAR-8B, \sys{} uses $16$--$36\%$ fewer target
forwards than Spiffy's calibrated draft graph and achieves up to $1.4\times$
its throughput, with the adaptive-budget configuration reaching up to
$1.7\times$.  Relative to standard draft-free decoding, \sys{}
reduces target passes by up to $6.8\times$ and reaches up to $3.4\times$
throughput.  The gains hold across tasks at both greedy decoding and
$T{=}1$ sampling.  We further evaluate on SDAR-4B, SDAR-1.7B and LLaDA2.0-mini, and test
drafter transfer to two additional targets, LLaDA-8B and Dream-7B, covering
six DLMs in total. On top of the threshold decoder, \sysSets{}, with a drafter trained on general chat only, raises SDAR-8B's throughput to
$5.2$--$9.9\times$ that of standard decoding ($7.9\times$ on average), against $3.3$--$6.1\times$ for the
threshold decoder alone.

Our main contributions are:
\begin{itemize}
\item We analyze speculative decoding for diffusion targets along a drafting and a
verification axis, show that commit-order prediction is the drafting bottleneck while
the target's stale distributions already supply the tokens, show why independent
sampling breaks state-equality verification at $T>0$, and motivate coupling the drafter
and verifier through shared randomness.
\item Based on this analysis, we introduce \sys{}, a $0.62$M-parameter
prefix-conditional commit-order drafter that builds multi-step trees of block states
from stale distributions, paired with coupled-noise verification, which provably
preserves the target's distribution at any sampling temperature; inspired by \sys{},
\sysSets{} drafts the commit sets and their tokens for dynamic threshold decoding, with a tree
over whole target steps.
\item We evaluate \sys{} across tasks, temperatures, block widths and six diffusion
language models: it beats or matches calibrated and training-free baselines in every setting at $T{=}0$ and
$T{=}1$, cuts target passes by up to $6.8\times$ on SDAR-8B, and its checkpoints transfer between
models, including LLaDA-8B and Dream-7B, without retraining; on top of dynamic threshold decoding, \sysSets{} reaches
$7.9\times$ the throughput of standard decoding on SDAR-8B on average, against $5.0\times$ for the threshold decoder alone.
\end{itemize}
\section{Related Work}
\label{sec:related}

\paragraph{Speculative decoding and DLMs.}
Speculative decoding was introduced for autoregressive (AR) language models, where a
lightweight drafter proposes tokens that the target verifies in parallel while preserving its  distribution~\citep{leviathan_fast_2023,chen_accelerating_2023}. Later work improved the drafter with auxiliary decoding heads, learned drafters and dynamically built draft trees~\citep{cai_medusa_2024,li_eagle_2024,li_eagle-2_2024,li_eagle-3_2025}, and a complementary line uses diffusion models as drafters for AR targets~\citep{christopher_speculative_2025,li_diffuspec_2026,chen_dflash_2026,ringel_accelerating_2026}. Drafters accept more on the tasks they are trained
on~\citep{hong_training_2025,zbib_taps_2026} and lose speed under distribution shift~\citep{wang_when_2026}, so we
train our main \sysSets{} drafter on general chat, not on the train splits of our benchmarks. DLMs instead generate text by denoising~\citep{austin_structured_2021,nie_large_2025,ye_dream_2025}; block-causal
models such as SDAR make this practical block by
block~\citep{arriola_block_2025,cheng_sdar_2026}, while approximate caching and
parallel unmasking accelerate decoding at a cost in
fidelity~\citep{wu_fast-dllm_2026,sun_psd_2026}. Fast-dLLM's dynamic threshold commits every position whose confidence clears a bar; \sysSets{} (Section~\ref{sec:method-sets}) speculates on top of this rule; with exact acceptance it leaves the output unchanged. Because attention inside a block is
bidirectional, the target decides not only what to write but where, so the commit order is an extra dimension of decoding and speculative decoding does not transfer directly. For verification we build on communication-free coupling: \citet{daliri_coupling_2025} show that two parties sampling from $P$ and $Q$ under
shared Gumbel noise agree with probability at least $(1-D_{TV})/(1+D_{TV})$, and apply
this to \emph{drafter-invariant} speculative decoding. We carry the construction to diffusion decoding, where per-\emph{position} shared noise makes the target's sampled
commitment at any hypothetical block state a deterministic function of that state (Section~\ref{sec:method-verification}).
\paragraph{Speculative decoding for diffusion targets.}
SSD~\citep{gao_self_2025} and Spiffy~\citep{agrawal_structuring_2026} study speculative decoding with DLMs targets: SSD orders positions by confidence into a self-drafted chain and verifies it greedily, whereas Spiffy drafts states along an offline-calibrated graph and verifies them blockwise. Both are training-free and greedy. MRP~\citep{xu_multi-token_2026} learns a residual head that extrapolates the target's
logits a few denoising steps ahead and accepts drafted tokens in one pass, greedy only (Section~\ref{sec:exp-transfer} sets its reported gains beside ours). BlockSpec~\citep{pan_blockspec_2025} extends training-free self-speculation to trees on decoding trajectories, and SimSD~\citep{cui_simsd_2026} drafts with a same-family
sibling model but verifies approximately, so it is not lossless and is not added to our
comparison. In contrast, our drafter learns the commit order, builds multi-step trees,
verifies exactly under sampling, and transfers across targets without retraining; with \sysSets{} it also drafts the commit sets of threshold decoding.
\section{Analysis of Block-Diffusion Speculative Decoding}
\label{sec:observations}

\paragraph{Setting.}
The target decodes a block of $L$ masked positions one commit per forward (we write b32 for $L{=}32$). A draft is
a hypothetical block state with a few more positions filled by the drafter, and $D$
drafts are verified at once by appending each behind the cached prefix as its own copy
of the block, so a single forward over $(1{+}D)L$ positions returns the target's next
commit at the root and at every draft. A draft is accepted iff it equals that commit at
its parent, so the accepted set is one path and the output cannot change. A draft
therefore costs a whole block rather than a token: budgets are small, and target
forwards (NFE) and wall clock diverge (Appendix~\ref{app:layout}). Two measurements on
SDAR-8B shape the method: \emph{what} a drafter must predict, and \emph{how} its
proposals must be verified once the target samples.

\subsection{Draft-Trajectory Dynamics}
\label{sec:obs-oracle}
\begin{keytakeaway}
Knowing the future position order is almost as effective as knowing the full draft trajectory: the commit order carries most of the useful drafting information.
\end{keytakeaway}
A draft must predict which position the target commits next and what token it writes.
To isolate their value we evaluate \emph{live oracles} that peek at the target
continuation and propose (i) true \emph{positions} filled with stale-argmax tokens,
(ii) true \emph{tokens} ordered by stale confidence, or (iii) both; each bounds the
gain obtainable from one information source under standard verification.

\begin{figure}[t]
\centering
\includegraphics[width=0.98\linewidth]{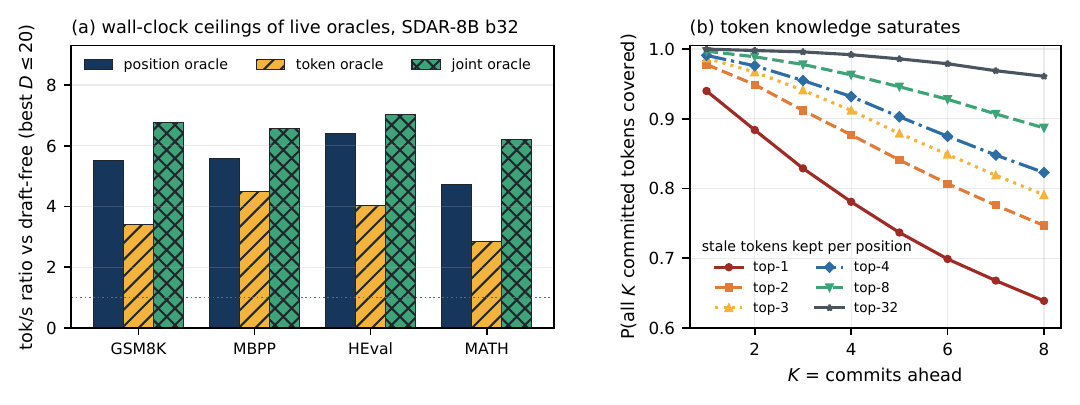}
\caption{\textbf{(a) Live oracle ceilings} on SDAR-8B (b32): best throughput ratio
over draft-free decoding for $D \le 20$; positions are worth more than tokens on every task, and a perfect
drafter reaches $6.2$--$7.0\times$ (full grid in Table~\ref{tab:oracle}).
\textbf{(b) Stale coverage}: probability that all committed tokens over the next $K$
steps lie within the stale top-$t$ candidates.}
\label{fig:oracle}
\end{figure}

Figure~\ref{fig:oracle}(a) and Table~\ref{tab:oracle} (Appendix~\ref{app:ext-oracle})
show that position knowledge dominates on every task and keeps paying across the
tested budgets (NFE reduction $5.4\times \to 7.2\times$ from $D8$ to $D20$ on GSM8K,
$5.7\times \to 8.2\times$ on MBPP), whereas true tokens yield only $3.4$--$6.5\times$
at $D20$. The stale distributions already carry token identity: the stale argmax
matches the true token in $92\%$ of first-block commits and $\ge 90\%$ thereafter,
and eight stale candidates per position cover $89\%$ of entire eight-commit
continuations (Figure~\ref{fig:oracle}(b)). They fail to predict order: the
probability that the previous round's top-confidence position commits next drops to
$3\%$ by block $9$ (Appendix~\ref{app:mechanism}). \emph{What} the model writes is
nearly known; \emph{which position} it writes next is not. This holds for single commits; when a
step commits several positions at once (dynamic threshold decoding), the stale tokens of later steps
are less reliable, which \sysSets{} addresses (Section~\ref{sec:method-sets}).

NFE does not map directly to wall clock, since every drafted round forwards
$(1{+}D)L$ positions, so we charge each oracle its measured verification cost
(lookahead subtracted). By learning the
missing order, \sys{}
predicts the next commit with $75\%$ accuracy on held-out trajectories
(Appendix~\ref{app:training}) against $22\%$ for the stale-confidence order. Commit-order prediction is
therefore the primary information bottleneck for drafting on this target, token
knowledge a complementary second axis, and verification cost a separate limit on
realized speed.

\subsection{Verification Dynamics Across Sampling Temperatures}
\label{sec:obs-temp}
\begin{keytakeaways}
\begin{itemize}[leftmargin=4mm, itemsep=0pt, topsep=1pt]
\item \textbf{Exactness.} Coupled and uncoupled matching preserve the target
sampling distribution, whereas greedy matching increasingly drifts as temperature rises.
\item \textbf{Acceptance.} Coupling becomes increasingly important at higher
temperatures, where independent matching rapidly loses acceptance.
\item \textbf{High-temperature regime.} Target quality deteriorates sharply beyond
$T{\approx}1.5$ and becomes the dominant limitation.
\end{itemize}
\end{keytakeaways}

At $T{=}0$ the acceptance test above is \emph{state equality} against a
deterministic commit. At $T{>}0$ the commitment is a random draw, and there are
three ways to keep verifying by equality: a \emph{greedy
match} accepts a drafted state iff it equals the target's \emph{argmax} commitment,
as the $T{=}0$ verifier would if simply run at temperature; an \emph{uncoupled
match} lets the target sample independently and accepts only if the sample happens
to equal the proposal; a \emph{coupled-noise match} shares the sampler's Gumbel
noise between drafter and verifier, so that the target's draw is again a
deterministic function of the state (Section~\ref{sec:method-verification}). We ran
all three on \sys{}, Spiffy and the BlockSpec-style tree at
$T \in \{0.5, 1, 1.5, 2, 2.5\}$ (Table~\ref{tab:tsweep}, Appendix~\ref{app:ext-ablations}).

\begin{table}[t]
\centering\small
\caption{Temperature sweep on SDAR-8B MBPP ($n{=}100$). Draft-free columns: target accuracy, argmax fraction, mean committed-token probability. \sys{} $D8$: NFE and throughput against the $T$-matched draft-free control. Greedy-match columns show distribution drift. \SEEDN{tsweep}. $^{\dagger}$Greedy match is not distribution-preserving, so its NFE is not comparable.}
\label{tab:tsweep}
\setlength{\tabcolsep}{3.5pt}
\begin{tabular}{lcccccccccc}
\toprule
& \multicolumn{3}{c}{draft-free ($D0$)} & \multicolumn{3}{c}{\sys{} $D8$ NFE (tok/s ratio)} & \multicolumn{2}{c}{greedy \sys{}} & \multicolumn{2}{c}{Spiffy $D8$ NFE} \\
\cmidrule(lr){2-4}\cmidrule(lr){5-7}\cmidrule(lr){8-9}\cmidrule(lr){10-11}
$T$ & acc & argmax & mean $p$ & coupled & uncoupled & greedy$^{\dagger}$ & acc & argmax & coupled & uncoupled \\
\midrule
0.5 & 0.36 & 0.998 & 0.994 & 56.4\,(3.70) & 56.4\,(3.72) & 56.0\,(3.67) & 0.35 & 0.999 & 83.2 & 83.0 \\
1.0 & 0.42 & 0.990 & 0.968 & 58.5\,(3.56) & 58.6\,(3.56) & 56.1\,(3.68) & 0.41 & 0.995 & 85.3 & 85.5 \\
1.5 & 0.11 & 0.893 & 0.841 & 71.2\,(2.98) & 79.1\,(2.69) & 65.0\,(3.24) & 0.15 & 0.925 & 92.8 & 106.1 \\
2.0 & 0.00 & 0.062 & 0.054 & 187.2\,(1.21) & 257.5\,(0.88) & 183.7\,(1.23) & 0.00 & 0.072 & 114.8 & 259.4 \\
2.5 & 0.00 & 0.002 & 0.000 & 198.9\,(1.14) & 267.0\,(0.85) & 198.9\,(1.14) & 0.00 & 0.002 & 109.9 & 267.0 \\
\bottomrule
\end{tabular}

\end{table}

\textbf{(i) Coupled and uncoupled match are exact; greedy match is not.} In both
exact regimes the emitted tokens track the draft-free statistics within sampling
noise: over all three drafters at $D8$ and $D20$ and $T \le 1.5$ (Table~\ref{tab:tsweep-full}), the emitted argmax
fraction stays within $0.006$ of the draft-free value, the mean committed-token
probability within $0.006$, and accuracy within $0.05$ ($n{=}100$). Greedy match drifts, and the drift grows with
$T$: at $T{=}1.5$ \sys{} emits accuracy $0.15$--$0.16$ where the model samples $0.11$ and picks
the argmax at $92.2$--$92.5\%$ of commits against the model's $89.4\%$, while Spiffy and the tree, which accept less, stay within $0.2$ points; the drift is
proportional to acceptance, so it makes the \emph{best} drafter the \emph{most}
biased. \textbf{(ii) Coupling keeps acceptance alive as the target's
samples move away from its argmax.} At $T \le 1$ the sampled token is the argmax
$99$--$99.8\%$ of the time and the two verifiers agree within $2\%$ in NFE; at
$T{=}1.5$ uncoupled costs $11$--$18\%$ more NFE for \sys{} and Spiffy ($71.2 \to
79.1$ and $92.8 \to 106.1$ at $D8$); at $T{=}2$ uncoupled accepts almost nothing ($257.5$ NFE
against $267.1$ draft-free) while coupled still drafts ($187.2$). The mechanism is the
coupling bound of \citet{daliri_coupling_2025}: a stale-informed proposal survives
with probability at least $(1-\delta)/(1+\delta)$ in the total variation $\delta$
between stale and fresh conditionals, whereas an independent sample equals the
argmax only with probability $p_{\max}$. \textbf{(iii) At high temperatures, target quality becomes the limiting factor.}
Draft-free accuracy already drops sharply to $0.11$ at $T{=}1.5$ and reaches
$0.00$ at $T \ge 2$, leaving little useful generation quality to preserve.

\section{Method}
\label{sec:method}

\subsection{Setup}
\label{sec:method-setup}

Let $p_\theta$ be a block-causal diffusion language model
\citep{cheng_sdar_2026, bie_llada20_2025}.  Within an active block of length $L$, a state
$x\in\mathcal{X}=(\mathcal{V}\cup\{\textsc{m}\})^L$ assigns each
position either a token in vocabulary $\mathcal{V}$ or the mask symbol
\textsc{m}, and
$\mathcal{M}(x)=\{s:x_s=\textsc{m}\}$ denotes its masked positions.
Given the completed causal context $C$, one target forward returns
$p_{\theta,s}(\cdot\mid x,C)$ for every $s\in\mathcal{M}(x)$. Let 
    $s^\star(x)
    =
    \argmax_{s\in\mathcal{M}(x)}
    \max_{v\in\mathcal{V}}
    p_{\theta,s}(v\mid x,C)$
be the highest-confidence masked position
and $x\oplus(s,v)$ be
the state obtained by filling position $s$ with token $v$.  At $T=0$,
the target commits
$v^\star_0(x)=\argmax_v
p_{\theta,s^\star(x)}(v\mid x,C)$, giving the deterministic transition
\begin{equation}
\label{eq:phi-greedy}
    \Phi_0(x)
    =
    x\oplus
    \bigl(s^\star(x),v^\star_0(x)\bigr).
\end{equation}
At $T>0$, the committed token is instead sampled from the
temperature-adjusted conditional
\begin{equation}
\label{eq:tempered-dist}
    \pi_{T,s}(v\mid x,C)
    =
    \frac{
        p_{\theta,s}(v\mid x,C)^{1/T}
    }{
        \sum_{u\in\mathcal{V}}
        p_{\theta,s}(u\mid x,C)^{1/T}
    }.
\end{equation}

A \emph{drafter} $\mathcal{D}$ proposes a finite tree of hypothetical
future states rooted at the current state, with each edge representing
one position--token commitment.  The target evaluates the root and all
drafted states in one batched forward.  The verifier accepts an edge
$(y,z)$ only when $z$ equals the target's own next state from $y$,
which we call \emph{state-equality verification}; because the target
has a unique next commitment, the accepted subset of a draft tree is
always a single path. Let $X_T^{\mathrm{DF}}$ and
$X_{T,\mathcal{D}}^{\mathrm{Spec}}$ denote the draft-free and
speculative committed trajectories.  We call a verifier
\emph{distribution preserving} when $ X_{T,\mathcal{D}}^{\mathrm{Spec}}
    \overset{d}{=}
    X_T^{\mathrm{DF}}.$

\subsection{The drafter: learned commit order}
\label{sec:method-drafter}

\paragraph{Order model.}
At speculative round $t$, let $\widehat p_{t,s}(\cdot)$ denote the
most recently available target conditional at masked position $s$.
These conditionals are \emph{stale}: hypothetical future commitments
change the true target distribution, but no additional target forward
is performed while drafting.  Our drafter is a three-layer transformer
encoder ($d_{\mathrm{model}}=128$, $8$ heads, $0.62$M parameters).
Each masked position is represented using its stale top-$32$
probabilities (raw and log), eight summary statistics, and two path
features encoding the positions already committed along the
hypothetical prefix and its depth.  Given a hypothetical state $y$,
the model predicts $q_\phi(s\mid y)$,
    where
    $s\in\mathcal M(y),$ a distribution over the target's next committed position.  The same
prefix-conditional predictor is learned in
Algorithm~\ref{alg:drafter-training} and recursively queried during
Algorithm~\ref{alg:drafter-inference}.  The drafter uses no target hidden states, token embeddings, or text, and can therefore be easily evaluated across different DLMs.

\paragraph{Training.}
We train $q_\phi$ entirely from draft-free target trajectories
(Algorithm~\ref{alg:drafter-training}).  Let
$(s_t^\star,s_{t+1}^\star,\ldots)$ denote the target commit sequence
following decoding state $x_t$.  For each prefix depth $d$, the first
$d$ commitments define a hypothetical state $y_{t,d}$ while the stale
conditionals $\widehat p_t$ remain fixed; the next target commitment
$s_{t+d}^\star$ is used as supervision.  This produces
prefix-conditioned examples
$(\widehat p_t,y_{t,d},s_{t+d}^\star)$ that match the states encountered
during recursive drafting.  We optimize $
    \mathcal L_{\mathrm{order}}(\phi)
    =
    -\mathbb E_{(\widehat p,y,s^\star)\sim\mathcal E}
    \log q_\phi(s^\star\mid y),
$ where $\mathcal E$ is the resulting training set.  Full data-generation
and optimization details are given in Appendix~\ref{app:training}.

\paragraph{Inference.}
At inference (Algorithm~\ref{alg:drafter-inference}), drafting starts
from the current target state $x_t$ and stale conditionals
$\widehat p_t$.  For every candidate prefix $y$, the order model scores
the remaining positions using $q_\phi(\cdot\mid y)$, and each position
is paired with token hypotheses from the stale target conditional.
Let $\mathcal H_t(s)$ denote the token-hypothesis set for position $s$.
At $T=0$, its primary hypothesis is $
    \widetilde v_0(s)
    =
    \argmax_v \widehat p_{t,s}(v),
$ with the runner-up additionally included when the stale top-$1$
probability is below $0.6$ (the token hedge).  At $T>0$, $\mathcal H_t(s)$ uses the
coupled stale draw of Eq.~\ref{eq:stale-gumbel}, introduced formally in
Section~\ref{sec:method-verification}.

For a candidate child $z=y\oplus(s,v)$, define its \emph{prefix mass}
recursively as
\begin{equation}
\label{eq:prefix-mass}
    m(z)
    =
    m(y)\,
    q_\phi(s\mid y)\,
    \widehat p_{t,s}(v),
    \qquad
    m(x_t)=1.
\end{equation}
Algorithm~\ref{alg:drafter-inference} uses DDTree, a level-batched
best-first search adapted from \citet{ringel_accelerating_2026}, to
allocate budget $D$ to the highest-mass candidate states.  Since every
factor in Eq.~\ref{eq:prefix-mass} lies in $[0,1]$, a child's mass
cannot exceed its parent's, allowing selected nodes to be expanded
level-by-level while recovering the same top-$D$ fixpoint as
sequential best-first search.  The final model applies no calibration temperature to
$q_\phi$ before the search.

\paragraph{Adaptive budget and padding.}
Algorithm~\ref{alg:drafter-inference} treats $D$ as a hard upper bound.
We additionally apply a prefix-mass floor $\gamma$, discarding
candidates with $m(z)<\gamma$ and thereby adapting the realized
verification budget to drafter confidence.  Independently, the target
batch may be padded with dummy replicas to maintain a fixed $(1+D)L$
shape and reduce batch-shape-dependent numerical variation.  We report the unpadded $\gamma=0$ configuration as \sys{} and the unpadded
configuration with a floor calibrated once per target ($\gamma=0.02$ on SDAR-8B;
Appendix~\ref{app:gamma-calib}) as \sys{}-$\gamma$. The complete budget,
mass-floor, and padding ablations are given in
Section~\ref{sec:exp-ablation}.

\subsection{Diffusion-block verification by coupled-noise match}
\label{sec:method-verification}

\paragraph{Greedy verification.}
At $T=0$, the target transition $\Phi_0$ in
Eq.~\ref{eq:phi-greedy} is deterministic. Therefore, a drafted edge $(y,z)$ is
 accepted exactly when $z=\Phi_0(y)$. Proposition~\ref{prop:greedy} is elementary and included both for completeness and as a basis for the proof of Proposition~\ref{prop:match}.

\begin{proposition}[Greedy exactness]
\label{prop:greedy}
For any drafter $\mathcal{D}$ and any finite chain or tree draft,
state-equality verification at $T=0$ produces exactly the draft-free
target trajectory:     $X^{\mathrm{Spec}}_{0,\mathcal{D}}
    =
    X^{\mathrm{DF}}_0 .$

\end{proposition}

\paragraph{Sampling with shared Gumbel noise.}
At $T>0$, state equality cannot use independent drafter and verifier
samples: even identical distributions would then produce different
tokens with nonzero probability.  Instead, we couple their sampling
through shared randomness.  For every absolute output position $a$ and
vocabulary item $v$, draw $G_{a,v} \overset{\mathrm{iid}}{\sim} \mathrm{Gumbel}(0,1),$ and let $G=\{G_{a,v}\}_{a,v}$ denote the resulting table.  The table is
shared by drafter and verifier and indexed by absolute position, so the
same position uses the same Gumbel row in both; $a(s)$ denotes the absolute output position of block position $s$.

For target state $x$, the target draw is
    $v^\star_G(x)
    =
    \argmax_{v\in\mathcal{V}}
    \left[
        \frac{1}{T}
        \log p_{\theta,s^\star(x)}(v\mid x,C)
        +
        G_{a(s^\star(x)),v}
    \right]$, 
which defines the transition $\Phi_G(x)
    =
    x\oplus
    \bigl(
        s^\star(x),v^\star_G(x)
    \bigr).$ By the Gumbel--argmax identity,
$v^\star_G(x)\sim
\pi_{T,s^\star(x)}(\cdot\mid x,C)$. The drafter applies the \emph{same} noise to its stale conditional:
\begin{equation}
\label{eq:stale-gumbel}
    \widetilde v_G(s)
    =
    \argmax_{v\in\mathcal{V}}
    \left[
        \frac{1}{T}
        \log \widehat p_{t,s}(v)
        +
        G_{a(s),v}
    \right].
\end{equation}
Thus, whenever the stale and fresh conditionals have the same
Gumbel-perturbed argmax, drafter and target propose exactly the same
token.  Conditioned on $G$, $\Phi_G$ is deterministic, so verification
again reduces to state equality: $z=\Phi_G(y).$

\begin{proposition}[Coupled-noise exactness]
\label{prop:match}
For any drafter $\mathcal D$ and $T>0$, the Gumbel-driven draft-free
trajectory satisfies
$X^{\mathrm{DF}}_T(G)\overset{d}{=}X^{\mathrm{DF}}_T$, and for every
realization $g$ of $G$,
$X^{\mathrm{Spec}}_{T,\mathcal D}(g)=X^{\mathrm{DF}}_T(g)$.
Consequently,
$X^{\mathrm{Spec}}_{T,\mathcal D}\overset{d}{=}X^{\mathrm{DF}}_T$.
\end{proposition}

Propositions~\ref{prop:greedy} and~\ref{prop:match} therefore separate
\emph{correctness} from \emph{draft quality}: the verifier always
preserves the target trajectory or distribution, while the drafter
determines only how many future target transitions are already present
in the tree and can be accepted in one target forward.  The complete
formal setup and proofs are given in Appendix~\ref{app:proofs}. For a drafted position $s$, let $\widehat\pi_{T,s}$ and $\pi_{T,s}$
denote the stale and fresh temperature-adjusted conditionals, and $v_G(s)$ the target's own draw at $s$ with
the same row $G_{a(s),\cdot}$.  The
shared-Gumbel coupling of \citet{daliri_coupling_2025} satisfies $\Pr[
        \widetilde v_G(s)=v_G(s)
    ]
    \ge
    \frac{1-\delta_T}{1+\delta_T},$
    where
    $\delta_T
    =
    D_{\mathrm{TV}}
    \!\left(
        \widehat\pi_{T,s},
        \pi_{T,s}
    \right).$ Hence coupled verification becomes more effective as the stale and
fresh sampling distributions become closer.

\paragraph{Relation to rejection sampling.}
When the drafter exposes an explicit sequential proposal distribution,
exact verification can alternatively be performed by rejection
sampling (RS).  Appendix~\ref{app:rs} formalizes this alternative, proves
its exactness and maximal per-link acceptance probability
$1-D_{\mathrm{TV}}(\pi_k,q_k)$, and validates its implementation.
Classical per-link RS assumes a sequential proposal interface, whereas
coupled-noise match requires only state equality and therefore applies
directly to our batched draft trees.

\subsection{\sysSets{}: drafting commit sets for threshold decoding}
\label{sec:method-sets}
Faster decoders commit, in each forward, every masked position whose confidence
$c_s(x)=\max_v p_{\theta,s}(v\mid x,C)$ clears a bar $\tau$, or the most confident one if none
does~\citep{wu_fast-dllm_2026}. Given the shared noise $G$, this transition $\Phi_\tau$ is still deterministic, so
state-equality verification and Propositions~\ref{prop:greedy} and~\ref{prop:match} carry over to it. Inspired by
\sys{}, \sysSets{} drafts its future steps, each a whole set of positions whose tokens the stale argmax often misses
($83/60/44\%$ right one, two and three steps ahead; Appendix~\ref{app:sets-drafter}). Its set drafter ($5.1$M
parameters) therefore also reads the embeddings of the $16$ most probable stale candidates and the target's last hidden
state, and predicts for each position whether it clears $\tau$ at the next step and which candidate it takes. Tree
nodes are whole steps, searched with the same prefix mass, floor $\gamma$ and cap $D$ (\sysSets{}-$\gamma$); a draft
that completes a block also carries the next one (cross-block drafting), and the next block's cache-refresh forward
verifies drafts too (draft refresh; Figure~\ref{fig:sets-round}).

\paragraph{Training data.} The drafter learns from the target's own threshold-decoded trajectories ($\tau{=}0.9$,
$T{=}0$) on general chat prompts only, about $1$M examples in all (Appendix~\ref{app:sets-drafter}). The prompts are
the first user turn of ShareGPT and UltraChat-200K~\citep{ding_enhancing_2023} conversations, half from each, the
datasets EAGLE-3 trains on~\citep{li_eagle-3_2025}. Keeping the benchmarks' own data out matters, since drafters accept
more on the tasks they are trained on and lose speed under distribution shift~\citep{hong_training_2025,zbib_taps_2026,wang_when_2026};
Appendix~\ref{app:sets-data} compares a drafter trained on task prompts instead.

\paragraph{Max-commit.} Max-commit relaxes state equality and follows the accepted path with the most commits
(Figure~\ref{fig:sets-maxcommit}). Intuitively, its cut rule can only reduce the loss of threshold decoding, as it
makes decoding stricter: it accepts a drafted step only if it is the set a higher bar $\tau'\ge\tau$ would commit, a
step that a stricter threshold decoder would also take. Our main rows use the looser product rule, which accepts a step
whose tokens are the target's top choices and whose confidences have a geometric mean of at least $\tau$; it commits
more per forward but no longer reproduces the threshold decoder's output, so we also report the exact variant
(Appendix~\ref{app:sets-maxcommit}).

\begin{figure}[t]
\centering
\includegraphics[width=\linewidth]{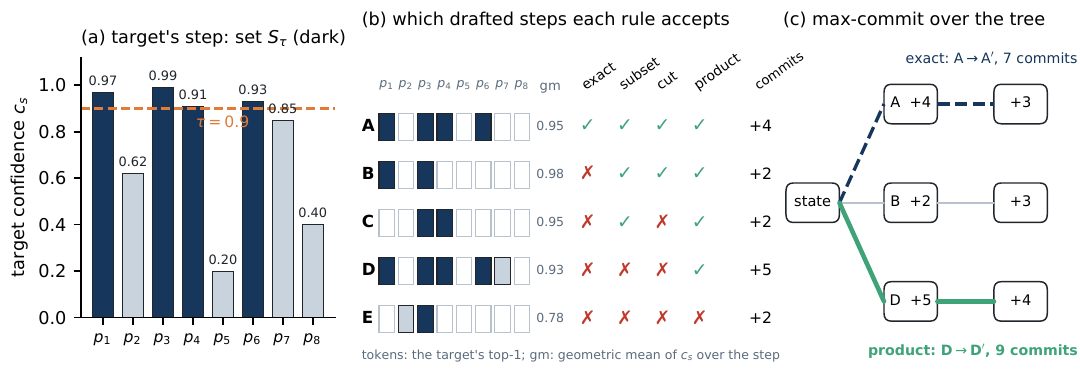}
\caption{Max-commit on a toy block (illustrative numbers). (a) The threshold decoder commits $S_\tau$, the
positions above $\tau$. (b) Exact acceptance takes only the drafted step equal to $S_\tau$ (A); subset and cut also
take parts of it, and product any step whose confidences have a geometric mean of at least $\tau$. (c) Max-commit
follows the accepted path with the most commits: $7$ under exact, $9$ under product, whose output is no longer the
threshold decoder's own.}
\label{fig:sets-maxcommit}
\end{figure}

\section{Experiments}
\label{sec:experiments}
We establish that \sys{} takes fewer target forwards with higher throughput than every baseline at $T{=}0$ and $T{=}1$ (Sections~\ref{sec:exp-main} and~\ref{sec:exp-temp}), and so does \sysSets{} on top of dynamic threshold decoding (Section~\ref{sec:exp-main}), then that a checkpoint transfers across targets without additional training (Section~\ref{sec:exp-transfer}), and finally that each component of the drafter earns its place (Section~\ref{sec:exp-ablation}).
\subsection{Experimental setup}
\label{sec:exp-setup}
Our headline target is \textbf{SDAR-8B}~\citep{cheng_sdar_2026}, a dense block-causal
diffusion LM at block width $L{=}32$ with its exact prefix KV cache. The same recipe
gives drafters for \textbf{SDAR-4B} and \textbf{SDAR-1.7B} and for the mixture-of-experts
\textbf{LLaDA2.0-mini}~\citep{bie_llada20_2025}, and transfer is measured on two further
targets, \textbf{LLaDA-8B}~\citep{nie_large_2025} and
\textbf{Dream-7B}~\citep{ye_dream_2025}: six diffusion language models in all. We evaluate
GSM8K~\citep{cobbe_training_2021}, MBPP~\citep{austin_program_2021},
HumanEval~\citep{chen_evaluating_2021} and MATH500~\citep{hendrycks_measuring_2021} on
$n{=}350$ samples (all $164$ for HumanEval) with seed $0$; Tables~\ref{tab:tsweep} and~\ref{tab:main-sdar} give the mean of seeds $0$--$2$, except where a caption says otherwise, at generation length $256$ and budgets $D\in\{6,8,12,20\}$ and report NFE per sample, accuracy, and a \emph{throughput ratio} against the draft-free run of the same task, model and temperature. \sys{} is one checkpoint and one calibration per target, and \sys{}-$\gamma$ runs it unpadded with that floor. Baselines are Spiffy's calibrated draft graph~\citep{agrawal_structuring_2026}, an SSD-style confidence-ordered free chain~\citep{gao_self_2025} and a BlockSpec-style confidence
tree~\citep{pan_blockspec_2025}, all sharing our decoder and verifier so that the comparison is about drafting alone. Other block widths, temperatures, and generation lengths, and the full protocol, are in Appendices~\ref{app:setup}, \ref{app:dataset}
and~\ref{app:ext}. Every drafted run is unpadded. Throughput is measured on two machines, each with eight NVIDIA RTX PRO 6000 Blackwell cards, one job per card; every ratio is over the draft-free run of the same machine.

\subsection{Main results}
\label{sec:exp-main}

\begin{table}[t]
\centering\footnotesize
\caption{Main results, SDAR-8B b32 at $T{=}0$ and $T{=}1$: \nfe{NFE per sample} (\tps{throughput ratio} over the draft-free run of the same temperature), best in bold, second best underlined. Shaded rows give each block's accuracy range. \sys{}-$\gamma$: unpadded, $\gamma{=}0.02$. At $T{=}1$ all methods use coupled-noise match except the free chain (RS). \SEEDN{main_sdar8b_both}; $n{=}350$ per task (HumanEval $164$). Last two blocks: on top of the threshold decoder ($\tau{=}0.9$, Section~\ref{sec:method-sets}), ratios still over standard decoding.}
\label{tab:main-sdar}\label{tab:main-t1}
\setlength{\tabcolsep}{4pt}\scriptsize
\resizebox{\linewidth}{!}{\begin{tabular}{lcccccccc}
\toprule
\nfe{NFE}$\downarrow$ (\tps{tok/s}$\uparrow$) & \multicolumn{2}{c}{GSM8K} & \multicolumn{2}{c}{MBPP} & \multicolumn{2}{c}{HumanEval} & \multicolumn{2}{c}{MATH500} \\
\cmidrule(lr){2-3}\cmidrule(lr){4-5}\cmidrule(lr){6-7}\cmidrule(lr){8-9}
method & $D8$ & $D20$ & $D8$ & $D20$ & $D8$ & $D20$ & $D8$ & $D20$ \\
\midrule
\rowcolor[gray]{0.90}$T{=}0$ & \multicolumn{2}{c}{acc.\ 0.89--0.90} & \multicolumn{2}{c}{acc.\ 0.43} & \multicolumn{2}{c}{acc.\ 0.71--0.74} & \multicolumn{2}{c}{acc.\ 0.10--0.12} \\
no draft ($D0$) & \multicolumn{2}{c}{\nfe{270.8}$\downarrow$ (\tps{1.00}$\uparrow$)\SZ} & \multicolumn{2}{c}{\nfe{269.2}$\downarrow$ (\tps{1.00}$\uparrow$)\SZ} & \multicolumn{2}{c}{\nfe{271.4}$\downarrow$ (\tps{1.00}$\uparrow$)\SZ} & \multicolumn{2}{c}{\nfe{271.8}$\downarrow$ (\tps{1.00}$\uparrow$)\SZ} \\
\midrule
\rowcolor{oursrow}\textbf{\sys{} (ours)} & \nfe{\textbf{70.1}}\,(\tps{\underline{2.95}})\SZ & \nfe{\textbf{52.4}}\,(\tps{\underline{2.37}})\SZ & \nfe{\textbf{56.7}}\,(\tps{\underline{3.38}})\SZ & \nfe{\textbf{39.5}}\,(\tps{\underline{2.49}})\SZ & \nfe{\textbf{62.5}}\,(\tps{\underline{3.35}})\SZ & \nfe{\textbf{44.8}}\,(\tps{\underline{2.74}})\SZ & \nfe{\textbf{84.6}}\,(\tps{\underline{2.50}})\SZ & \nfe{\textbf{66.0}}\,(\tps{\underline{1.86}})\SZ \\
\rowcolor{oursrow}\sys{}-$\gamma$ (ours) & \nfe{\underline{70.1}}\,(\tps{\textbf{3.03}})\SZ & \nfe{\underline{54.7}}\,(\tps{\textbf{2.88}})\SZ & \nfe{\underline{56.9}}\,(\tps{\textbf{3.48}})\SZ & \nfe{\underline{41.8}}\,(\tps{\textbf{3.36}})\SZ & \nfe{\underline{62.6}}\,(\tps{\textbf{3.42}})\SZ & \nfe{\underline{46.7}}\,(\tps{\textbf{3.36}})\SZ & \nfe{\underline{84.7}}\,(\tps{\textbf{2.53}})\SZ & \nfe{\underline{69.1}}\,(\tps{\textbf{2.26}})\SZ \\
Spiffy graph & \nfe{92.5}\,(\tps{2.48})\SZ & \nfe{71.6}\,(\tps{2.12})\SZ & \nfe{83.6}\,(\tps{2.40})\SZ & \nfe{61.2}\,(\tps{1.97})\SZ & \nfe{87.6}\,(\tps{2.64})\SZ & \nfe{65.5}\,(\tps{2.28})\SZ & \nfe{100.6}\,(\tps{2.29})\SZ & \nfe{82.5}\,(\tps{1.80})\SZ \\
SSD free chain & \nfe{92.4}\,(\tps{2.44})\SZ & \nfe{84.7}\,(\tps{1.60})\SZ & \nfe{84.5}\,(\tps{2.39})\SZ & \nfe{75.8}\,(\tps{1.41})\SZ & \nfe{87.0}\,(\tps{2.63})\SZ & \nfe{80.0}\,(\tps{1.69})\SZ & \nfe{100.6}\,(\tps{2.28})\SZ & \nfe{94.0}\,(\tps{1.41})\SZ \\
BlockSpec tree & \nfe{136.5}\,(\tps{1.64})\SZ & \nfe{114.6}\,(\tps{1.10})\SZ & \nfe{135.1}\,(\tps{1.46})\SZ & \nfe{111.5}\,(\tps{0.86})\SZ & \nfe{136.0}\,(\tps{1.65})\SZ & \nfe{112.7}\,(\tps{1.10})\SZ & \nfe{136.8}\,(\tps{1.64})\SZ & \nfe{116.3}\,(\tps{1.09})\SZ \\
\specialrule{0.6pt}{2.5pt}{0pt}
\rowcolor[gray]{0.90}$T{=}1$ & \multicolumn{2}{c}{acc.\ 0.88--0.90} & \multicolumn{2}{c}{acc.\ 0.33--0.38} & \multicolumn{2}{c}{acc.\ 0.67--0.69} & \multicolumn{2}{c}{acc.\ 0.10--0.11} \\
\midrule
\rowcolor{oursrow}\textbf{\sys{} (ours)} & \nfe{\textbf{72.6}}\,(\tps{\underline{2.87}})\SZ & \nfe{\textbf{54.8}}\,(\tps{\underline{2.25}})\SZ & \nfe{\textbf{60.1}}\,(\tps{\underline{3.12}})\SZ & \nfe{\textbf{42.5}}\,(\tps{\underline{2.27}})\SZ & \nfe{\textbf{65.2}}\,(\tps{\underline{3.20}})\SZ & \nfe{\textbf{47.4}}\,(\tps{\underline{2.57}})\SZ & \nfe{\textbf{87.6}}\,(\tps{\underline{2.38}})\SZ & \nfe{\textbf{68.8}}\,(\tps{\underline{1.83}})\SZ \\
\rowcolor{oursrow}\sys{}-$\gamma$ (ours) & \nfe{\underline{73.0}}\,(\tps{\textbf{2.93}})\SZ & \nfe{\underline{57.6}}\,(\tps{\textbf{2.77}})\SZ & \nfe{\underline{61.0}}\,(\tps{\textbf{3.27}})\SZ & \nfe{\underline{45.8}}\,(\tps{\textbf{3.21}})\SZ & \nfe{\underline{65.7}}\,(\tps{\textbf{3.24}})\SZ & \nfe{\underline{49.5}}\,(\tps{\textbf{3.23}})\SZ & \nfe{\underline{88.3}}\,(\tps{\textbf{2.42}})\SZ & \nfe{\underline{72.6}}\,(\tps{\textbf{2.20}})\SZ \\
Spiffy graph & \nfe{95.3}\,(\tps{2.42})\SZ & \nfe{75.4}\,(\tps{1.97})\SZ & \nfe{86.5}\,(\tps{2.28})\SZ & \nfe{64.9}\,(\tps{1.87})\SZ & \nfe{89.9}\,(\tps{2.59})\SZ & \nfe{68.7}\,(\tps{2.23})\SZ & \nfe{104.5}\,(\tps{2.21})\SZ & \nfe{87.5}\,(\tps{1.69})\SZ \\
free chain, RS & \nfe{89.9}\,(\tps{1.87})\SZ & \nfe{84.0}\,(\tps{1.46})\SZ & \nfe{64.5}\,(\tps{2.59})\SZ & \nfe{52.3}\,(\tps{2.09})\SZ & \nfe{79.2}\,(\tps{2.18})\SZ & \nfe{70.3}\,(\tps{1.74})\SZ & \nfe{116.8}\,(\tps{1.40})\SZ & \nfe{115.8}\,(\tps{1.02})\SZ \\
BlockSpec tree & \nfe{139.3}\,(\tps{1.61})\SZ & \nfe{118.2}\,(\tps{1.08})\SZ & \nfe{138.2}\,(\tps{1.41})\SZ & \nfe{115.1}\,(\tps{0.83})\SZ & \nfe{138.3}\,(\tps{1.63})\SZ & \nfe{115.5}\,(\tps{1.06})\SZ & \nfe{140.5}\,(\tps{1.59})\SZ & \nfe{121.2}\,(\tps{1.05})\SZ \\
\specialrule{0.6pt}{2.5pt}{0pt}
\multicolumn{9}{c}{\emph{threshold decoding, $\tau{=}0.9$ ($n{=}350$, HumanEval $164$)}} \\
\rowcolor[gray]{0.90}$T{=}0$ & \multicolumn{2}{c}{acc.\ 0.88--0.89} & \multicolumn{2}{c}{acc.\ 0.39--0.41} & \multicolumn{2}{c}{acc.\ 0.67--0.71} & \multicolumn{2}{c}{acc.\ 0.10--0.12} \\
no draft & \multicolumn{2}{c}{\nfe{58.8}$\downarrow$ (\tps{4.51}$\uparrow$)\SZ} & \multicolumn{2}{c}{\nfe{44.0}$\downarrow$ (\tps{5.89}$\uparrow$)\SZ} & \multicolumn{2}{c}{\nfe{43.4}$\downarrow$ (\tps{6.13}$\uparrow$)\SZ} & \multicolumn{2}{c}{\nfe{80.1}$\downarrow$ (\tps{3.34}$\uparrow$)\SZ} \\
\midrule
\rowcolor{oursrow}\textbf{\sysSets{}-$\gamma$ (ours)} & \nfe{\textbf{30.6}}\,(\tps{\textbf{7.30}})\SZ & \nfe{\textbf{30.2}}\,(\tps{\textbf{7.26}})\SZ & \nfe{\textbf{22.8}}\,(\tps{\textbf{9.28}})\SZ & \nfe{\textbf{22.4}}\,(\tps{\textbf{9.15}})\SZ & \nfe{\textbf{22.7}}\,(\tps{\textbf{9.85}})\SZ & \nfe{\textbf{22.6}}\,(\tps{\textbf{9.62}})\SZ & \nfe{\textbf{42.3}}\,(\tps{\textbf{5.24}})\SZ & \nfe{\textbf{41.8}}\,(\tps{\textbf{5.22}})\SZ \\
\rowcolor{oursrow}\sysSets{}-$\gamma$, exact (ours) & \nfe{\underline{32.5}}\,(\tps{\underline{6.93}})\SZ & \nfe{\underline{32.2}}\,(\tps{\underline{6.86}})\SZ & \nfe{\underline{23.8}}\,(\tps{\underline{8.97}})\SZ & \nfe{\underline{23.4}}\,(\tps{\underline{8.72}})\SZ & \nfe{\underline{24.0}}\,(\tps{\underline{9.35}})\SZ & \nfe{\underline{23.7}}\,(\tps{\underline{9.22}})\SZ & \nfe{\underline{45.1}}\,(\tps{\underline{4.96}})\SZ & \nfe{\underline{44.5}}\,(\tps{\underline{4.92}})\SZ \\
Spiffy graph & \nfe{38.6}\,(\tps{6.06})\SZ & \nfe{37.5}\,(\tps{5.53})\SZ & \nfe{26.6}\,(\tps{7.97})\SZ & \nfe{26.3}\,(\tps{6.85})\SZ & \nfe{29.0}\,(\tps{8.12})\SZ & \nfe{28.6}\,(\tps{7.37})\SZ & \nfe{52.3}\,(\tps{4.47})\SZ & \nfe{51.5}\,(\tps{4.04})\SZ \\
SSD free chain & \nfe{38.5}\,(\tps{6.03})\SZ & \nfe{37.4}\,(\tps{4.13})\SZ & \nfe{26.6}\,(\tps{7.92})\SZ & \nfe{26.0}\,(\tps{4.95})\SZ & \nfe{28.9}\,(\tps{8.18})\SZ & \nfe{28.4}\,(\tps{5.54})\SZ & \nfe{52.3}\,(\tps{4.44})\SZ & \nfe{51.0}\,(\tps{2.87})\SZ \\
BlockSpec tree & \nfe{41.8}\,(\tps{5.49})\SZ & \nfe{39.0}\,(\tps{3.76})\SZ & \nfe{30.0}\,(\tps{6.97})\SZ & \nfe{27.7}\,(\tps{4.46})\SZ & \nfe{31.3}\,(\tps{7.39})\SZ & \nfe{29.4}\,(\tps{5.12})\SZ & \nfe{56.2}\,(\tps{4.06})\SZ & \nfe{52.8}\,(\tps{2.66})\SZ \\
\specialrule{0.6pt}{2.5pt}{0pt}
\rowcolor[gray]{0.90}$T{=}1$ & \multicolumn{2}{c}{acc.\ 0.87--0.89} & \multicolumn{2}{c}{acc.\ 0.33--0.35} & \multicolumn{2}{c}{acc.\ 0.65--0.66} & \multicolumn{2}{c}{acc.\ 0.09--0.11} \\
no draft & \multicolumn{2}{c}{\nfe{61.5}$\downarrow$ (\tps{4.33}$\uparrow$)\SZ} & \multicolumn{2}{c}{\nfe{46.9}$\downarrow$ (\tps{5.54}$\uparrow$)\SZ} & \multicolumn{2}{c}{\nfe{44.4}$\downarrow$ (\tps{6.07}$\uparrow$)\SZ} & \multicolumn{2}{c}{\nfe{81.1}$\downarrow$ (\tps{3.31}$\uparrow$)\SZ} \\
\midrule
\rowcolor{oursrow}\textbf{\sysSets{}-$\gamma$ (ours)} & \nfe{\textbf{34.0}}\,(\tps{\textbf{6.44}})\SZ & \nfe{\textbf{33.5}}\,(\tps{\textbf{6.41}})\SZ & \nfe{\textbf{27.5}}\,(\tps{\textbf{7.53}})\SZ & \nfe{\textbf{27.2}}\,(\tps{\textbf{7.25}})\SZ & \nfe{\textbf{26.2}}\,(\tps{\textbf{8.37}})\SZ & \nfe{\textbf{25.9}}\,(\tps{\textbf{8.29}})\SZ & \nfe{\textbf{45.8}}\,(\tps{\textbf{4.74}})\SZ & \nfe{\textbf{45.0}}\,(\tps{\textbf{4.74}})\SZ \\
\rowcolor{oursrow}\sysSets{}-$\gamma$, exact (ours) & \nfe{\underline{36.0}}\,(\tps{\underline{6.21}})\SZ & \nfe{\underline{35.4}}\,(\tps{\underline{6.21}})\SZ & \nfe{\underline{28.5}}\,(\tps{\underline{7.40}})\SZ & \nfe{\underline{28.1}}\,(\tps{\underline{7.14}})\SZ & \nfe{\underline{27.2}}\,(\tps{\underline{8.30}})\SZ & \nfe{\underline{26.8}}\,(\tps{\underline{8.17}})\SZ & \nfe{\underline{48.6}}\,(\tps{\underline{4.58}})\SZ & \nfe{\underline{47.7}}\,(\tps{\underline{4.57}})\SZ \\
Spiffy graph & \nfe{42.7}\,(\tps{5.47})\SZ & \nfe{42.3}\,(\tps{4.99})\SZ & \nfe{30.7}\,(\tps{6.88})\SZ & \nfe{30.4}\,(\tps{5.76})\SZ & \nfe{31.7}\,(\tps{7.45})\SZ & \nfe{31.3}\,(\tps{6.72})\SZ & \nfe{57.2}\,(\tps{4.08})\SZ & \nfe{56.6}\,(\tps{3.65})\SZ \\
SSD free chain & \nfe{42.7}\,(\tps{5.44})\SZ & \nfe{42.1}\,(\tps{3.61})\SZ & \nfe{30.6}\,(\tps{6.84})\SZ & \nfe{30.1}\,(\tps{4.16})\SZ & \nfe{31.7}\,(\tps{7.46})\SZ & \nfe{31.6}\,(\tps{5.02})\SZ & \nfe{57.1}\,(\tps{4.05})\SZ & \nfe{56.4}\,(\tps{2.59})\SZ \\
BlockSpec tree & \nfe{45.0}\,(\tps{5.03})\SZ & \nfe{43.1}\,(\tps{3.40})\SZ & \nfe{33.6}\,(\tps{6.23})\SZ & \nfe{31.7}\,(\tps{3.74})\SZ & \nfe{33.5}\,(\tps{6.92})\SZ & \nfe{32.5}\,(\tps{4.70})\SZ & \nfe{59.6}\,(\tps{3.79})\SZ & \nfe{57.0}\,(\tps{2.48})\SZ \\
\bottomrule
\end{tabular}
}
\end{table}

In the upper block of Table~\ref{tab:main-sdar}, \sys{} delivers $3.2$--$4.7\times$ fewer target forwards than draft-free decoding at $D8$ and $4.1$--$6.8\times$ at $D20$, at $2.5$--$3.4\times$ the throughput of draft-free decoding at $D8$; against Spiffy's calibrated graph it has $16$--$36\%$ fewer NFE and $1.04$--$1.41\times$ its throughput in every task-and-budget setting, with larger margins against the chain and the tree. The baselines' throughput collapses with budget where their NFE does not: from $D8$ to $D20$ the chain falls from $2.3$--$2.6\times$ to $1.4$--$1.7\times$ and the tree from $1.5$--$1.7\times$ to $0.86$--$1.10\times$ (\emph{slower than draft-free} on MBPP), since a wider draft costs $(1{+}D)L$ positions whether or not it is accepted, while they accept only $1.1$--$2.9$ tokens per drafted round against \sys{}'s $2.5$--$7.5$ (Appendix~\ref{app:mechanism}). \sys{}-$\gamma$ removes the budget dependence: its throughput ratio is nearly flat in the budget ($2.5$--$3.5\times$ at $D8$, $2.3$--$3.4\times$ at $D20$) and at or above fixed-budget \sys{} in every SDAR-8B setting at both temperatures. The floor pays at most $4$ NFE per setting, so \sys{} stays NFE-optimal while \sys{}-$\gamma$ still has fewer forwards than every baseline in every setting, at every budget from $D6$ to $D20$ (Tables~\ref{tab:main-full} and~\ref{tab:main-full-b}). MRP~\citep{xu_multi-token_2026}, the closest learned drafter, reports $1.17$--$1.40\times$ over one-token-per-step decoding on these same targets in a different implementation, against \sys{}'s $2.5$--$3.8\times$ at $D8$, a margin wider than that difference can account for, and its argmax-equality acceptance has no distribution-preserving extension to sampling (Table~\ref{tab:mrp}).

The last two blocks of Table~\ref{tab:main-sdar} run every method on top of the threshold decoder
of \citet{wu_fast-dllm_2026}, which alone is $3.3$--$6.1\times$ faster than standard decoding. There \sysSets{} has
the lowest NFE and the highest throughput in every column, and so it does on SDAR-4B, SDAR-1.7B and LLaDA2.0-mini at
both temperatures (Tables~\ref{tab:main-4b}--\ref{tab:main-mini-t1}). On SDAR-8B at $T{=}0$ and $D8$,
\sysSets{}-$\gamma$ takes $47$--$48\%$ fewer forwards than the threshold decoder and reaches $5.2$--$9.9\times$ the
throughput of standard decoding ($7.9\times$ on average, against $5.0\times$ for the threshold decoder alone); its
exact variant, which keeps the threshold decoder's output up to floating-point batch effects, reaches $7.6\times$. Our drafter trains on general chat
only. A second drafter, trained on the train splits of GSM8K, MATH and MBPP and on CodeAlpaca, is faster still
($8.2\times$ on average against $7.9\times$), as expected when a drafter's training data match the
evaluation tasks~\citep{hong_training_2025,zbib_taps_2026}; on HumanEval, which neither drafter saw, the two
drafters are within $3\%$ of each other (Appendix~\ref{app:sets-data}). We keep the general chat drafter as our main
setting, since its speed-ups do not rely on training data from the benchmarks' own distribution.

\subsection{Sampling temperature: coupled-noise match against the baselines}
\label{sec:exp-temp}

The second block of Table~\ref{tab:main-sdar} shows that nothing changes when the target samples at $T{=}1$: under the same verifier, \sys{} keeps $16$--$35\%$ fewer NFE than Spiffy and $1.08$--$1.37\times$ its throughput. Free chain under rejection sampling is the strongest baseline on MBPP NFE, but stays $0.18$--$0.53$ below \sys{} in throughput, since every accepted link pays an acceptance test, and stalls on MATH500 ($116$ NFE at both budgets, $1.02\times$ at $D20$). Every $T{=}1$ result is distribution-preserving, with accuracies within $0.04$ of the draft-free run over seeds; tree $+$ coupled-noise match is the deployment configuration at any temperature
(Appendix~\ref{app:verifier-extra}).

\subsection{Transferability within and across model families}
\label{sec:exp-transfer}

\begin{table}[t]
\centering\footnotesize
\caption{Checkpoint transfer ($T{=}0$, b32, $D8$, unpadded, seed $0$): each checkpoint pointed at a
target with no retraining or recalibration, \nfe{NFE per sample}. Ratios in Table~\ref{tab:transfer-full}.
mini: the LLaDA2.0-mini checkpoint. $^{\dagger}$approximate prefix cache of \citet{wu_fast-dllm_2026}, not
distribution-preserving.
Lower block: \sysSets{}-$\gamma$ on the threshold decoder ($\tau{=}0.9$), its input layers refit on $99$ prompts
of each target (ratios in Table~\ref{tab:sets-transfer-full-d}).}
\label{tab:transfer}
\setlength{\tabcolsep}{4pt}
\begin{tabular}{lcccccccc}
\toprule
\nfe{NFE}$\downarrow$ at $D8$ & \multicolumn{2}{c}{GSM8K} & \multicolumn{2}{c}{MBPP} & \multicolumn{2}{c}{HumanEval} & \multicolumn{2}{c}{MATH500} \\
\cmidrule(lr){2-3}\cmidrule(lr){4-5}\cmidrule(lr){6-7}\cmidrule(lr){8-9}
target $\backslash$ drafter & SDAR-8B & mini & SDAR-8B & mini & SDAR-8B & mini & SDAR-8B & mini \\
\midrule
SDAR-8B & \nfe{70.1}\SZ & \nfe{73.0}\SZ & \nfe{56.7}\SZ & \nfe{58.7}\SZ & \nfe{62.5}\SZ & \nfe{67.3}\SZ & \nfe{84.6}\SZ & \nfe{88.1}\SZ \\
LLaDA2.0-mini & \nfe{85.9}\SZ & \nfe{82.6}\SZ & \nfe{101.7}\SZ & \nfe{94.1}\SZ & \nfe{97.6}\SZ & \nfe{87.7}\SZ & \nfe{87.1}\SZ & \nfe{84.6}\SZ \\
LLaDA-8B$^{\dagger}$ & \nfe{79.9}\SZ & \nfe{83.0}\SZ & \nfe{54.8}\SZ & \nfe{57.3}\SZ & \nfe{83.4}\SZ & \nfe{87.5}\SZ & \nfe{83.8}\SZ & \nfe{86.7}\SZ \\
Dream-7B$^{\dagger}$ & \nfe{66.3}\SZ & \nfe{68.0}\SZ & \nfe{51.2}\SZ & \nfe{52.5}\SZ & \nfe{56.6}\SZ & \nfe{58.3}\SZ & \nfe{77.1}\SZ & \nfe{79.5}\SZ \\
\midrule
\multicolumn{9}{l}{\emph{threshold decoding ($\tau{=}0.9$), \sysSets{}-$\gamma$ (cap 8) with a calibrated drafter}} \\
\midrule
SDAR-8B & \nfe{30.6}\SZ & \nfe{30.1}\SZ & \nfe{22.8}\SZ & \nfe{22.3}\SZ & \nfe{22.7}\SZ & \nfe{23.3}\SZ & \nfe{42.3}\SZ & \nfe{41.7}\SZ \\
LLaDA2.0-mini & \nfe{48.7}\SZ & \nfe{44.0}\SZ & \nfe{37.7}\SZ & \nfe{34.7}\SZ & \nfe{19.8}\SZ & \nfe{19.0}\SZ & \nfe{47.7}\SZ & \nfe{43.7}\SZ \\
LLaDA-8B$^{\dagger}$ & \nfe{40.6}\SZ & \nfe{38.7}\SZ & \nfe{22.4}\SZ & \nfe{21.6}\SZ & \nfe{36.6}\SZ & \nfe{34.8}\SZ & \nfe{48.7}\SZ & \nfe{46.7}\SZ \\
Dream-7B$^{\dagger}$ & \nfe{30.2}\SZ & \nfe{30.0}\SZ & \nfe{19.2}\SZ & \nfe{19.0}\SZ & \nfe{23.6}\SZ & \nfe{23.1}\SZ & \nfe{48.5}\SZ & \nfe{49.0}\SZ \\
\bottomrule
\end{tabular}

\end{table}

Transfer takes two forms. Re-running identical training recipe on
\textbf{SDAR-4B} reproduces the SDAR-8B picture, lowest NFE and highest throughput
in every comparison at $T{=}0$ and $T{=}1$. On \textbf{SDAR-1.7B}, and on
\textbf{LLaDA2.0-mini}, a mixture-of-experts target with a different tokenizer, it gives the lowest NFE in every
comparison and beats or matches every baseline in throughput (within $5\%$, Appendix~\ref{app:accounting}). The method
carries to block width 16 untuned (Appendices~\ref{app:ext-models}
and~\ref{app:ext-blocks}). Table~\ref{tab:transfer} swaps the trained checkpoints
between targets, tokenizers included, with no retraining or recalibration, and
Table~\ref{tab:transfer-full} gives the two checkpoints on all six targets: the
transferred drafter costs $4$--$8\%$ NFE on SDAR-8B and \emph{still beats every
baseline calibrated on SDAR-8B} (Table~\ref{tab:main-sdar}), and $3$--$11\%$
on LLaDA2.0-mini, where it stays ahead of every baseline on GSM8K and MATH500. Beyond the
block-causal family, \textbf{LLaDA-8B}~\citep{nie_large_2025} and
\textbf{Dream-7B}~\citep{ye_dream_2025} attend over the sequence, so drafted
runs and control both use the approximate cache of \citet{wu_fast-dllm_2026}
(Appendix~\ref{app:setup}) and these rows measure reach rather than losslessness:
checkpoints that never saw either model cut LLaDA-8B's $256$ forwards to $55$--$88$
and Dream-7B's to $51$--$80$: the order model learns \emph{where a diffusion target
commits}, and much of that carries.

The \sysSets{} drafter also reads hidden states and token embeddings, whose spaces differ between models,
so for each new target we refit only its input layers on $99$ calibration prompts from the GSM8K, MATH, MBPP and
CodeAlpaca train splits (Appendix~\ref{app:sets-transfer}). With this light calibration, a transferred drafter stays
within $8\%$ of the native drafter's throughput (Table~\ref{tab:sets-transfer-full-d}), and it also speeds up LLaDA-8B and
Dream-7B, on which no drafter was trained (lower block of Table~\ref{tab:transfer}).

\subsection{Ablations}
\label{sec:exp-ablation}

Table~\ref{tab:ablation} (Appendix~\ref{app:ext-ablations}) varies one component on SDAR-8B MBPP. The largest effect comes from replacing the tree with a chain: NFE increases by $7.0$ at $D8$ and by $11.7$ at $D20$, and the ratio at $D8$ drops from $3.65$ to $3.34\times$. Budget therefore pays through branching, not a longer chain. The
token hedge and the larger training corpus are each worth under one NFE, and at $\gamma{=}0$ padding moves
throughput by under $0.1$. The mass floor trades NFE for wall-clock speed, but
only without padding: at $D20$, $\gamma = 0 / 0.02 / 0.05$ gives
$39.5 / 41.8 / 44.4$ NFE at $3.19 / 4.13 / 4.16\times$, by computing fewer verification positions; padded, $\gamma{=}0.02$ gives $3.32\times$.
\begin{figure}[t]
\centering
\begin{minipage}[c]{0.58\linewidth}\centering\includegraphics[width=\linewidth]{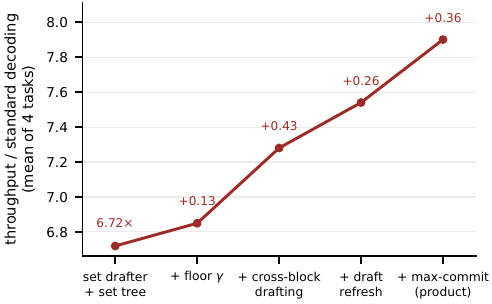}\end{minipage}\hfill
\begin{minipage}[c]{0.39\linewidth}
\caption{What each part of \sysSets{} adds on SDAR-8B ($\tau{=}0.9$, $T{=}0$, cap $8$): throughput over standard
decoding, mean of the four tasks, as each part is added in turn (labels: gain per step). The block-end additions
(cross-block drafting, draft refresh) give the most. Numbers in Table~\ref{tab:sets-ladder}.}
\label{fig:sets-ladder}
\end{minipage}
\end{figure}

For \sysSets{} (Figure~\ref{fig:sets-ladder}; Table~\ref{tab:sets-ladder}, Appendix~\ref{app:sets}), the set drafter and tree
alone reach $6.7\times$ the throughput of standard decoding at $D8$ (mean of the four tasks); the two block-end additions, cross-block drafting and draft refresh, add the most
($6.9\to7.5\times$), since a threshold decoder ends a block every few forwards, and the product max-commit rule brings it to $7.9\times$.
\section{Discussion}
\label{sec:conclusion}

DLM decoding remains constrained by a fundamental speed--quality trade-off. \sys{}
addresses this challenge with a two-fold solution. We show that stale target
distributions already contain highly reliable token information, making commit
order the key missing signal for effective speculation. A lightweight
$0.62$M-parameter model predicts this order and constructs multi-step draft
trees, while coupled-noise verification preserves the target distribution under
sampling; inspired by \sys{}, \sysSets{} drafts the commit sets and their tokens for dynamic threshold decoding. \sys{} transfers
across six DLMs without retraining, and on SDAR-8B it uses $16$--$36\%$ fewer target forwards than an
offline-calibrated draft graph at up to $1.4\times$ its throughput; there \sysSets{} lifts dynamic threshold decoding from $5.0\times$ to $7.9\times$ the throughput of standard decoding on average, and from $4.8\times$ to $6.8\times$ at $T{=}1$.
More broadly, our results show that substantially faster DLM generation can be
achieved without sacrificing the quality or exactness of the underlying decoding
process (for \sysSets{}, in its exact mode, or with the cut rule, which only makes the decoder stricter), establishing commit-order prediction as a promising foundation for the
next generation of speculative decoding methods for DLMs.

\paragraph{Limitations.}
The main limitation of \sys{} is that its gains diminish when verification is costly
or acceptance is low. Verification becomes the bottleneck at wide blocks, while exact
tree packing remains unresolved under bidirectional attention. Gains also narrow on
low-acceptance tasks, and at $T \ge 2$ the $T{=}0$-trained order model is outperformed
by noise-aware ranking (Appendix~\ref{app:ext-ablations}). \sysSets{} reads the target's hidden states and token embeddings, so unlike \sys{} it needs a short calibration on each new target, and its larger drafter ($5.1$M parameters, about $6\%$ of a round's time on SDAR-8B) matters more when the target forward is cheap.

\section*{AI use statement}

Generative AI tools were used during manuscript preparation to support language refinement and editing, and during development to assist with implementation of the experimental codebase. They were not used to formulate the research questions, develop the proposed method, define the evaluation procedure, or draw conclusions from the experimental findings. The quantitative results reported in this paper originate from the experiments described in Section~\ref{sec:exp-setup} and the appendix. The authors inspected and verified all AI-assisted contributions, including code incorporated into the experimental workflow, and remain fully responsible for the manuscript, implementation, experimental results, and conclusions.

\subsection*{Ethics statement}

This work focuses on accelerating inference in diffusion large language models and does not involve human subjects or the collection of personal data. We do not anticipate any ethical concerns beyond those associated with the underlying language models and their downstream applications.

\subsection*{Reproducibility statement}

Code and the trained drafter checkpoints are available at \url{https://github.com/Eliezer318/SPOT-dLLM}, with instructions to reproduce the experiments. Every number in the benchmark tables comes from one benchmark run with the protocol of Section~\ref{sec:exp-setup} and Appendix~\ref{app:setup}, on machines with eight identical NVIDIA RTX PRO 6000 Blackwell cards and one job per card. The one-commit \sys{} results come from one machine, and Table~\ref{tab:main-sdar}, the SDAR-4B $T{=}1$ table and the threshold-decoding (\sysSets{}) results from a second one, each ratio over that machine's own draft-free run. 

\bibliography{references}
\bibliographystyle{iclr2027_conference}

\appendix
\clearpage

\section{Formal Setup and Exactness Proofs}
\label{app:proofs}

This section makes the decoding semantics used in
Section~\ref{sec:method} explicit and proves
Propositions~\ref{prop:greedy} and~\ref{prop:match}.

\subsection{Formal decoding model}
\label{app:formal-setup}

\paragraph{States and terminal states.}
Recall
\[
    \mathcal{X}
    =
    (\mathcal{V}\cup\{\textsc{m}\})^L,
    \qquad
    \mathcal{M}(x)
    =
    \{s:x_s=\textsc{m}\}.
\]
Define the nonterminal state space
\[
    \mathcal{X}_{\mathrm{nt}}
    =
    \{x\in\mathcal{X}:
      \mathcal{M}(x)\neq\varnothing\}.
\]
All target transition maps below are defined on
$\mathcal{X}_{\mathrm{nt}}$.

The completed causal context $C$ is fixed while an active block is
decoded.  A target forward at state $x$ returns
\[
    p_{\theta,s}(\cdot\mid x,C)
\]
for every $s\in\mathcal{M}(x)$.

\paragraph{Idealized target evaluation.}
The exactness results concern the mathematical target function:
evaluating $p_\theta(\cdot\mid x,C)$ as one replica of a batch is
assumed to return the same mathematical conditional as evaluating the
same $(x,C)$ alone.  This separates algorithmic exactness from
finite-precision effects caused by physical batch shape.

\paragraph{Position selection.}
The target chooses
\[
    s^\star(x)
    =
    \argmax_{s\in\mathcal{M}(x)}
    \max_{v\in\mathcal{V}}
    p_{\theta,s}(v\mid x,C).
\]
All argmax operations use a fixed deterministic tie-breaking rule.
In particular, $s^\star(x)$ depends only on the current target state
and not on sampling randomness that has not yet been consumed.

For $s\in\mathcal{M}(x)$ and $v\in\mathcal{V}$, write
\[
    x\oplus(s,v)
\]
for the state obtained by replacing $x_s=\textsc{m}$ by $v$.

\paragraph{Greedy transition.}
At $T=0$,
\[
    v^\star_0(x)
    =
    \argmax_v
    p_{\theta,s^\star(x)}(v\mid x,C),
\]
and therefore
\[
    \Phi_0:
    \mathcal{X}_{\mathrm{nt}}
    \rightarrow
    \mathcal{X},
    \qquad
    \Phi_0(x)
    =
    x\oplus
    \bigl(
        s^\star(x),v^\star_0(x)
    \bigr).
\]
Starting from initial block state $x_0$, draft-free greedy decoding is
the deterministic orbit
\[
    x^{\mathrm{DF}}_0=x_0,
    \qquad
    x^{\mathrm{DF}}_{j+1}
    =
    \Phi_0(x^{\mathrm{DF}}_j)
\]
until a terminal state is reached.

\paragraph{Sampling transition.}
At $T>0$, the target token law at position $s$ is
\[
    \pi_{T,s}(v\mid x,C)
    =
    \frac{
        p_{\theta,s}(v\mid x,C)^{1/T}
    }{
        \sum_{u\in\mathcal{V}}
        p_{\theta,s}(u\mid x,C)^{1/T}
    }.
\]
Standard draft-free target sampling at state $x$ first chooses the
deterministic position $s^\star(x)$ and then draws
\[
    V
    \sim
    \pi_{T,s^\star(x)}(\cdot\mid x,C).
\]

\paragraph{Shared Gumbel representation.}
Let $a(s)$ denote the absolute output position corresponding to local
block position $s$.  Draw
\[
    G_{a,v}
    \overset{\mathrm{iid}}{\sim}
    \mathrm{Gumbel}(0,1)
\]
for every absolute position $a$ and token $v$, and let
\[
    G=\{G_{a,v}\}_{a,v}.
\]
For fixed $G$, define
\[
    v^\star_G(x)
    =
    \argmax_v
    \left[
        \frac{1}{T}
        \log p_{\theta,s^\star(x)}(v\mid x,C)
        +
        G_{a(s^\star(x)),v}
    \right]
\]
and
\[
    \Phi_G:
    \mathcal{X}_{\mathrm{nt}}
    \rightarrow
    \mathcal{X},
    \qquad
    \Phi_G(x)
    =
    x\oplus
    \bigl(
        s^\star(x),v^\star_G(x)
    \bigr).
\]
For any fixed realization $G=g$, $\Phi_g$ is deterministic.

\paragraph{Draft trees.}
At speculative round $t$, a drafter $\mathcal{D}$ returns a finite
rooted tree
\[
    \mathcal{T}_t
    =
    (\mathcal{N}_t,\mathcal{E}_t)
\]
whose root is the current committed state $x_t$.  Every edge
$(y,z)\in\mathcal{E}_t$ corresponds to exactly one hypothetical
commitment:
\[
    z=y\oplus(s,v)
\]
for some $s\in\mathcal{M}(y)$ and $v\in\mathcal{V}$.

The exactness arguments make no assumption about how this tree is
constructed.  In particular, $\mathcal{D}$ may be learned or
hand-designed, may return a chain or a tree, may depend on all previous
target outputs, and at $T>0$ may depend on the shared randomness $G$.

\paragraph{State-equality verifier.}
Let $\Phi$ denote $\Phi_0$ for $T=0$ or $\Phi_G$ after fixing the
shared randomness for $T>0$.  Beginning at the root $y_0=x_t$, the
verifier proceeds recursively.

If $y_j$ has a drafted child $z$ such that
\[
    z=\Phi(y_j),
\]
the verifier accepts that child and continues from it.  If no such
child exists, it rejects all remaining descendants, commits
$\Phi(y_j)$ directly, and terminates the speculative round.  If $y_j$
is terminal, no further transition is made.

Since $\Phi(y_j)$ is a unique one-step extension, at most one child of
a node can be accepted.  Hence the accepted portion of an arbitrary
draft tree is a single path.

Let $X^{\mathrm{DF}}_T$ denote the complete draft-free trajectory and
$X^{\mathrm{Spec}}_{T,\mathcal D}$ the complete committed speculative
trajectory.  A verifier is distribution preserving when
\[
    X^{\mathrm{Spec}}_{T,\mathcal D}
    \overset{d}{=}
    X^{\mathrm{DF}}_T
\]
for every drafter $\mathcal D$ that returns a finite tree of the form above.

\subsection{Proof of Proposition~\ref{prop:greedy}}
\label{app:proof-greedy}

\begin{proof}
Fix target $p_\theta$, causal context $C$, initial block state $x_0$,
and an arbitrary drafter $\mathcal D$. By the fixed deterministic
tie-breaking rule, at $T=0$ the target transition
\[
    \Phi_0:
    \mathcal X_{\mathrm{nt}}
    \rightarrow
    \mathcal X
\]
is deterministic and gives a unique successor for every nonterminal
state.

Let
\[
    x^{\mathrm{DF}}_0=x_0,
    \qquad
    x^{\mathrm{DF}}_{k+1}
    =
    \Phi_0(x^{\mathrm{DF}}_k)
\]
be the draft-free trajectory, and let $x^{\mathrm{Spec}}_k$ denote the
state of the speculative decoder after $k$ committed transitions.

We prove by induction on $k$ that
\[
    x^{\mathrm{Spec}}_k
    =
    x^{\mathrm{DF}}_k
\]
for every commitment before termination.

\paragraph{Base case.}
For $k=0$, both procedures start from the same initial state:
\[
    x^{\mathrm{Spec}}_0
    =
    x_0
    =
    x^{\mathrm{DF}}_0.
\]

\paragraph{Inductive step.}
Assume
\[
    x^{\mathrm{Spec}}_k
    =
    x^{\mathrm{DF}}_k
\]
and that this state is nonterminal.  From $x^{\mathrm{Spec}}_k$, the target has
the unique successor
\[
    \Phi_0(x^{\mathrm{Spec}}_k).
\]
If the draft tree contains this state as a child of
$x^{\mathrm{Spec}}_k$, state-equality verification accepts that child.
Otherwise, the verifier rejects the speculative continuation and
commits the same state as the target fallback.  Hence, in either case,
\[
\begin{aligned}
    x^{\mathrm{Spec}}_{k+1}
    &=
    \Phi_0(x^{\mathrm{Spec}}_k) \\
    &=
    \Phi_0(x^{\mathrm{DF}}_k) \\
    &=
    x^{\mathrm{DF}}_{k+1}.
\end{aligned}
\]
Thus the induction hypothesis holds for $k+1$.

Since $\Phi_0(x)$ is unique, at most one child of any tree node can be
accepted.  Therefore, even for an arbitrary finite draft tree, the
accepted states form a single path, and the argument above applies to
every committed transition.

By induction,
\[
    X^{\mathrm{Spec}}_{0,\mathcal D}
    =
    X^{\mathrm{DF}}_0.
\]
\end{proof}

The argument is block-local only for notation. Since the completed
block is identical to draft-free decoding, the causal context of the
next block is also identical. Induction over blocks therefore gives
equality of the complete generated sequence at $T=0$.

\subsection{Proof of Proposition~\ref{prop:match}}
\label{app:proof-match}

\begin{proof}
Fix $T>0$, target $p_\theta$, causal context $C$, initial state $x_0$,
and an arbitrary drafter $\mathcal D$.

The proof has three steps. First, we show that the shared-Gumbel
construction induces the same draft-free trajectory distribution as
ordinary temperature-$T$ sampling. Second, conditioned on a realization
of the Gumbel table, verification reduces to the deterministic case of
Proposition~\ref{prop:greedy}. Third, we marginalize over the shared
randomness.

\paragraph{Step 1: Gumbel-driven draft-free law.}
Fix a nonterminal state $x$ and let
\[
    s=s^\star(x).
\]
Since $\{G_{a(s),v}\}_{v\in\mathcal V}$ are independent standard
Gumbels, the Gumbel--argmax identity gives
\begin{align}
    \Pr[v^\star_G(x)=u]
    &=
    \frac{
        \exp\!\left(
            T^{-1}\log p_{\theta,s}(u\mid x,C)
        \right)
    }{
        \sum_{v\in\mathcal V}
        \exp\!\left(
            T^{-1}\log p_{\theta,s}(v\mid x,C)
        \right)
    }
    \\
    &=
    \frac{
        p_{\theta,s}(u\mid x,C)^{1/T}
    }{
        \sum_{v\in\mathcal V}
        p_{\theta,s}(v\mid x,C)^{1/T}
    }
    \\
    &=
    \pi_{T,s}(u\mid x,C).
\end{align}
Thus, for any fixed state, the Gumbel construction induces exactly the
temperature-$T$ target token law.

It remains to show that using a single Gumbel table indexed by absolute
position preserves this law along the full adaptive trajectory.
Consider commitment step $j$ of Gumbel-driven draft-free decoding, with
current state $x_j$ and selected position
\[
    s_j=s^\star(x_j).
\]
Let $\mathcal F_j$ denote the history generated by the states and
Gumbel rows used during the first $j$ commitments. Since each
commitment fills a previously masked position, $s_j$ has not been
committed before, and therefore the row
\[
    G_{a(s_j),\cdot}
\]
has not previously been used.

Moreover, $s_j$ is determined by $x_j$ and hence by $\mathcal F_j$;
its selection does not depend on any unused Gumbel row. Since the
Gumbel rows are independent across absolute positions,
$G_{a(s_j),\cdot}$ remains an independent standard-Gumbel vector
conditional on $\mathcal F_j$. Therefore,
\[
    \Pr\!\left[
        v^\star_G(x_j)=u
        \mid
        \mathcal F_j
    \right]
    =
    \pi_{T,s_j}(u\mid x_j,C).
\]
Hence the Gumbel-driven decoder has the same conditional transition law
as ordinary draft-free sampling at every commitment. Since both start
from $x_0$, the chain rule gives
\begin{equation}
\label{eq:gumbel-df-law}
    X^{\mathrm{DF}}_T(G)
    \overset{d}{=}
    X^{\mathrm{DF}}_T.
\end{equation}

\paragraph{Step 2: conditional pathwise exactness.}
Fix any realization $G=g$. Then
\[
    \Phi_g:
    \mathcal X_{\mathrm{nt}}
    \rightarrow
    \mathcal X
\]
is a deterministic transition map.

The proof of Proposition~\ref{prop:greedy} requires only that the target
have a unique deterministic successor at each state and that a drafted
child be accepted exactly when it equals that successor. Both properties
hold with $\Phi_g$ in place of $\Phi_0$. Therefore, for every realization
$g$,
\begin{equation}
\label{eq:pathwise-all-g}
    X^{\mathrm{Spec}}_{T,\mathcal D}(g)
    =
    X^{\mathrm{DF}}_T(g).
\end{equation}

This argument is independent of how $\mathcal D$ constructs its draft.
In particular, the draft may depend on stale target distributions,
learned parameters, previous decoding history, or the shared table
$G$ itself; a drafted edge is committed only when it equals the target
successor under the same realization $g$.

\paragraph{Step 3: marginalization.}
By Step 2, for every realization $g$ of the shared Gumbel table,
\[
    X^{\mathrm{Spec}}_{T,\mathcal D}(g)
    =
    X^{\mathrm{DF}}_T(g).
\]
Therefore, when $G$ is drawn according to its Gumbel distribution,
\[
    X^{\mathrm{Spec}}_{T,\mathcal D}(G)
    =
    X^{\mathrm{DF}}_T(G).
\]
By definition, $X^{\mathrm{Spec}}_{T,\mathcal D}$ denotes this same
random speculative trajectory with its dependence on $G$ left implicit.

Hence, for any set $A$ of complete decoding trajectories,
\[
    \Pr\!\left[
        X^{\mathrm{Spec}}_{T,\mathcal D}\in A
    \right]
    =
    \Pr\!\left[
        X^{\mathrm{DF}}_T(G)\in A
    \right].
\]
Using \eqref{eq:gumbel-df-law},
\[
    \Pr\!\left[
        X^{\mathrm{Spec}}_{T,\mathcal D}\in A
    \right]
    =
    \Pr\!\left[
        X^{\mathrm{DF}}_T\in A
    \right].
\]
Since this holds for every $A$,
\[
    X^{\mathrm{Spec}}_{T,\mathcal D}
    \overset{d}{=}
    X^{\mathrm{DF}}_T.
\]

Finally, for fixed $G$, the committed trajectory is the deterministic
orbit of $\Phi_G$ and is therefore invariant to the choice of drafter.
This establishes drafter-invariance.

\end{proof}

\section{Rejection Sampling as an Alternative Verifier}
\label{app:rs}

Coupled-noise match is the verifier used by \sys{} because it requires
only state equality and applies directly to DDTree.  If a drafter
instead exposes an explicit \emph{sequential proposal distribution},
exact verification can alternatively be performed by rejection
sampling (RS).  We formalize this alternative here, prove its
correctness and maximal per-link acceptance probability, and describe
the implementation-level validation used in our experiments.

\subsection{Sequential proposal setup}
\label{app:rs-setup}

Consider a verified predecessor state
\[
    y_k\in\mathcal X_{\mathrm{nt}}.
\]
Let
\[
    \mathcal Z(y_k)
    =
    \mathcal M(y_k)\times\mathcal V
\]
be the set of possible one-step position--token commitments.

\paragraph{Target commit law.}
Let
\[
    s_k^\star
    =
    s^\star(y_k).
\]
At $T>0$, the target distribution over one-step commitments is
\begin{equation}
\label{eq:rs-target-law}
    \pi_k(s,v)
    =
    \mathbf 1\{s=s_k^\star\}
    \,
    \pi_{T,s_k^\star}(v\mid y_k,C).
\end{equation}
Therefore
\[
    \sum_{z\in\mathcal Z(y_k)}
    \pi_k(z)
    =
    1.
\]

\paragraph{Proposal law.}
Suppose the drafter exposes a normalized proposal distribution
\[
    q_k
    \in
    \Delta(\mathcal Z(y_k))
\]
and generates one proposed commitment
\begin{equation}
\label{eq:rs-proposal-draw}
    \widetilde Z_k
    =
    (\widetilde S_k,\widetilde V_k)
    \sim q_k.
\end{equation}
For the chain proposal used in our experiments,
\begin{equation}
\label{eq:rs-proposal-law}
    q_k(s,v)
    =
    q_k^{\mathrm{pos}}(s)
    \,
    \widehat\pi_{T,k,s}(v),
\end{equation}
where $q_k^{\mathrm{pos}}$ is the drafter's next-position distribution and
$\widehat\pi_{T,k,s}$ is its stale temperature-adjusted token
conditional.

The factorization in \eqref{eq:rs-proposal-law} is not required by the
proof; exactness needs only that the proposal actually be sampled from
the explicit distribution $q_k$ supplied to the verifier.

\paragraph{Accept--residual rule.}
Given proposal $\widetilde Z_k=z$ with $q_k(z)>0$, accept it with
probability
\begin{equation}
\label{eq:rs-acceptance-rule}
    \alpha_k(z)
    =
    \min
    \left\{
        1,
        \frac{\pi_k(z)}{q_k(z)}
    \right\}.
\end{equation}
Upon rejection, draw a replacement from
\begin{equation}
\label{eq:rs-residual}
    r_k(z)
    =
    \frac{
        (\pi_k(z)-q_k(z))_+
    }{
        \sum_{w\in\mathcal Z(y_k)}
        (\pi_k(w)-q_k(w))_+
    }.
\end{equation}
If the denominator is zero, then $\pi_k=q_k$, rejection has
probability zero, and $r_k$ is never queried.

\subsection{Exactness and optimal acceptance}
\label{app:rs-proof}

\begin{proposition}[Rejection-sampling exactness and optimal acceptance]
\label{prop:rs}
Fix a verified predecessor state $y_k$.  Let
$\widetilde Z_k\sim q_k$ be accepted according to
\eqref{eq:rs-acceptance-rule}, with rejected proposals replaced using
\eqref{eq:rs-residual}.  Then the committed pair $Z_k$ satisfies
\[
    Z_k\sim\pi_k.
\]
Moreover,
\[
    \Pr[
        \mathrm{accept}\mid y_k
    ]
    =
    1-D_{\mathrm{TV}}(\pi_k,q_k),
\]
which is the maximal possible agreement probability among all
couplings with marginals $\pi_k$ and $q_k$.  Applying the rule
sequentially at every verified predecessor therefore reproduces the
complete draft-free target distribution.
\end{proposition}

\begin{proof}
Fix $y_k$ and abbreviate
\[
    \pi=\pi_k,
    \qquad
    q=q_k,
    \qquad
    r=r_k.
\]

\paragraph{Step 1: accepted mass.}
For arbitrary $z\in\mathcal Z(y_k)$ with $q(z)>0$,
\begin{align}
    \Pr[
        \widetilde Z_k=z,
        \mathrm{accept}
    ]
    &=
    q(z)
    \min
    \left\{
        1,\frac{\pi(z)}{q(z)}
    \right\}
    \\
    &=
    \min\{q(z),\pi(z)\}.
\label{eq:rs-proof-accepted-mass}
\end{align}
If $q(z)=0$, both sides of
\eqref{eq:rs-proof-accepted-mass} are zero, so the identity holds for
all $z$.

Summing over $z$,
\begin{equation}
\label{eq:rs-proof-total-accept}
    \Pr[\mathrm{accept}]
    =
    \sum_z
    \min\{q(z),\pi(z)\}.
\end{equation}
Hence
\begin{equation}
\label{eq:rs-proof-reject}
    \Pr[\mathrm{reject}]
    =
    1-
    \sum_z
    \min\{q(z),\pi(z)\}.
\end{equation}

\paragraph{Step 2: rejected mass.}
Because $\pi$ is normalized,
\begin{align}
    1-
    \sum_z\min\{q(z),\pi(z)\}
    &=
    \sum_z
    \left[
        \pi(z)-\min\{q(z),\pi(z)\}
    \right]
    \\
    &=
    \sum_z
    (\pi(z)-q(z))_+.
\label{eq:rs-proof-residual-mass}
\end{align}
Therefore
\[
    \Pr[\mathrm{reject}]
    =
    \sum_z
    (\pi(z)-q(z))_+.
\]
If this quantity is positive, then
\[
    r(z)
    =
    \frac{
        (\pi(z)-q(z))_+
    }{
        \Pr[\mathrm{reject}]
    },
\]
so
\begin{equation}
\label{eq:rs-proof-rejected-mass}
    \Pr[
        \mathrm{reject},
        Z_k=z
    ]
    =
    (\pi(z)-q(z))_+.
\end{equation}

\paragraph{Step 3: one-step exactness.}
Combining the accepted and rejected branches,
\begin{align}
    \Pr[Z_k=z]
    &=
    \min\{q(z),\pi(z)\}
    +
    (\pi(z)-q(z))_+
    \\
    &=
    \pi(z).
\end{align}
Thus
\[
    Z_k\sim\pi_k.
\]

If the denominator of \eqref{eq:rs-residual} is zero, then
$(\pi(z)-q(z))_+=0$ for every $z$.  Hence $\pi(z)\le q(z)$
everywhere, and because both distributions have total mass one,
$\pi=q$.  In that case \eqref{eq:rs-proof-total-accept} gives
$\Pr[\mathrm{accept}]=1$, so the residual branch is never used.

\paragraph{Step 4: maximal acceptance.}
For discrete distributions,
\[
    D_{\mathrm{TV}}(\pi,q)
    =
    1-
    \sum_z
    \min\{\pi(z),q(z)\}.
\]
Combining this identity with
\eqref{eq:rs-proof-total-accept},
\[
    \Pr[\mathrm{accept}]
    =
    1-D_{\mathrm{TV}}(\pi_k,q_k).
\]
The quantity
\[
    \sum_z
    \min\{\pi_k(z),q_k(z)\}
\]
is the maximal possible agreement probability of two random variables
with marginals $\pi_k$ and $q_k$.  The accept--residual construction
therefore realizes a maximal coupling.

\paragraph{Step 5: trajectory exactness.}
Let $Y_0,Y_1,\ldots$ denote the states produced by sequential RS.
Conditioned on any realized history $Y_{0:k}=y_{0:k}$, one-step
exactness gives
\[
    Z_k \mid (Y_{0:k}=y_{0:k})
    \sim \pi_k(\cdot\mid y_k),
\]
which is exactly the target's one-step commit law at $y_k$. Hence, for
every successor $y'$,
\[
    \Pr[
        Y_{k+1}=y'
        \mid
        Y_{0:k}=y_{0:k}
    ]
    =
    \Pr_{\mathrm{DF}}[
        X_{k+1}=y'
        \mid
        X_k=y_k
    ].
\]

Since RS and draft-free decoding start from the same initial state, the
chain rule implies that every finite trajectory $y_0,\ldots,y_m$ has
the same probability under both procedures:
\[
    \Pr[
        Y_0=y_0,\ldots,Y_m=y_m
    ]
    =
    \Pr_{\mathrm{DF}}[
        X_0=y_0,\ldots,X_m=y_m
    ].
\]
Therefore, the complete RS trajectory has the draft-free target
distribution.
\end{proof}

\subsection{Relation to batched tree verification}
\label{app:rs-tree}

Proposition~\ref{prop:rs} assumes the classical per-link proposal
interface: at each verified state $y_k$, a single proposal is sampled
from $q_k(\cdot \mid y_k)$ and corrected by the accept--residual rule.
Only after this step determines the realized successor is the proposal
at the next verified state generated.

After a rejection, the residual sample need not continue the drafted
chain, so the proposals are verified one link at a time.

Coupled-noise verification avoids this sequential dependence. Once the shared
randomness $G$ is fixed, the target defines a unique successor
$\Phi_G(y)$ for every hypothetical state $y$. Each drafted edge
$(y,z)$ can therefore be verified independently by the state-equality
test
\[
    z = \Phi_G(y),
\]
allowing all hypothetical states in the tree to be evaluated in the
same target batch and the unique target-consistent path to be
recovered afterward.

\subsection{Empirical validation}
\label{app:rs-certification}

Proposition~\ref{prop:rs} establishes the correctness of the idealized
accept--residual procedure. We also checked that our implementation matches its acceptance-rate identity on live
target states, including deliberately miscalibrated proposal distributions: on every tested state, the acceptance rate
over repeated executions differs from $1-D_{\mathrm{TV}}(\pi_k,q_k)$ by at most $0.002$. These checks validate the
implementation only; the distributional guarantee follows from Proposition~\ref{prop:rs}.

\section{Drafter Training and Draft-Tree Construction}
\label{app:drafter-algorithms}

This appendix gives the complete procedures used to train the
prefix-conditional commit-order predictor and to construct the speculative
draft tree at inference time. The notation follows
Section~\ref{sec:method-drafter}. In particular, $\widehat p_t$ denotes the
stale target conditionals available at speculative round $t$, and
$q_\phi(s\mid y)$ denotes the predicted distribution over the target's next
committed position conditioned on hypothetical prefix state $y$.

\subsection{Training the commit-order predictor}
\label{app:drafter-training-algorithm}

Training uses only draft-free trajectories generated by the target.
For each target state $x_t$, we retain the stale conditionals
$\widehat p_t$ available at that state together with the subsequent target
commit sequence
\[
    \mathbf{s}_t^\star
    =
    (s_t^\star,s_{t+1}^\star,\ldots).
\]
For every prefix depth $d$, we replay the first $d$ target commitments to
construct the hypothetical state $y_{t,d}$ while keeping
$\widehat p_t$ fixed. The next target commitment $s_{t+d}^\star$ is then
used as the supervision target. Hence each training example has the form
\[
    (\widehat p_t,y_{t,d},s_{t+d}^\star).
\]

\begin{algorithm}[t]
\caption{
Training the prefix-conditional commit-order predictor from draft-free
target trajectories.
}
\label{alg:drafter-training}
\begin{algorithmic}[1]

\Require Target $p_\theta$, prompt set $\mathcal S$, maximum replay depth
$d_{\max}$
\Ensure Trained order model $q_{\phi^\star}$

\State $\mathcal R_\theta
       \gets \textsc{Decode}(p_\theta,\mathcal S)$
       \Comment{Collect draft-free target trajectories}

\State $\mathcal E\gets\varnothing$

\For{each recorded target state
      $(x_t,\widehat p_t,\mathbf{s}_t^\star)\in\mathcal R_\theta$}

    \For{$d=0,\ldots,d_{\max}$}

        \State $y_{t,d}\gets
        \textsc{Prefix}(x_t,\mathbf{s}_t^\star,d)$
        \Comment{Replay the first $d$ target commitments}

        \State $\mathcal E\gets
        \mathcal E\cup
        \{(\widehat p_t,y_{t,d},s^\star_{t+d})\}$
        \Comment{Next position is the supervision target}

    \EndFor

\EndFor

\State $\displaystyle
\phi^\star
\gets
\argmin_\phi
\mathbb E_{(\widehat p,y,s^\star)\sim\mathcal E}
\left[-\log q_\phi(s^\star\mid y)\right]$

\State \Return $q_{\phi^\star}$

\end{algorithmic}
\end{algorithm}

The important feature of this construction is that the stale conditionals
$\widehat p_t$ remain unchanged while the hypothetical prefix is extended.
The resulting examples therefore match the information available to the
drafter during recursive inference rather than giving it access to future
target evaluations.

\subsection{Prefix-conditional draft-tree construction}
\label{app:drafter-inference-algorithm}

At speculative round $t$, drafting begins from the current state $x_t$ and
the stale target conditionals $\widehat p_t$. For each active hypothetical
state $y$, the order model produces
\[
    q_\phi(s\mid y),
    \qquad s\in\mathcal M(y).
\]
Each candidate position is paired with token hypotheses from
$\mathcal H_t(s)$. At $T=0$, the primary hypothesis is
\[
    \widetilde v_0(s)
    =
    \argmax_v \widehat p_{t,s}(v),
\]
with the stale runner-up additionally retained when the top-$1$ probability
is below $0.6$. At $T>0$, $\mathcal H_t(s)$ is determined by the coupled
stale draw in Eq.~\ref{eq:stale-gumbel}.

For every extension
\[
    z=y\oplus(s,v),
\]
we use the same prefix mass defined in the main text,
\[
    m(z)
    =
    m(y)\,
    q_\phi(s\mid y)\,
    \widehat p_{t,s}(v),
    \qquad
    m(x_t)=1.
\]
DDTree then allocates the available draft budget to the highest-mass states.

\begin{algorithm}[t]
\caption{
Prefix-conditional construction of the speculative draft tree under budget
$D$.
}
\label{alg:drafter-inference}
\begin{algorithmic}[1]

\Require Current target state $x_t$, stale conditionals $\widehat p_t$,
order model $q_\phi$, draft budget $D$
\Ensure Draft tree $\mathcal T_t$

\State $\mathcal T_t\gets\{x_t\}$
\State $m(x_t)\gets1$
\State $\mathcal F\gets\{x_t\}$
\Comment{Current expansion frontier}

\While{$\mathcal F\neq\varnothing$}

    \State $\mathcal C\gets\varnothing$

    \For{each active prefix $y\in\mathcal F$}

        \State $q_y(\cdot)\gets q_\phi(\cdot\mid y)$
        \Comment{Predict next commit position}

        \For{each $s\in\mathcal M(y)$}

            \State construct $\mathcal H_t(s)$ from
            $\widehat p_{t,s}$
            \Comment{Token hypotheses}

            \For{each $v\in\mathcal H_t(s)$}

                \State $z\gets y\oplus(s,v)$

                \State $\displaystyle
                m(z)\gets
                m(y)\,
                q_y(s)\,
                \widehat p_{t,s}(v)$

                \State $\mathcal C\gets\mathcal C\cup\{z\}$

            \EndFor

        \EndFor

    \EndFor

    \State $(\mathcal T_t,\mathcal F)
    \gets
    \textsc{DDTreeTopD}(\mathcal T_t,\mathcal C,D)$

\EndWhile

\State \Return $\mathcal T_t$

\end{algorithmic}
\end{algorithm}

Because every multiplicative factor in the prefix mass lies in $[0,1]$,
a child's mass cannot exceed that of its parent. This monotonicity permits
level-wise expansion while recovering the same top-$D$ fixpoint as
sequential best-first search. The final model applies no calibration temperature to $q_\phi$ before the
search.

\section{Experimental setup}
\label{app:setup}

\subsection{Drafter training details}
\label{app:training}
Architecture: three-layer transformer encoder over the masked positions,
$d_{\text{model}}{=}128$, $8$ heads, top-$32$ stale probabilities per position
($m_{\text{top}}{=}32$), the stale runner-up as a second token candidate (the hedge) when the top-$1$ probability
is below $0.6$;
$0.62$M parameters. Inputs per masked position: stale top-32 probabilities (raw and
log), eight summary statistics, two path features (taken-position set, depth). Training
data: decoding trajectories of the target itself at $T{=}0$, with
prefix-conditional examples extracted at every (round, depth). The prompt pools are
recorded per shard and differ by target. The SDAR-8B b32 and LLaDA2.0-mini b32
drafters are trained on CodeAlpaca~\citep{chaudhary_code_2023} alone, ${\approx}21$k
and ${\approx}14$k trajectories; the SDAR-4B drafter on a four-pool mixture of
CodeAlpaca, UltraChat~\citep{ding_enhancing_2023} and the chat and code shards of the
Llama-Nemotron post-training corpus~\citep{bercovich_llama-nemotron_2025}, $15{,}027$
trajectories. No pool contains MATH-500 or any evaluation sample. Optimization: Adam,
learning rate $3{\times}10^{-4}$, $10$ epochs; the released SDAR-8B b32 checkpoint
reaches $0.75$ held-out next-commit accuracy ($0.77$ at depth 0, $0.72$ at
depth 7). Training a drafter takes under two GPU-hours on one card. The identical
procedure (data regenerated per target and width) produces the checkpoints for
SDAR-8B b32/b16, SDAR-4B, SDAR-1.7B and LLaDA2.0-mini b32/b16; no per-setting tuning.

\paragraph{What the drafters have seen.} Two consequences follow, disclosed rather
than controlled. The SDAR-8B and mini drafters have seen only code, while two of the
four evaluation tasks are mathematics: their GSM8K and MATH500 runs are out of domain
for the drafter, which makes those results evidence that the learned order carries
across domains rather than memorizing one. And because the SDAR-4B drafter was trained
on a different mixture, any difference between the SDAR-8B and SDAR-4B tables
confounds target scale with corpus composition; we report both and attribute neither.
Regenerating the SDAR-8B corpus would invalidate every downstream run of the
benchmark, so the confound is left in place and stated here.

\subsection{Draft layout and verification cost}
\label{app:layout}
A target forward at a block state yields a distribution at every masked position; the
most confident position is filled and the block is re-run. To verify $D$ drafts in one
forward, each draft is appended behind the cached prefix as a separate copy of the
block with its own mask layout, and a copy attends to the prefix and to itself only, so
its logits are exactly what the target would compute at that state. One forward over
$(1{+}D)L$ positions therefore returns the target's own next commit at the root and at
every draft, a rejected draft costs positions but never changes the output, and packing
a tree into a single sequence is not exact here, since a copy's logits depend on its
complete mask layout.

\subsection{Decoding configuration}
Strict single-commit decoding (the threshold decoder of Section~\ref{sec:method-sets} for \sysSets{}),
generation length 256 ($64$ and $512$ in Appendix~\ref{app:ext-lengths}), block width 32 (16 where
stated), exact prefix KV cache on (for LLaDA2.0-mini: an exact prefix-KV path we
added to the released model code, verified to reproduce the uncached decoder's
commits), drafts laid out as copies of the block appended behind the prefix with a blockwise
mask, no padding of drafted forwards (every drafted run is unpadded),
and each run records whether its configuration is distribution-preserving. Prompt templates and stopping rules per task are those of the released
harness (Appendix~\ref{app:dataset}).

\subsection{Baseline configuration}
Spiffy's draft graph is calibrated with the released procedure of
\citet{agrawal_structuring_2026} (offline collection on held-out prompts:
MBPP train and MATH500 problems disjoint from the evaluation samples, 25 each,
lookahead 6, graph size 24) separately for SDAR-8B, SDAR-4B, SDAR-1.7B and LLaDA2.0-mini, one graph per target and decoding rule (one commit per forward; threshold $\tau{=}0.9$), and
truncated to the budget at decode time. The SSD-style free chain uses candidate
rank 1, the sampler's ordering and depth 4. The BlockSpec-style confidence tree is
a best-first tree over the target's renormalized stale confidences. All baselines run through the same decoder and
verifier as \sys{}; their draft-free control is shared with ours per task.

\subsection{Mass-floor calibration}
\label{app:gamma-calib}
The floor $\gamma$ of \sys{}-$\gamma$ is calibrated per target on $50$ held-out
samples, MBPP test samples $350$--$399$ (every evaluated MBPP run uses samples $0$--$349$):
$\gamma \in \{0, 0.02, 0.05, 0.1, 0.2\}$, unpadded, $D \in \{8, 12, 20\}$, $T{=}0$,
seed $0$, the fifteen runs back to back on one card with that card's draft-free run, which every
ratio of the grid is taken against. The selected value is the argmax of the mean throughput ratio over
the three budgets, the smaller floor being kept when two are within $3\%$
(Table~\ref{tab:gamma-calib}). $\gamma{=}0.02$ wins on the three SDAR targets, by $0.4\%$ over $0.05$ on
SDAR-8B, $2.7\%$ on SDAR-4B and $2.3\%$ on SDAR-1.7B, and by $9.5$--$12.5\%$ over no floor; the gain
comes from the large budgets ($D20$ on SDAR-8B: $3.29\times$ without a floor, $4.33\times$ with it). On
LLaDA2.0-mini $0.02$ leads $\gamma{=}0$ by $1.3\%$, inside the tie band, so the floor is switched off: that
target's throughput keeps rising with the budget without one. On the evaluated samples the calibrated
floor is at or above the unpadded $\gamma{=}0$ recipe in every SDAR-8B and SDAR-1.7B setting
(Tables~\ref{tab:main-sdar}, \ref{tab:main-full}, \ref{tab:main-full-b}, \ref{tab:main-17}
and~\ref{tab:main-17-t1}), and in all $16$ SDAR-4B settings (Tables~\ref{tab:main-4b} and~\ref{tab:main-4b-t1}).

\begin{table}[htbp]
\centering\small
\caption{Mass-floor calibration grids: mean throughput ratio over $D8/D12/D20$ on
the $50$ held-out MBPP samples (test $350$--$399$), unpadded, seed $0$, against the same card's draft-free run.
Bold: selected.}
\label{tab:gamma-calib}
\begin{tabular}{lccccc}
\toprule
target & $\gamma{=}0$ & $0.02$ & $0.05$ & $0.1$ & $0.2$ \\
\midrule
SDAR-8B b32 & 3.64\SZ & \textbf{4.10}\SZ & 4.08\SZ & 3.92\SZ & 3.72\SZ \\
SDAR-4B b32 & 4.04\SZ & \textbf{4.47}\SZ & 4.35\SZ & 4.24\SZ & 4.02\SZ \\
SDAR-1.7B b32 & 4.04\SZ & \textbf{4.43}\SZ & 4.33\SZ & 4.21\SZ & 3.92\SZ \\
LLaDA2.0-mini b32 & \textbf{2.25}\SZ & 2.28\SZ & 2.18\SZ & 2.09\SZ & 1.88\SZ \\
\bottomrule
\end{tabular}

\end{table}

\subsection{Methods outside the comparison}
\label{app:not-compared}
\paragraph{SimSD.} SimSD~\citep{cui_simsd_2026} drafts with a same-family sibling model
and verifies the drafted tokens in one forward under a plug-in attention mask that
lends them temporally valid contexts. That test approximates the sequential denoising
distribution instead of reproducing it, and its reported speedups come with changed
(on average improved) generation quality. It is therefore a lossy accelerator, not a
lossless speculative decoder, and its throughput is not commensurable with that of the
exact methods compared here; we do not compare against it.

\paragraph{MRP, from reported numbers.} MRP~\citep{xu_multi-token_2026} trains a
three-layer residual head to extrapolate the backbone's logits up to $K{=}3$ denoising
steps ahead and keeps a drafted token where the backbone's argmax agrees with it. It
reports results on the same SDAR-Chat 1.7B, 4B and 8B targets and the same four tasks
as this paper (its Table~2: block width $16$, greedy decoding, one H100), against a
one-token-per-step baseline, which is the draft-free decoder of
Section~\ref{sec:observations}. We did not run its code; Table~\ref{tab:mrp} sets its
reported speedups beside \sys{}'s $D8$ throughput ratio over our draft-free control
(block width $32$, length $256$, one card). The two are measured in different
implementations, so the comparison is between relative gains: \sys{} is ahead in all
twelve settings, by $1.8$--$3.2\times$, and uses $3.2$--$4.9\times$ fewer target
forwards than draft-free decoding, a quantity MRP does not report. Its acceptance test
is argmax equality, which has no distribution-preserving extension to sampling.

\begin{table}[htbp]
\centering\small
\caption{MRP's reported speedups next to \sys{}. MRP: speculative mode ($K{=}3$)
over its one-token-per-step baseline, Table~2 of \citet{xu_multi-token_2026}
(one H100, block width $16$, greedy). \sys{}: $D8$ throughput ratio over our
draft-free control (block width $32$, length $256$, one card, seed $0$), NFE reduction
in parentheses. MRP generates up to $4096$ tokens. Different implementations: relative
gains only.}
\label{tab:mrp}
\begin{tabular}{llcccc}
\toprule
target & method & GSM8K & MATH500 & HumanEval & MBPP \\
\midrule
\multirow{2}{*}{SDAR-1.7B} & MRP (reported) & $1.25\times$ & $1.23\times$ & $1.23\times$ & $1.17\times$ \\
 & \sys{} (ours) & $\mathbf{2.81\times}$ (3.6)\SZ & $\mathbf{2.48\times}$ (3.2)\SZ & $\mathbf{3.23\times}$ (4.2)\SZ & $\mathbf{3.74\times}$ (4.9)\SZ \\
\midrule
\multirow{2}{*}{SDAR-4B} & MRP (reported) & $1.36\times$ & $1.26\times$ & $1.35\times$ & $1.27\times$ \\
 & \sys{} (ours) & $\mathbf{3.27\times}$ (4.0)\SZ & $\mathbf{2.61\times}$ (3.3)\SZ & $\mathbf{3.49\times}$ (4.4)\SZ & $\mathbf{3.79\times}$ (4.9)\SZ \\
\midrule
\multirow{2}{*}{SDAR-8B} & MRP (reported) & $1.40\times$ & $1.39\times$ & $1.34\times$ & $1.34\times$ \\
 & \sys{} (ours) & $\mathbf{3.00\times}$ (3.9)\SZ & $\mathbf{2.47\times}$ (3.2)\SZ & $\mathbf{3.44\times}$ (4.3)\SZ & $\mathbf{3.65\times}$ (4.7)\SZ \\
\bottomrule
\end{tabular}

\end{table}

\subsection{Full-attention targets (LLaDA-8B, Dream-7B)}
LLaDA-8B-Instruct and Dream-7B-Instruct attend over the whole sequence, so the keys
and values of the prefix depend on the block being decoded and no exact prefix cache
exists for them. For the transfer rows of Table~\ref{tab:transfer} both run under an
approximate prefix cache of the Fast-dLLM kind~\citep{wu_fast-dllm_2026}: the prefix
is encoded once when a block opens and reused, unrefreshed, for every forward of that
block; the draft-free control of each task runs under the same cache, and every such
run is recorded as not distribution-preserving. The draft replicas use the same
concat layout as on SDAR, rotated at the positions of the block they draft for; for
LLaDA-8B this required threading explicit position ids through the released model
code (Dream's already accepts them), and the patched model reproduces the released
one exactly when positions are left implicit. Two checks were run in fp32 before the
benchmark: an appended, attention-isolated replica changes the decoded block's logits
by at most $10^{-4}$ (LLaDA) and $4{\times}10^{-4}$ (Dream), and a replica of the
true block reproduces that block's own logits to $10^{-4}$ and $2{\times}10^{-4}$,
so verification on these targets is exact relative to the cached decoder; the
approximation lives entirely in the cache, which is why these rows are reported
against the cached control and kept out of every lossless claim.

\subsection{Accounting details}
\label{app:accounting}
NFE counts target forwards only; oracle lookahead and drafter forwards are excluded
from it, and the drafter's own cost is charged to wall clock and therefore to the
throughput ratio. Throughput is measured per sample as generated
tokens over net wall time of the generation loop, summed over the $n$ samples of a task, and
reported as the ratio to the draft-free run of the same task, model and
temperature on the same machine. The draft-free run is a separate run with the same seed (at $T{>}0$, the same Gumbel
table), and every method of a task is normalized by it; each run is alone on its card. Margins under about $5\%$ are
reported but not called wins: within that margin we say that one method matches another. Oracle throughput subtracts
the oracle's lookahead forwards from wall time. Accuracy uses the task's standard
scorer on the generated text.

\section{Datasets}
\label{app:dataset}

GSM8K: the \texttt{main} test split (1319 problems), first 350, zero-shot chat
prompt, answer extracted after ``\#\#\#\#'' or as the last number. MBPP: the
\texttt{full} test split, first 350 problems (task ids 11--360), 3-shot with the
exemplars of the \texttt{prompt} split as in the MBPP paper and lm-eval, answers
wrapped in \texttt{[BEGIN]}/\texttt{[DONE]} and executed against the reference
tests. HumanEval: all $164$ problems of the test set, completion executed
against the reference tests. MATH500: the standard 500-problem subset of MATH,
first 350, exact-match on the boxed answer after normalization. LLaDA2.0-mini and the tables whose captions
say $n{=}100$ (Tables~\ref{tab:oracle}, \ref{tab:tsweep} and~\ref{tab:fourbed}) use the first $100$ of the same
order. For every task the samples are the exact prefix of the full-set order the released harness
uses, so runs at larger $n$ extend these rather than replace them. The prompt templates and stopping rules are those of the
released harness and ship with it.

\section{Extended results}
\label{app:ext}

\subsection{Extended comparison against related methods}
\label{app:related-table}
Table~\ref{tab:related} compares the published diffusion drafters on the axes this
paper argues are the relevant ones.

\begin{table}[htbp]
\centering\small
\caption{Speculative decoding for diffusion targets, on the axes this paper argues are
the relevant ones. ``Learns order'' = the commit order is predicted by a trained model
rather than read off stale confidence or fixed offline; ``multi-step'' = one drafted
proposal spans more than one commit; ``exact $T{>}0$'' = the committed distribution is
provably the target's at sampling temperature.}
\label{tab:related}
\setlength{\tabcolsep}{4pt}
\begin{tabular}{llccccc}
\toprule
method & drafter & size & learns order & multi-step & exact $T{>}0$ & single device \\
\midrule
Spiffy~\citeyearpar{agrawal_structuring_2026}   & offline graph      & ---    & \xmark & \cmark & \xmark & \cmark \\
SSD~\citeyearpar{gao_self_2025}                 & self (confidence)  & ---    & \xmark & \cmark & \xmark & \cmark \\
BlockSpec~\citeyearpar{pan_blockspec_2025}      & self (conf.\ tree) & ---    & \xmark & \cmark & \xmark & \cmark \\
SimSD~\citeyearpar{cui_simsd_2026}              & sibling model      & 1.7B   & \cmark & \cmark & approx.\ & \xmark \\
MRP~\citeyearpar{xu_multi-token_2026}           & residual head      & 3 layers & \xmark & \cmark & \xmark & \cmark \\
\textbf{\sys{} (ours)}                          & order model        & 0.62M  & \cmark & \cmark & \cmark & \cmark \\
\bottomrule
\end{tabular}
\end{table}

\subsection{Oracle grid}
\label{app:ext-oracle}
Table~\ref{tab:oracle} is the full live-oracle grid behind Figure~\ref{fig:oracle}(a):
NFE (reduction over the setting's draft-free NFE) and throughput ratio (lookahead subtracted) for the
position, token and joint oracles at $D \in \{8, 12, 20\}$ on the four SDAR-8B b32 settings
and on LLaDA2.0-mini MBPP. Honest throughput subtracts the oracle's lookahead
forwards from wall time (on SDAR-8B b32 the correction removes $82$--$89\%$ of the oracle's raw
wall) and divides by the setting's draft-free run.
NFE never rises with the budget for any oracle. At b32 the joint oracle's throughput still rises at
$D20$ on every task ($6.21$--$7.04\times$), the position oracle's on three of four (MATH500 peaks at $D12$,
$4.74\times$), and the token oracle's peaks at $D12$ or earlier ($2.84$--$4.49\times$). On LLaDA2.0-mini MBPP the
order is the same: the position oracle reaches $3.23$--$3.45\times$ and the token oracle $1.98$--$2.25\times$; the joint
oracle still rises at $D20$ ($4.40\times$).

\begin{table}[htbp]
\centering\small
\caption{Live-oracle grid, $T{=}0$, generation length 256, $n{=}100$, seed 0. Per
setting: NFE (reduction $\times$ over the setting's draft-free NFE) and tok/s ratio (lookahead subtracted)
over the setting's draft-free throughput, for the position, token and joint oracles at
$D8$/$D12$/$D20$; the draft-free column gives the baseline NFE and tok/s. Bold:
best throughput per setting.}
\label{tab:oracle}
\setlength{\tabcolsep}{3pt}
\resizebox{\linewidth}{!}{\begin{tabular}{llcccccccccc}
\toprule
 & & & \multicolumn{3}{c}{position oracle} & \multicolumn{3}{c}{token oracle} & \multicolumn{3}{c}{joint oracle} \\
\cmidrule(lr){4-6}\cmidrule(lr){7-9}\cmidrule(lr){10-12}
setting & metric & draft-free & $D8$ & $D12$ & $D20$ & $D8$ & $D12$ & $D20$ & $D8$ & $D12$ & $D20$ \\
\midrule
\multirow{2}{*}{SDAR-8B b32 GSM8K} & NFE & 271.1\SZ & 49.8 (5.4$\times$)\SZ & 42.5 (6.4$\times$)\SZ & 37.5 (7.2$\times$)\SZ & 68.3 (4.0$\times$)\SZ & 65.0 (4.2$\times$)\SZ & 61.6 (4.4$\times$)\SZ & 46.3 (5.8$\times$)\SZ & 37.9 (7.2$\times$)\SZ & 30.4 (8.9$\times$)\SZ \\
 & tok/s ratio & 16.3 tok/s\SZ & 4.62$\times$\SZ & 5.37$\times$\SZ & 5.51$\times$\SZ & 3.34$\times$\SZ & 3.42$\times$\SZ & 3.32$\times$\SZ & 5.00$\times$\SZ & 5.93$\times$\SZ & \textbf{6.78$\times$}\SZ \\
\multirow{2}{*}{SDAR-8B b32 MBPP} & NFE & 267.1\SZ & 46.8 (5.7$\times$)\SZ & 39.7 (6.7$\times$)\SZ & 32.8 (8.2$\times$)\SZ & 53.5 (5.0$\times$)\SZ & 47.3 (5.6$\times$)\SZ & 41.0 (6.5$\times$)\SZ & 43.5 (6.1$\times$)\SZ & 35.3 (7.6$\times$)\SZ & 27.4 (9.7$\times$)\SZ \\
 & tok/s ratio & 16.3 tok/s\SZ & 4.91$\times$\SZ & 5.43$\times$\SZ & 5.58$\times$\SZ & 4.23$\times$\SZ & 4.49$\times$\SZ & 4.35$\times$\SZ & 5.20$\times$\SZ & 6.06$\times$\SZ & \textbf{6.57$\times$}\SZ \\
\multirow{2}{*}{SDAR-8B b32 HumanEval} & NFE & 272.6\SZ & 47.1 (5.8$\times$)\SZ & 39.7 (6.9$\times$)\SZ & 32.9 (8.3$\times$)\SZ & 61.9 (4.4$\times$)\SZ & 57.4 (4.7$\times$)\SZ & 52.8 (5.2$\times$)\SZ & 45.7 (6.0$\times$)\SZ & 37.7 (7.2$\times$)\SZ & 29.9 (9.1$\times$)\SZ \\
 & tok/s ratio & 15.9 tok/s\SZ & 4.94$\times$\SZ & 5.85$\times$\SZ & 6.40$\times$\SZ & 3.85$\times$\SZ & 4.04$\times$\SZ & 3.98$\times$\SZ & 5.22$\times$\SZ & 6.15$\times$\SZ & \textbf{7.04$\times$}\SZ \\
\multirow{2}{*}{SDAR-8B b32 MATH500} & NFE & 272.9\SZ & 53.5 (5.1$\times$)\SZ & 47.7 (5.7$\times$)\SZ & 44.5 (6.1$\times$)\SZ & 81.9 (3.3$\times$)\SZ & 81.9 (3.3$\times$)\SZ & 81.0 (3.4$\times$)\SZ & 48.0 (5.7$\times$)\SZ & 39.8 (6.9$\times$)\SZ & 32.6 (8.4$\times$)\SZ \\
 & tok/s ratio & 16.2 tok/s\SZ & 4.33$\times$\SZ & 4.74$\times$\SZ & 4.65$\times$\SZ & 2.84$\times$\SZ & 2.77$\times$\SZ & 2.50$\times$\SZ & 4.69$\times$\SZ & 5.75$\times$\SZ & \textbf{6.21$\times$}\SZ \\
\midrule
\multirow{2}{*}{LLaDA2.0-mini b32 MBPP} & NFE & 270.4\SZ & 59.4 (4.6$\times$)\SZ & 53.7 (5.0$\times$)\SZ & 51.9 (5.2$\times$)\SZ & 90.1 (3.0$\times$)\SZ & 89.8 (3.0$\times$)\SZ & 88.8 (3.0$\times$)\SZ & 54.1 (5.0$\times$)\SZ & 46.7 (5.8$\times$)\SZ & 42.3 (6.4$\times$)\SZ \\
 & tok/s ratio & 6.6 tok/s\SZ & 3.23$\times$\SZ & 3.45$\times$\SZ & 3.39$\times$\SZ & 2.25$\times$\SZ & 2.21$\times$\SZ & 1.98$\times$\SZ & 3.66$\times$\SZ & 3.96$\times$\SZ & \textbf{4.40$\times$}\SZ \\
\bottomrule
\end{tabular}
}
\end{table}

\subsection{More models}
\label{app:ext-models}
The main comparison repeated on the three non-headline targets the recipe was trained
for: \textbf{SDAR-4B} at $T{=}0$ and $T{=}1$ (Tables~\ref{tab:main-4b}
and~\ref{tab:main-4b-t1}), \textbf{SDAR-1.7B} at both
(Tables~\ref{tab:main-17} and~\ref{tab:main-17-t1}), and \textbf{LLaDA2.0-mini}
at $T{=}0$ and $T{=}1$ (Tables~\ref{tab:main-mini} and~\ref{tab:main-mini-t1}). The recipe is unchanged in each case: one checkpoint, one
calibration, the methods and budgets of Table~\ref{tab:main-sdar}.

\begin{table}[htbp]
\centering\small
\caption{Main results, SDAR-4B b32, $T{=}0$ (format as
Table~\ref{tab:main-sdar}). \sys{} has the lowest NFE and the highest throughput
ratio in every ours-vs-baseline comparison. Seed $0$. Lower block: the threshold decoder ($\tau{=}0.9$), second machine, \THRN{main_sdar4b}, \sysSets{} rows and ratios as in Table~\ref{tab:main-sdar}.}
\label{tab:main-4b}
\setlength{\tabcolsep}{3pt}
\resizebox{\linewidth}{!}{\begin{tabular}{lcccccccc}
\toprule
\nfe{NFE}$\downarrow$ (\tps{tok/s}$\uparrow$) & \multicolumn{2}{c}{GSM8K} & \multicolumn{2}{c}{MBPP} & \multicolumn{2}{c}{HumanEval} & \multicolumn{2}{c}{MATH500} \\
 & \multicolumn{2}{c}{\footnotesize acc.\ 0.85--0.87} & \multicolumn{2}{c}{\footnotesize acc.\ 0.24--0.27} & \multicolumn{2}{c}{\footnotesize acc.\ 0.59--0.63} & \multicolumn{2}{c}{\footnotesize acc.\ 0.13--0.14} \\
\cmidrule(lr){2-3}\cmidrule(lr){4-5}\cmidrule(lr){6-7}\cmidrule(lr){8-9}
method & $D8$ & $D20$ & $D8$ & $D20$ & $D8$ & $D20$ & $D8$ & $D20$ \\
\midrule
no draft ($D0$) & \multicolumn{2}{c}{\nfe{270.8}$\downarrow$ (\tps{1.00}$\uparrow$)\SZ} & \multicolumn{2}{c}{\nfe{269.2}$\downarrow$ (\tps{1.00}$\uparrow$)\SZ} & \multicolumn{2}{c}{\nfe{271.4}$\downarrow$ (\tps{1.00}$\uparrow$)\SZ} & \multicolumn{2}{c}{\nfe{271.8}$\downarrow$ (\tps{1.00}$\uparrow$)\SZ} \\
\midrule
\rowcolor{oursrow}\textbf{\sys{} (ours)} & \nfe{\textbf{68.5}}\,(\tps{\underline{3.27}})\SZ & \nfe{\textbf{51.1}}\,(\tps{\underline{3.69}})\SZ & \nfe{\textbf{54.6}}\,(\tps{\underline{3.79}})\SZ & \nfe{\textbf{37.8}}\,(\tps{\underline{3.74}})\SZ & \nfe{\textbf{61.1}}\,(\tps{\underline{3.49}})\SZ & \nfe{\textbf{45.0}}\,(\tps{\underline{3.93}})\SZ & \nfe{\textbf{82.3}}\,(\tps{\underline{2.61}})\SZ & \nfe{\textbf{64.6}}\,(\tps{\underline{2.87}})\SZ \\
\rowcolor{oursrow}\sys{}-$\gamma$ (ours) & \nfe{\underline{68.5}}\,(\tps{\textbf{3.29}})\SZ & \nfe{\underline{53.2}}\,(\tps{\textbf{3.96}})\SZ & \nfe{\underline{54.8}}\,(\tps{\textbf{3.92}})\SZ & \nfe{\underline{39.7}}\,(\tps{\textbf{4.69}})\SZ & \nfe{\underline{61.3}}\,(\tps{\textbf{3.54}})\SZ & \nfe{\underline{46.6}}\,(\tps{\textbf{4.33}})\SZ & \nfe{\underline{82.6}}\,(\tps{\textbf{2.63}})\SZ & \nfe{\underline{68.2}}\,(\tps{\textbf{3.08}})\SZ \\
Spiffy graph & \nfe{91.2}\,(\tps{2.67})\SZ & \nfe{71.0}\,(\tps{3.24})\SZ & \nfe{82.4}\,(\tps{2.83})\SZ & \nfe{59.7}\,(\tps{3.14})\SZ & \nfe{88.9}\,(\tps{2.70})\SZ & \nfe{66.8}\,(\tps{3.44})\SZ & \nfe{97.7}\,(\tps{2.39})\SZ & \nfe{80.5}\,(\tps{2.79})\SZ \\
SSD free chain & \nfe{91.5}\,(\tps{2.67})\SZ & \nfe{83.7}\,(\tps{2.68})\SZ & \nfe{83.3}\,(\tps{2.82})\SZ & \nfe{74.5}\,(\tps{2.26})\SZ & \nfe{89.5}\,(\tps{2.74})\SZ & \nfe{78.6}\,(\tps{2.74})\SZ & \nfe{98.5}\,(\tps{2.38})\SZ & \nfe{91.5}\,(\tps{2.33})\SZ \\
BlockSpec tree & \nfe{136.3}\,(\tps{1.74})\SZ & \nfe{114.3}\,(\tps{1.80})\SZ & \nfe{135.1}\,(\tps{1.70})\SZ & \nfe{111.0}\,(\tps{1.35})\SZ & \nfe{135.9}\,(\tps{1.71})\SZ & \nfe{113.3}\,(\tps{1.75})\SZ & \nfe{136.1}\,(\tps{1.68})\SZ & \nfe{115.5}\,(\tps{1.75})\SZ \\
\specialrule{0.6pt}{2.5pt}{0pt}
\multicolumn{9}{c}{\emph{threshold decoding, $\tau{=}0.9$ (second machine, $n{=}350$, HumanEval $164$)}} \\
\rowcolor[gray]{0.90}$T{=}0$ & \multicolumn{2}{c}{acc.\ 0.84--0.86} & \multicolumn{2}{c}{acc.\ 0.24--0.27} & \multicolumn{2}{c}{acc.\ 0.58--0.61} & \multicolumn{2}{c}{acc.\ 0.13--0.15} \\
no draft & \multicolumn{2}{c}{\nfe{59.0}$\downarrow$ (\tps{4.55}$\uparrow$)\SZ} & \multicolumn{2}{c}{\nfe{40.6}$\downarrow$ (\tps{6.43}$\uparrow$)\SZ} & \multicolumn{2}{c}{\nfe{38.4}$\downarrow$ (\tps{6.91}$\uparrow$)\SZ} & \multicolumn{2}{c}{\nfe{85.4}$\downarrow$ (\tps{3.14}$\uparrow$)\SZ} \\
\midrule
\rowcolor{oursrow}\textbf{\sysSets{}-$\gamma$ (ours)} & \nfe{\textbf{30.9}}\,(\tps{\textbf{7.34}})\SZ & \nfe{\textbf{30.6}}\,(\tps{\textbf{7.38}})\SZ & \nfe{\textbf{22.0}}\,(\tps{\textbf{10.22}})\SZ & \nfe{\textbf{21.5}}\,(\tps{\textbf{10.40}})\SZ & \nfe{\textbf{21.8}}\,(\tps{\textbf{10.42}})\SZ & \nfe{\textbf{21.7}}\,(\tps{\textbf{10.42}})\SZ & \nfe{\textbf{45.0}}\,(\tps{\textbf{5.04}})\SZ & \nfe{\textbf{44.9}}\,(\tps{\textbf{5.07}})\SZ \\
\rowcolor{oursrow}\sysSets{}-$\gamma$, exact (ours) & \nfe{\underline{33.2}}\,(\tps{\underline{6.85}})\SZ & \nfe{\underline{33.0}}\,(\tps{\underline{6.83}})\SZ & \nfe{\underline{22.7}}\,(\tps{\underline{9.90}})\SZ & \nfe{\underline{22.3}}\,(\tps{\underline{10.00}})\SZ & \nfe{\underline{23.4}}\,(\tps{\underline{9.73}})\SZ & \nfe{\underline{23.2}}\,(\tps{\underline{9.66}})\SZ & \nfe{\underline{48.0}}\,(\tps{\underline{4.75}})\SZ & \nfe{\underline{47.4}}\,(\tps{\underline{4.77}})\SZ \\
Spiffy graph & \nfe{37.8}\,(\tps{6.27})\SZ & \nfe{37.7}\,(\tps{6.23})\SZ & \nfe{25.1}\,(\tps{9.48})\SZ & \nfe{24.8}\,(\tps{8.91})\SZ & \nfe{28.0}\,(\tps{8.59})\SZ & \nfe{28.0}\,(\tps{8.51})\SZ & \nfe{54.4}\,(\tps{4.34})\SZ & \nfe{53.8}\,(\tps{4.33})\SZ \\
SSD free chain & \nfe{37.9}\,(\tps{6.26})\SZ & \nfe{37.1}\,(\tps{5.01})\SZ & \nfe{24.8}\,(\tps{9.46})\SZ & \nfe{23.9}\,(\tps{6.43})\SZ & \nfe{28.1}\,(\tps{8.56})\SZ & \nfe{27.6}\,(\tps{6.77})\SZ & \nfe{54.2}\,(\tps{4.36})\SZ & \nfe{52.9}\,(\tps{3.44})\SZ \\
BlockSpec tree & \nfe{41.1}\,(\tps{5.62})\SZ & \nfe{38.4}\,(\tps{4.64})\SZ & \nfe{27.7}\,(\tps{8.33})\SZ & \nfe{25.4}\,(\tps{5.82})\SZ & \nfe{29.6}\,(\tps{7.93})\SZ & \nfe{28.3}\,(\tps{6.40})\SZ & \nfe{58.9}\,(\tps{3.92})\SZ & \nfe{55.3}\,(\tps{3.04})\SZ \\
\bottomrule
\end{tabular}
}
\end{table}

\begin{table}[htbp]
\centering\small
\caption{SDAR-4B b32 at $T{=}1$ (coupled-noise match; free chain also under
rejection sampling). Every comparison won on both columns. Seed $0$. Both blocks from the second machine, each ratio over that machine's draft-free run. Lower block: the threshold decoder ($\tau{=}0.9$), \THRN{main_sdar4b_t1}, \sysSets{} rows and ratios as in Table~\ref{tab:main-sdar}.}
\label{tab:main-4b-t1}
\setlength{\tabcolsep}{3pt}
\resizebox{\linewidth}{!}{\begin{tabular}{lcccccccc}
\toprule
\nfe{NFE}$\downarrow$ (\tps{tok/s}$\uparrow$) & \multicolumn{2}{c}{GSM8K} & \multicolumn{2}{c}{MBPP} & \multicolumn{2}{c}{HumanEval} & \multicolumn{2}{c}{MATH500} \\
 & \multicolumn{2}{c}{\footnotesize acc.\ 0.85--0.88} & \multicolumn{2}{c}{\footnotesize acc.\ 0.19--0.23} & \multicolumn{2}{c}{\footnotesize acc.\ 0.55--0.63} & \multicolumn{2}{c}{\footnotesize acc.\ 0.13--0.17} \\
\cmidrule(lr){2-3}\cmidrule(lr){4-5}\cmidrule(lr){6-7}\cmidrule(lr){8-9}
method & $D8$ & $D20$ & $D8$ & $D20$ & $D8$ & $D20$ & $D8$ & $D20$ \\
\midrule
no draft ($D0$) & \multicolumn{2}{c}{\nfe{270.8}$\downarrow$ (\tps{1.00}$\uparrow$)\SZ} & \multicolumn{2}{c}{\nfe{269.2}$\downarrow$ (\tps{1.00}$\uparrow$)\SZ} & \multicolumn{2}{c}{\nfe{271.4}$\downarrow$ (\tps{1.00}$\uparrow$)\SZ} & \multicolumn{2}{c}{\nfe{271.8}$\downarrow$ (\tps{1.00}$\uparrow$)\SZ} \\
\midrule
\rowcolor{oursrow}\textbf{\sys{} (ours)} & \nfe{\textbf{70.9}}\,(\tps{\underline{2.94}})\SZ & \nfe{\textbf{53.4}}\,(\tps{\underline{2.77}})\SZ & \nfe{\textbf{57.2}}\,(\tps{\underline{3.64}})\SZ & \nfe{\textbf{40.4}}\,(\tps{\underline{2.79}})\SZ & \nfe{\textbf{62.9}}\,(\tps{\underline{3.32}})\SZ & \nfe{\textbf{46.8}}\,(\tps{\underline{3.05}})\SZ & \nfe{\textbf{85.6}}\,(\tps{\underline{2.48}})\SZ & \nfe{\textbf{67.8}}\,(\tps{\underline{2.19}})\SZ \\
\rowcolor{oursrow}\sys{}-$\gamma$ (ours) & \nfe{\underline{71.3}}\,(\tps{\textbf{3.02}})\SZ & \nfe{\underline{55.9}}\,(\tps{\textbf{3.36}})\SZ & \nfe{\underline{57.8}}\,(\tps{\textbf{3.73}})\SZ & \nfe{\underline{42.9}}\,(\tps{\textbf{3.98}})\SZ & \nfe{\underline{63.4}}\,(\tps{\textbf{3.39}})\SZ & \nfe{\underline{48.7}}\,(\tps{\textbf{3.72}})\SZ & \nfe{\underline{86.5}}\,(\tps{\textbf{2.49}})\SZ & \nfe{\underline{72.0}}\,(\tps{\textbf{2.63}})\SZ \\
Spiffy graph & \nfe{94.3}\,(\tps{2.47})\SZ & \nfe{74.4}\,(\tps{2.53})\SZ & \nfe{84.9}\,(\tps{2.80})\SZ & \nfe{63.3}\,(\tps{2.23})\SZ & \nfe{90.2}\,(\tps{2.61})\SZ & \nfe{68.6}\,(\tps{2.63})\SZ & \nfe{102.4}\,(\tps{2.29})\SZ & \nfe{86.0}\,(\tps{2.15})\SZ \\
free chain, match & \nfe{94.6}\,(\tps{2.45})\SZ & \nfe{86.5}\,(\tps{1.99})\SZ & \nfe{85.8}\,(\tps{2.73})\SZ & \nfe{77.3}\,(\tps{1.64})\SZ & \nfe{90.7}\,(\tps{2.56})\SZ & \nfe{79.9}\,(\tps{2.04})\SZ & \nfe{103.3}\,(\tps{2.23})\SZ & \nfe{96.4}\,(\tps{1.74})\SZ \\
free chain, RS & \nfe{88.5}\,(\tps{1.95})\SZ & \nfe{83.4}\,(\tps{1.66})\SZ & \nfe{61.2}\,(\tps{2.97})\SZ & \nfe{50.7}\,(\tps{2.37})\SZ & \nfe{76.7}\,(\tps{2.24})\SZ & \nfe{69.7}\,(\tps{1.94})\SZ & \nfe{110.3}\,(\tps{1.55})\SZ & \nfe{110.2}\,(\tps{1.29})\SZ \\
BlockSpec tree & \nfe{138.9}\,(\tps{1.63})\SZ & \nfe{117.7}\,(\tps{1.37})\SZ & \nfe{137.9}\,(\tps{1.66})\SZ & \nfe{114.2}\,(\tps{1.00})\SZ & \nfe{137.7}\,(\tps{1.65})\SZ & \nfe{115.8}\,(\tps{1.32})\SZ & \nfe{140.8}\,(\tps{1.62})\SZ & \nfe{121.1}\,(\tps{1.27})\SZ \\
\specialrule{0.6pt}{2.5pt}{0pt}
\multicolumn{9}{c}{\emph{threshold decoding, $\tau{=}0.9$ (second machine, $n{=}350$, HumanEval $164$)}} \\
\rowcolor[gray]{0.90}$T{=}1$ & \multicolumn{2}{c}{acc.\ 0.81--0.85} & \multicolumn{2}{c}{acc.\ 0.19--0.21} & \multicolumn{2}{c}{acc.\ 0.57--0.60} & \multicolumn{2}{c}{acc.\ 0.15--0.17} \\
no draft & \multicolumn{2}{c}{\nfe{59.9}$\downarrow$ (\tps{4.41}$\uparrow$)\SZ} & \multicolumn{2}{c}{\nfe{41.2}$\downarrow$ (\tps{6.46}$\uparrow$)\SZ} & \multicolumn{2}{c}{\nfe{41.9}$\downarrow$ (\tps{6.35}$\uparrow$)\SZ} & \multicolumn{2}{c}{\nfe{86.8}$\downarrow$ (\tps{3.11}$\uparrow$)\SZ} \\
\midrule
\rowcolor{oursrow}\textbf{\sysSets{}-$\gamma$ (ours)} & \nfe{\textbf{34.3}}\,(\tps{\textbf{6.44}})\SZ & \nfe{\textbf{33.8}}\,(\tps{\textbf{6.50}})\SZ & \nfe{\textbf{25.8}}\,(\tps{\underline{8.53}})\SZ & \nfe{\textbf{25.9}}\,(\tps{\underline{8.46}})\SZ & \nfe{\textbf{25.7}}\,(\tps{\textbf{8.64}})\SZ & \nfe{\textbf{25.5}}\,(\tps{\textbf{8.50}})\SZ & \nfe{\textbf{49.8}}\,(\tps{\textbf{4.45}})\SZ & \nfe{\textbf{49.1}}\,(\tps{\textbf{4.46}})\SZ \\
\rowcolor{oursrow}\sysSets{}-$\gamma$, exact (ours) & \nfe{\underline{36.4}}\,(\tps{\underline{6.27}})\SZ & \nfe{\underline{35.9}}\,(\tps{\underline{6.30}})\SZ & \nfe{\underline{26.4}}\,(\tps{\textbf{8.57}})\SZ & \nfe{\underline{26.2}}\,(\tps{\textbf{8.56}})\SZ & \nfe{\underline{26.7}}\,(\tps{\underline{8.39}})\SZ & \nfe{\underline{26.6}}\,(\tps{\underline{8.37}})\SZ & \nfe{\underline{52.5}}\,(\tps{\underline{4.30}})\SZ & \nfe{\underline{51.6}}\,(\tps{\underline{4.34}})\SZ \\
Spiffy graph & \nfe{41.5}\,(\tps{5.66})\SZ & \nfe{41.4}\,(\tps{5.58})\SZ & \nfe{28.1}\,(\tps{8.36})\SZ & \nfe{28.3}\,(\tps{7.79})\SZ & \nfe{31.6}\,(\tps{7.43})\SZ & \nfe{31.8}\,(\tps{7.33})\SZ & \nfe{60.2}\,(\tps{3.90})\SZ & \nfe{59.4}\,(\tps{3.91})\SZ \\
SSD free chain & \nfe{41.1}\,(\tps{5.72})\SZ & \nfe{40.6}\,(\tps{4.65})\SZ & \nfe{27.8}\,(\tps{8.48})\SZ & \nfe{27.6}\,(\tps{5.45})\SZ & \nfe{31.3}\,(\tps{7.53})\SZ & \nfe{31.3}\,(\tps{5.93})\SZ & \nfe{60.1}\,(\tps{3.91})\SZ & \nfe{59.4}\,(\tps{2.96})\SZ \\
BlockSpec tree & \nfe{43.5}\,(\tps{5.30})\SZ & \nfe{41.3}\,(\tps{4.27})\SZ & \nfe{30.2}\,(\tps{7.62})\SZ & \nfe{28.6}\,(\tps{5.02})\SZ & \nfe{33.0}\,(\tps{7.08})\SZ & \nfe{31.7}\,(\tps{5.65})\SZ & \nfe{63.2}\,(\tps{3.62})\SZ & \nfe{60.3}\,(\tps{2.80})\SZ \\
\bottomrule
\end{tabular}
}
\end{table}

\begin{table}[htbp]
\centering\small
\caption{Main results, SDAR-1.7B b32, $T{=}0$ (format as Table~\ref{tab:main-sdar}).
\sys{} has the lowest NFE in every ours-vs-baseline comparison and the higher
throughput ratio in $23$ of $24$ (Spiffy takes MATH500 at $D20$); \sys{}-$\gamma$ wins all $24$ on both columns. Seed $0$. Lower block: the threshold decoder ($\tau{=}0.9$), second machine, \THRN{main_sdar17}, \sysSets{} rows and ratios as in Table~\ref{tab:main-sdar}.}
\label{tab:main-17}
\setlength{\tabcolsep}{3pt}
\resizebox{\linewidth}{!}{\begin{tabular}{lcccccccc}
\toprule
\nfe{NFE}$\downarrow$ (\tps{tok/s}$\uparrow$) & \multicolumn{2}{c}{GSM8K} & \multicolumn{2}{c}{MBPP} & \multicolumn{2}{c}{HumanEval} & \multicolumn{2}{c}{MATH500} \\
 & \multicolumn{2}{c}{\footnotesize acc.\ 0.69--0.71} & \multicolumn{2}{c}{\footnotesize acc.\ 0.29--0.30} & \multicolumn{2}{c}{\footnotesize acc.\ 0.45--0.47} & \multicolumn{2}{c}{\footnotesize acc.\ 0.09--0.11} \\
\cmidrule(lr){2-3}\cmidrule(lr){4-5}\cmidrule(lr){6-7}\cmidrule(lr){8-9}
method & $D8$ & $D20$ & $D8$ & $D20$ & $D8$ & $D20$ & $D8$ & $D20$ \\
\midrule
no draft ($D0$) & \multicolumn{2}{c}{\nfe{270.8}$\downarrow$ (\tps{1.00}$\uparrow$)\SZ} & \multicolumn{2}{c}{\nfe{269.2}$\downarrow$ (\tps{1.00}$\uparrow$)\SZ} & \multicolumn{2}{c}{\nfe{271.4}$\downarrow$ (\tps{1.00}$\uparrow$)\SZ} & \multicolumn{2}{c}{\nfe{271.8}$\downarrow$ (\tps{1.00}$\uparrow$)\SZ} \\
\midrule
\rowcolor{oursrow}\textbf{\sys{} (ours)} & \nfe{\textbf{75.6}}\,(\tps{\underline{2.81}})\SZ & \nfe{\textbf{58.8}}\,(\tps{\underline{3.01}})\SZ & \nfe{\textbf{55.4}}\,(\tps{\underline{3.74}})\SZ & \nfe{\textbf{38.8}}\,(\tps{\underline{4.28}})\SZ & \nfe{\underline{64.6}}\,(\tps{\underline{3.23}})\SZ & \nfe{\textbf{46.4}}\,(\tps{\underline{3.69}})\SZ & \nfe{\textbf{85.2}}\,(\tps{\underline{2.48}})\SZ & \nfe{\textbf{67.7}}\,(\tps{2.67})\SZ \\
\rowcolor{oursrow}\sys{}-$\gamma$ (ours) & \nfe{\underline{75.7}}\,(\tps{\textbf{2.85}})\SZ & \nfe{\underline{62.2}}\,(\tps{\textbf{3.30}})\SZ & \nfe{\underline{55.7}}\,(\tps{\textbf{3.83}})\SZ & \nfe{\underline{40.4}}\,(\tps{\textbf{5.02}})\SZ & \nfe{\textbf{64.5}}\,(\tps{\textbf{3.30}})\SZ & \nfe{\underline{48.1}}\,(\tps{\textbf{4.09}})\SZ & \nfe{\underline{85.6}}\,(\tps{\textbf{2.55}})\SZ & \nfe{\underline{71.6}}\,(\tps{\textbf{2.85}})\SZ \\
Spiffy graph & \nfe{94.8}\,(\tps{2.49})\SZ & \nfe{76.6}\,(\tps{2.99})\SZ & \nfe{85.3}\,(\tps{2.79})\SZ & \nfe{64.4}\,(\tps{3.55})\SZ & \nfe{89.9}\,(\tps{2.64})\SZ & \nfe{69.3}\,(\tps{3.27})\SZ & \nfe{99.4}\,(\tps{2.39})\SZ & \nfe{82.8}\,(\tps{\underline{2.75}})\SZ \\
SSD free chain & \nfe{94.8}\,(\tps{2.47})\SZ & \nfe{88.1}\,(\tps{2.54})\SZ & \nfe{91.4}\,(\tps{2.54})\SZ & \nfe{76.5}\,(\tps{3.00})\SZ & \nfe{89.4}\,(\tps{2.66})\SZ & \nfe{81.8}\,(\tps{2.78})\SZ & \nfe{99.3}\,(\tps{2.43})\SZ & \nfe{92.7}\,(\tps{2.46})\SZ \\
BlockSpec tree & \nfe{136.5}\,(\tps{1.68})\SZ & \nfe{115.7}\,(\tps{1.72})\SZ & \nfe{135.1}\,(\tps{1.71})\SZ & \nfe{111.5}\,(\tps{1.81})\SZ & \nfe{135.2}\,(\tps{1.70})\SZ & \nfe{112.7}\,(\tps{1.83})\SZ & \nfe{136.0}\,(\tps{1.68})\SZ & \nfe{115.8}\,(\tps{1.75})\SZ \\
\specialrule{0.6pt}{2.5pt}{0pt}
\multicolumn{9}{c}{\emph{threshold decoding, $\tau{=}0.9$ (second machine, $n{=}350$, HumanEval $164$)}} \\
\rowcolor[gray]{0.90}$T{=}0$ & \multicolumn{2}{c}{acc.\ 0.67--0.69} & \multicolumn{2}{c}{acc.\ 0.28--0.30} & \multicolumn{2}{c}{acc.\ 0.42--0.45} & \multicolumn{2}{c}{acc.\ 0.09--0.10} \\
no draft & \multicolumn{2}{c}{\nfe{86.4}$\downarrow$ (\tps{3.13}$\uparrow$)\SZ} & \multicolumn{2}{c}{\nfe{40.7}$\downarrow$ (\tps{6.46}$\uparrow$)\SZ} & \multicolumn{2}{c}{\nfe{48.5}$\downarrow$ (\tps{5.50}$\uparrow$)\SZ} & \multicolumn{2}{c}{\nfe{95.7}$\downarrow$ (\tps{2.83}$\uparrow$)\SZ} \\
\midrule
\rowcolor{oursrow}\textbf{\sysSets{}-$\gamma$ (ours)} & \nfe{\textbf{43.4}}\,(\tps{\textbf{5.04}})\SZ & \nfe{\textbf{42.6}}\,(\tps{\textbf{5.08}})\SZ & \nfe{\textbf{22.5}}\,(\tps{\textbf{9.90}})\SZ & \nfe{\textbf{22.2}}\,(\tps{\textbf{9.91}})\SZ & \nfe{\textbf{26.7}}\,(\tps{\textbf{8.36}})\SZ & \nfe{\textbf{26.2}}\,(\tps{\textbf{8.47}})\SZ & \nfe{\textbf{51.6}}\,(\tps{\textbf{4.23}})\SZ & \nfe{\textbf{51.3}}\,(\tps{\textbf{4.26}})\SZ \\
\rowcolor{oursrow}\sysSets{}-$\gamma$, exact (ours) & \nfe{\underline{46.2}}\,(\tps{\underline{4.74}})\SZ & \nfe{\underline{45.5}}\,(\tps{\underline{4.73}})\SZ & \nfe{\underline{23.3}}\,(\tps{\underline{9.63}})\SZ & \nfe{\underline{22.9}}\,(\tps{\underline{9.65}})\SZ & \nfe{\underline{27.9}}\,(\tps{\underline{8.02}})\SZ & \nfe{\underline{27.5}}\,(\tps{\underline{8.09}})\SZ & \nfe{\underline{55.2}}\,(\tps{\underline{4.05}})\SZ & \nfe{\underline{54.5}}\,(\tps{\underline{4.06}})\SZ \\
Spiffy graph & \nfe{53.8}\,(\tps{4.36})\SZ & \nfe{52.7}\,(\tps{4.37})\SZ & \nfe{26.3}\,(\tps{8.98})\SZ & \nfe{25.7}\,(\tps{9.02})\SZ & \nfe{33.7}\,(\tps{7.10})\SZ & \nfe{33.0}\,(\tps{7.05})\SZ & \nfe{62.8}\,(\tps{3.73})\SZ & \nfe{61.5}\,(\tps{3.75})\SZ \\
SSD free chain & \nfe{53.1}\,(\tps{4.39})\SZ & \nfe{51.9}\,(\tps{4.22})\SZ & \nfe{25.7}\,(\tps{9.18})\SZ & \nfe{25.2}\,(\tps{8.92})\SZ & \nfe{33.3}\,(\tps{7.17})\SZ & \nfe{32.6}\,(\tps{6.98})\SZ & \nfe{62.4}\,(\tps{3.73})\SZ & \nfe{60.6}\,(\tps{3.64})\SZ \\
BlockSpec tree & \nfe{58.3}\,(\tps{3.92})\SZ & \nfe{54.3}\,(\tps{3.73})\SZ & \nfe{28.3}\,(\tps{8.10})\SZ & \nfe{26.5}\,(\tps{7.88})\SZ & \nfe{35.7}\,(\tps{6.51})\SZ & \nfe{33.4}\,(\tps{6.26})\SZ & \nfe{67.6}\,(\tps{3.38})\SZ & \nfe{63.1}\,(\tps{3.23})\SZ \\
\bottomrule
\end{tabular}
}
\end{table}

\begin{table}[htbp]
\centering\small
\caption{SDAR-1.7B b32 at $T{=}1$ (coupled-noise match; free chain also under
rejection sampling). Lowest NFE in every comparison and the higher throughput ratio
in all but one of $32$ (Spiffy takes MATH500 at $D20$); \sys{}-$\gamma$ wins every comparison on both columns. Seed $0$. Lower block: the threshold decoder ($\tau{=}0.9$), second machine, \THRN{main_sdar17_t1}, \sysSets{} rows and ratios as in Table~\ref{tab:main-sdar}.}
\label{tab:main-17-t1}
\setlength{\tabcolsep}{3pt}
\resizebox{\linewidth}{!}{\begin{tabular}{lcccccccc}
\toprule
\nfe{NFE}$\downarrow$ (\tps{tok/s}$\uparrow$) & \multicolumn{2}{c}{GSM8K} & \multicolumn{2}{c}{MBPP} & \multicolumn{2}{c}{HumanEval} & \multicolumn{2}{c}{MATH500} \\
 & \multicolumn{2}{c}{\footnotesize acc.\ 0.69--0.74} & \multicolumn{2}{c}{\footnotesize acc.\ 0.21--0.29} & \multicolumn{2}{c}{\footnotesize acc.\ 0.37--0.45} & \multicolumn{2}{c}{\footnotesize acc.\ 0.09--0.12} \\
\cmidrule(lr){2-3}\cmidrule(lr){4-5}\cmidrule(lr){6-7}\cmidrule(lr){8-9}
method & $D8$ & $D20$ & $D8$ & $D20$ & $D8$ & $D20$ & $D8$ & $D20$ \\
\midrule
no draft ($D0$) & \multicolumn{2}{c}{\nfe{270.8}$\downarrow$ (\tps{1.00}$\uparrow$)\SZ} & \multicolumn{2}{c}{\nfe{269.2}$\downarrow$ (\tps{1.00}$\uparrow$)\SZ} & \multicolumn{2}{c}{\nfe{271.4}$\downarrow$ (\tps{1.00}$\uparrow$)\SZ} & \multicolumn{2}{c}{\nfe{271.8}$\downarrow$ (\tps{1.00}$\uparrow$)\SZ} \\
\midrule
\rowcolor{oursrow}\textbf{\sys{} (ours)} & \nfe{\textbf{79.0}}\,(\tps{\textbf{2.67}})\SZ & \nfe{\textbf{62.3}}\,(\tps{\underline{2.86}})\SZ & \nfe{\textbf{58.0}}\,(\tps{\underline{3.52}})\SZ & \nfe{\textbf{40.9}}\,(\tps{\underline{3.99}})\SZ & \nfe{\textbf{66.7}}\,(\tps{\underline{3.14}})\SZ & \nfe{\textbf{48.5}}\,(\tps{\underline{3.49}})\SZ & \nfe{\textbf{88.9}}\,(\tps{\underline{2.39}})\SZ & \nfe{\textbf{71.7}}\,(\tps{2.48})\SZ \\
\rowcolor{oursrow}\sys{}-$\gamma$ (ours) & \nfe{\underline{79.9}}\,(\tps{\underline{2.67}})\SZ & \nfe{\underline{66.6}}\,(\tps{\textbf{3.04}})\SZ & \nfe{\underline{58.7}}\,(\tps{\textbf{3.54}})\SZ & \nfe{\underline{43.7}}\,(\tps{\textbf{4.54}})\SZ & \nfe{\underline{67.5}}\,(\tps{\textbf{3.15}})\SZ & \nfe{\underline{51.5}}\,(\tps{\textbf{3.85}})\SZ & \nfe{\underline{89.8}}\,(\tps{\textbf{2.40}})\SZ & \nfe{\underline{76.3}}\,(\tps{\textbf{2.78}})\SZ \\
Spiffy graph & \nfe{99.5}\,(\tps{2.37})\SZ & \nfe{82.7}\,(\tps{2.73})\SZ & \nfe{87.7}\,(\tps{2.68})\SZ & \nfe{67.3}\,(\tps{3.33})\SZ & \nfe{93.0}\,(\tps{2.54})\SZ & \nfe{72.4}\,(\tps{3.15})\SZ & \nfe{105.5}\,(\tps{2.28})\SZ & \nfe{90.1}\,(\tps{\underline{2.57}})\SZ \\
free chain, match & \nfe{98.9}\,(\tps{2.36})\SZ & \nfe{92.8}\,(\tps{2.44})\SZ & \nfe{93.6}\,(\tps{2.47})\SZ & \nfe{78.7}\,(\tps{2.79})\SZ & \nfe{92.9}\,(\tps{2.50})\SZ & \nfe{84.6}\,(\tps{2.62})\SZ & \nfe{104.8}\,(\tps{2.31})\SZ & \nfe{99.0}\,(\tps{2.29})\SZ \\
free chain, RS & \nfe{99.4}\,(\tps{1.62})\SZ & \nfe{97.0}\,(\tps{1.47})\SZ & \nfe{71.5}\,(\tps{2.29})\SZ & \nfe{67.4}\,(\tps{2.13})\SZ & \nfe{87.6}\,(\tps{1.87})\SZ & \nfe{81.7}\,(\tps{1.74})\SZ & \nfe{115.7}\,(\tps{1.43})\SZ & \nfe{115.5}\,(\tps{1.26})\SZ \\
BlockSpec tree & \nfe{141.5}\,(\tps{1.60})\SZ & \nfe{121.3}\,(\tps{1.67})\SZ & \nfe{137.4}\,(\tps{1.62})\SZ & \nfe{114.2}\,(\tps{1.72})\SZ & \nfe{138.5}\,(\tps{1.66})\SZ & \nfe{116.3}\,(\tps{1.74})\SZ & \nfe{141.7}\,(\tps{1.63})\SZ & \nfe{122.3}\,(\tps{1.70})\SZ \\
\specialrule{0.6pt}{2.5pt}{0pt}
\multicolumn{9}{c}{\emph{threshold decoding, $\tau{=}0.9$ (second machine, $n{=}350$, HumanEval $164$)}} \\
\rowcolor[gray]{0.90}$T{=}1$ & \multicolumn{2}{c}{acc.\ 0.69--0.72} & \multicolumn{2}{c}{acc.\ 0.21--0.23} & \multicolumn{2}{c}{acc.\ 0.35--0.37} & \multicolumn{2}{c}{acc.\ 0.10--0.12} \\
no draft & \multicolumn{2}{c}{\nfe{88.4}$\downarrow$ (\tps{3.01}$\uparrow$)\SZ} & \multicolumn{2}{c}{\nfe{42.7}$\downarrow$ (\tps{6.24}$\uparrow$)\SZ} & \multicolumn{2}{c}{\nfe{52.5}$\downarrow$ (\tps{5.11}$\uparrow$)\SZ} & \multicolumn{2}{c}{\nfe{97.4}$\downarrow$ (\tps{2.74}$\uparrow$)\SZ} \\
\midrule
\rowcolor{oursrow}\textbf{\sysSets{}-$\gamma$ (ours)} & \nfe{\textbf{48.3}}\,(\tps{\textbf{4.32}})\SZ & \nfe{\textbf{47.7}}\,(\tps{\textbf{4.36}})\SZ & \nfe{\textbf{26.8}}\,(\tps{\underline{8.09}})\SZ & \nfe{\textbf{26.4}}\,(\tps{\underline{8.00}})\SZ & \nfe{\textbf{30.1}}\,(\tps{\textbf{7.20}})\SZ & \nfe{\textbf{30.0}}\,(\tps{\textbf{7.16}})\SZ & \nfe{\textbf{56.8}}\,(\tps{\textbf{3.74}})\SZ & \nfe{\textbf{55.9}}\,(\tps{\textbf{3.77}})\SZ \\
\rowcolor{oursrow}\sysSets{}-$\gamma$, exact (ours) & \nfe{\underline{51.0}}\,(\tps{\underline{4.24}})\SZ & \nfe{\underline{50.2}}\,(\tps{\underline{4.25}})\SZ & \nfe{\underline{27.4}}\,(\tps{\textbf{8.10}})\SZ & \nfe{\underline{27.2}}\,(\tps{\textbf{8.10}})\SZ & \nfe{\underline{31.8}}\,(\tps{\underline{6.97}})\SZ & \nfe{\underline{31.8}}\,(\tps{\underline{6.92}})\SZ & \nfe{\underline{59.8}}\,(\tps{\underline{3.68}})\SZ & \nfe{\underline{58.8}}\,(\tps{\underline{3.74}})\SZ \\
Spiffy graph & \nfe{59.5}\,(\tps{3.86})\SZ & \nfe{58.9}\,(\tps{3.86})\SZ & \nfe{29.6}\,(\tps{8.00})\SZ & \nfe{29.1}\,(\tps{7.95})\SZ & \nfe{38.3}\,(\tps{6.24})\SZ & \nfe{38.0}\,(\tps{6.17})\SZ & \nfe{70.0}\,(\tps{3.34})\SZ & \nfe{68.8}\,(\tps{3.31})\SZ \\
SSD free chain & \nfe{58.7}\,(\tps{3.91})\SZ & \nfe{58.1}\,(\tps{3.72})\SZ & \nfe{29.1}\,(\tps{8.06})\SZ & \nfe{28.7}\,(\tps{7.81})\SZ & \nfe{37.9}\,(\tps{6.27})\SZ & \nfe{37.5}\,(\tps{6.06})\SZ & \nfe{69.6}\,(\tps{3.33})\SZ & \nfe{68.7}\,(\tps{3.14})\SZ \\
BlockSpec tree & \nfe{63.0}\,(\tps{3.53})\SZ & \nfe{59.9}\,(\tps{3.29})\SZ & \nfe{31.4}\,(\tps{7.29})\SZ & \nfe{29.7}\,(\tps{6.91})\SZ & \nfe{39.9}\,(\tps{5.77})\SZ & \nfe{38.2}\,(\tps{5.51})\SZ & \nfe{72.7}\,(\tps{3.09})\SZ & \nfe{69.8}\,(\tps{2.84})\SZ \\
\bottomrule
\end{tabular}
}
\end{table}

\paragraph{The smallest target is where verification costs most.} On SDAR-1.7B Spiffy takes one
throughput column from \sys{} at each temperature: MATH500 at $D20$, the setting with the lowest acceptance,
by $0.08$ at $T{=}0$ and $0.09$ at $T{=}1$. The NFE ordering is untouched -- \sys{} takes $18\%$ and $20\%$
fewer forwards than Spiffy there -- so the gap is verification cost, not drafting: a $1.7$B forward is cheap
enough that the drafted states it verifies, up to $(1{+}D)L$ positions, are a large share of it.
\sys{}-$\gamma$, whose floor drops low-mass drafts, verifies fewer and takes the column back at both
temperatures ($2.85\times$ and $2.78\times$ against Spiffy's $2.75\times$ and $2.57\times$).

\begin{table}[htbp]
\centering\small
\caption{Main results, LLaDA2.0-mini b32 (MoE, exact prefix-KV cache), $T{=}0$,
$n{=}100$, seed $0$. \sys{} has the lowest NFE in all $24$
comparisons and the higher throughput ratio in $23$; the exception is MBPP at $D8$, where Spiffy is ahead by $0.5\%$.
\sys{}-$\gamma$ is \sys{} here, since its calibrated floor is $0$ (Appendix~\ref{app:gamma-calib}). Mini's per-round acceptance is low for every drafter
($1.05$--$3.66$ accepted tokens per drafted round across the four tasks and both budgets), so the
wall-clock margins are thin. Lower block: the threshold decoder ($\tau{=}0.9$), second machine, \THRN{main_mini}, \sysSets{} rows and ratios as in Table~\ref{tab:main-sdar}. On this MoE target the threshold decoder's output depends on the batch shape (expert routing), so the ``exact'' rows keep its tokens only up to that noise; their accuracies stay within the spread of the other rows.}
\label{tab:main-mini}
\setlength{\tabcolsep}{3pt}
\resizebox{\linewidth}{!}{\begin{tabular}{lcccccccc}
\toprule
\nfe{NFE}$\downarrow$ (\tps{tok/s}$\uparrow$) & \multicolumn{2}{c}{GSM8K} & \multicolumn{2}{c}{MBPP} & \multicolumn{2}{c}{HumanEval} & \multicolumn{2}{c}{MATH500} \\
 & \multicolumn{2}{c}{\footnotesize acc.\ 0.55--0.64} & \multicolumn{2}{c}{\footnotesize acc.\ 0.23--0.31} & \multicolumn{2}{c}{\footnotesize acc.\ 0.66--0.68} & \multicolumn{2}{c}{\footnotesize acc.\ 0.10--0.12} \\
\cmidrule(lr){2-3}\cmidrule(lr){4-5}\cmidrule(lr){6-7}\cmidrule(lr){8-9}
method & $D8$ & $D20$ & $D8$ & $D20$ & $D8$ & $D20$ & $D8$ & $D20$ \\
\midrule
no draft ($D0$) & \multicolumn{2}{c}{\nfe{271.6}$\downarrow$ (\tps{1.00}$\uparrow$)\SZ} & \multicolumn{2}{c}{\nfe{270.4}$\downarrow$ (\tps{1.00}$\uparrow$)\SZ} & \multicolumn{2}{c}{\nfe{269.3}$\downarrow$ (\tps{1.00}$\uparrow$)\SZ} & \multicolumn{2}{c}{\nfe{270.2}$\downarrow$ (\tps{1.00}$\uparrow$)\SZ} \\
\midrule
\rowcolor{oursrow}\textbf{\sys{} (ours)} & \nfe{\textbf{83.6}}\,(\tps{\textbf{2.32}})\SZ & \nfe{\textbf{65.3}}\,(\tps{\textbf{2.55}})\SZ & \nfe{\textbf{93.5}}\,(\tps{\underline{2.11}})\SZ & \nfe{\textbf{74.2}}\,(\tps{\textbf{2.31}})\SZ & \nfe{\textbf{88.3}}\,(\tps{\textbf{2.23}})\SZ & \nfe{\textbf{65.9}}\,(\tps{\textbf{2.55}})\SZ & \nfe{\textbf{84.8}}\,(\tps{\textbf{2.38}})\SZ & \nfe{\textbf{66.3}}\,(\tps{\textbf{2.73}})\SZ \\
Spiffy graph & \nfe{97.1}\,(\tps{2.13})\SZ & \nfe{\underline{79.5}}\,(\tps{\underline{2.33}})\SZ & \nfe{102.0}\,(\tps{\textbf{2.12}})\SZ & \nfe{\underline{85.3}}\,(\tps{\underline{2.19}})\SZ & \nfe{97.3}\,(\tps{2.13})\SZ & \nfe{\underline{78.5}}\,(\tps{\underline{2.48}})\SZ & \nfe{98.7}\,(\tps{2.09})\SZ & \nfe{\underline{80.8}}\,(\tps{\underline{2.38}})\SZ \\
SSD free chain & \nfe{\underline{96.1}}\,(\tps{\underline{2.15}})\SZ & \nfe{90.6}\,(\tps{2.07})\SZ & \nfe{\underline{101.2}}\,(\tps{2.00})\SZ & \nfe{94.5}\,(\tps{2.09})\SZ & \nfe{\underline{96.9}}\,(\tps{\underline{2.16}})\SZ & \nfe{89.2}\,(\tps{2.15})\SZ & \nfe{\underline{98.3}}\,(\tps{\underline{2.16}})\SZ & \nfe{91.3}\,(\tps{2.12})\SZ \\
BlockSpec tree & \nfe{136.9}\,(\tps{1.45})\SZ & \nfe{116.9}\,(\tps{1.46})\SZ & \nfe{136.4}\,(\tps{1.45})\SZ & \nfe{116.9}\,(\tps{1.49})\SZ & \nfe{135.0}\,(\tps{1.48})\SZ & \nfe{115.5}\,(\tps{1.51})\SZ & \nfe{135.8}\,(\tps{1.47})\SZ & \nfe{116.1}\,(\tps{1.52})\SZ \\
\specialrule{0.6pt}{2.5pt}{0pt}
\multicolumn{9}{c}{\emph{threshold decoding, $\tau{=}0.9$ (second machine, $n{=}100$)}} \\
\rowcolor[gray]{0.90}$T{=}0$ & \multicolumn{2}{c}{acc.\ 0.53--0.63} & \multicolumn{2}{c}{acc.\ 0.23--0.29} & \multicolumn{2}{c}{acc.\ 0.60--0.62} & \multicolumn{2}{c}{acc.\ 0.09--0.12} \\
no draft & \multicolumn{2}{c}{\nfe{94.0}$\downarrow$ (\tps{2.69}$\uparrow$)\SZ} & \multicolumn{2}{c}{\nfe{75.2}$\downarrow$ (\tps{3.18}$\uparrow$)\SZ} & \multicolumn{2}{c}{\nfe{29.9}$\downarrow$ (\tps{7.40}$\uparrow$)\SZ} & \multicolumn{2}{c}{\nfe{93.0}$\downarrow$ (\tps{2.78}$\uparrow$)\SZ} \\
\midrule
\rowcolor{oursrow}\textbf{\sysSets{}-$\gamma$ (ours)} & \nfe{\textbf{44.0}}\,(\tps{\textbf{3.96}})\SZ & \nfe{\textbf{39.0}}\,(\tps{\textbf{3.94}})\SZ & \nfe{\textbf{34.7}}\,(\tps{\textbf{4.68}})\SZ & \nfe{\textbf{31.7}}\,(\tps{\textbf{4.48}})\SZ & \nfe{\textbf{19.0}}\,(\tps{\underline{8.41}})\SZ & \nfe{\textbf{18.3}}\,(\tps{\underline{8.15}})\SZ & \nfe{\textbf{43.7}}\,(\tps{\textbf{4.19}})\SZ & \nfe{\textbf{38.9}}\,(\tps{\textbf{4.08}})\SZ \\
\rowcolor{oursrow}\sysSets{}-$\gamma$, exact (ours) & \nfe{\underline{48.2}}\,(\tps{\underline{3.83}})\SZ & \nfe{\underline{47.2}}\,(\tps{\underline{3.78}})\SZ & \nfe{\underline{41.8}}\,(\tps{\underline{4.07}})\SZ & \nfe{\underline{39.5}}\,(\tps{\underline{4.26}})\SZ & \nfe{\underline{20.4}}\,(\tps{\textbf{8.53}})\SZ & \nfe{\underline{19.9}}\,(\tps{\textbf{8.68}})\SZ & \nfe{\underline{48.8}}\,(\tps{\underline{3.92}})\SZ & \nfe{\underline{47.7}}\,(\tps{\underline{3.99}})\SZ \\
Spiffy graph & \nfe{57.7}\,(\tps{3.27})\SZ & \nfe{56.6}\,(\tps{3.11})\SZ & \nfe{47.9}\,(\tps{3.80})\SZ & \nfe{45.7}\,(\tps{3.65})\SZ & \nfe{24.6}\,(\tps{7.45})\SZ & \nfe{24.2}\,(\tps{7.15})\SZ & \nfe{56.2}\,(\tps{3.55})\SZ & \nfe{55.1}\,(\tps{3.32})\SZ \\
SSD free chain & \nfe{56.9}\,(\tps{3.46})\SZ & \nfe{55.6}\,(\tps{3.17})\SZ & \nfe{46.7}\,(\tps{3.88})\SZ & \nfe{47.4}\,(\tps{3.52})\SZ & \nfe{24.8}\,(\tps{7.38})\SZ & \nfe{24.5}\,(\tps{7.16})\SZ & \nfe{56.8}\,(\tps{3.44})\SZ & \nfe{55.5}\,(\tps{3.31})\SZ \\
BlockSpec tree & \nfe{63.3}\,(\tps{2.96})\SZ & \nfe{58.5}\,(\tps{2.90})\SZ & \nfe{52.4}\,(\tps{3.43})\SZ & \nfe{47.5}\,(\tps{3.43})\SZ & \nfe{25.1}\,(\tps{7.32})\SZ & \nfe{24.7}\,(\tps{6.72})\SZ & \nfe{61.8}\,(\tps{3.16})\SZ & \nfe{58.2}\,(\tps{2.94})\SZ \\
\bottomrule
\end{tabular}
}
\end{table}

\begin{table}[htbp]
\centering\small
\caption{Main results, LLaDA2.0-mini b32 at $T{=}1$, $n{=}100$, seed $0$; every drafter uses coupled-noise match, and the
free chain also runs under RS. Of the $32$ comparisons, \sys{} has the lowest NFE in all $32$ and the higher throughput ratio in $30$;
the exceptions are on MBPP, where the free chain under match is ahead at $D8$ ($1.94$ against $1.90\times$) and Spiffy at
$D20$ ($2.11$ against $2.03\times$). Lower block: the threshold decoder at $T{=}1$, as in
Table~\ref{tab:main-mini}, \THRN{main_mini_t1}.}
\label{tab:main-mini-t1}
\setlength{\tabcolsep}{3pt}
\resizebox{\linewidth}{!}{\begin{tabular}{lcccccccc}
\toprule
\nfe{NFE}$\downarrow$ (\tps{tok/s}$\uparrow$) & \multicolumn{2}{c}{GSM8K} & \multicolumn{2}{c}{MBPP} & \multicolumn{2}{c}{HumanEval} & \multicolumn{2}{c}{MATH500} \\
 & \multicolumn{2}{c}{\footnotesize acc.\ 0.53--0.63} & \multicolumn{2}{c}{\footnotesize acc.\ 0.30--0.42} & \multicolumn{2}{c}{\footnotesize acc.\ 0.65--0.70} & \multicolumn{2}{c}{\footnotesize acc.\ 0.08--0.10} \\
\cmidrule(lr){2-3}\cmidrule(lr){4-5}\cmidrule(lr){6-7}\cmidrule(lr){8-9}
method & $D8$ & $D20$ & $D8$ & $D20$ & $D8$ & $D20$ & $D8$ & $D20$ \\
\midrule
no draft ($D0$) & \multicolumn{2}{c}{\nfe{271.6}$\downarrow$ (\tps{1.00}$\uparrow$)\SZ} & \multicolumn{2}{c}{\nfe{270.4}$\downarrow$ (\tps{1.00}$\uparrow$)\SZ} & \multicolumn{2}{c}{\nfe{269.3}$\downarrow$ (\tps{1.00}$\uparrow$)\SZ} & \multicolumn{2}{c}{\nfe{270.2}$\downarrow$ (\tps{1.00}$\uparrow$)\SZ} \\
\midrule
\rowcolor{oursrow}\textbf{\sys{} (ours)} & \nfe{\textbf{86.6}}\,(\tps{\textbf{2.17}})\SZ & \nfe{\textbf{67.8}}\,(\tps{\textbf{2.56}})\SZ & \nfe{\textbf{100.8}}\,(\tps{\underline{1.90}})\SZ & \nfe{\textbf{80.5}}\,(\tps{\underline{2.03}})\SZ & \nfe{\textbf{89.0}}\,(\tps{\textbf{2.29}})\SZ & \nfe{\textbf{67.4}}\,(\tps{\textbf{2.55}})\SZ & \nfe{\textbf{87.4}}\,(\tps{\textbf{2.28}})\SZ & \nfe{\textbf{69.3}}\,(\tps{\textbf{2.53}})\SZ \\
Spiffy graph & \nfe{102.5}\,(\tps{2.02})\SZ & \nfe{\underline{86.0}}\,(\tps{\underline{2.22}})\SZ & \nfe{\underline{106.6}}\,(\tps{1.88})\SZ & \nfe{\underline{90.7}}\,(\tps{\textbf{2.11}})\SZ & \nfe{98.5}\,(\tps{\underline{2.20}})\SZ & \nfe{\underline{79.5}}\,(\tps{\underline{2.41}})\SZ & \nfe{102.7}\,(\tps{2.00})\SZ & \nfe{\underline{86.1}}\,(\tps{\underline{2.21}})\SZ \\
free chain, match & \nfe{\underline{101.1}}\,(\tps{\underline{2.07}})\SZ & \nfe{95.3}\,(\tps{2.01})\SZ & \nfe{106.9}\,(\tps{\textbf{1.94}})\SZ & \nfe{99.3}\,(\tps{1.89})\SZ & \nfe{\underline{97.6}}\,(\tps{2.17})\SZ & \nfe{89.7}\,(\tps{2.15})\SZ & \nfe{\underline{101.8}}\,(\tps{\underline{2.06}})\SZ & \nfe{95.9}\,(\tps{2.00})\SZ \\
free chain, RS & \nfe{111.6}\,(\tps{1.55})\SZ & \nfe{109.9}\,(\tps{1.37})\SZ & \nfe{125.4}\,(\tps{1.38})\SZ & \nfe{122.3}\,(\tps{1.14})\SZ & \nfe{114.6}\,(\tps{1.59})\SZ & \nfe{115.2}\,(\tps{1.34})\SZ & \nfe{113.0}\,(\tps{1.62})\SZ & \nfe{113.1}\,(\tps{1.40})\SZ \\
BlockSpec tree & \nfe{142.1}\,(\tps{1.45})\SZ & \nfe{122.5}\,(\tps{1.44})\SZ & \nfe{140.5}\,(\tps{1.40})\SZ & \nfe{121.0}\,(\tps{1.42})\SZ & \nfe{135.7}\,(\tps{1.47})\SZ & \nfe{116.4}\,(\tps{1.45})\SZ & \nfe{140.6}\,(\tps{1.40})\SZ & \nfe{121.4}\,(\tps{1.49})\SZ \\
\specialrule{0.6pt}{2.5pt}{0pt}
\multicolumn{9}{c}{\emph{threshold decoding, $\tau{=}0.9$ (second machine, $n{=}100$)}} \\
\rowcolor[gray]{0.90}$T{=}1$ & \multicolumn{2}{c}{acc.\ 0.54--0.60} & \multicolumn{2}{c}{acc.\ 0.31--0.40} & \multicolumn{2}{c}{acc.\ 0.62--0.68} & \multicolumn{2}{c}{acc.\ 0.07--0.10} \\
no draft & \multicolumn{2}{c}{\nfe{94.6}$\downarrow$ (\tps{2.68}$\uparrow$)\SZ} & \multicolumn{2}{c}{\nfe{76.0}$\downarrow$ (\tps{3.03}$\uparrow$)\SZ} & \multicolumn{2}{c}{\nfe{30.4}$\downarrow$ (\tps{7.58}$\uparrow$)\SZ} & \multicolumn{2}{c}{\nfe{90.5}$\downarrow$ (\tps{2.89}$\uparrow$)\SZ} \\
\midrule
\rowcolor{oursrow}\textbf{\sysSets{}-$\gamma$ (ours)} & \nfe{\textbf{48.7}}\,(\tps{\textbf{3.70}})\SZ & \nfe{\textbf{43.4}}\,(\tps{\textbf{3.55}})\SZ & \nfe{\textbf{41.7}}\,(\tps{\underline{3.66}})\SZ & \nfe{\textbf{36.5}}\,(\tps{\underline{3.74}})\SZ & \nfe{\textbf{20.2}}\,(\tps{\underline{8.43}})\SZ & \nfe{\textbf{19.4}}\,(\tps{\underline{7.91}})\SZ & \nfe{\textbf{47.5}}\,(\tps{\underline{3.69}})\SZ & \nfe{\textbf{42.4}}\,(\tps{\textbf{3.78}})\SZ \\
\rowcolor{oursrow}\sysSets{}-$\gamma$, exact (ours) & \nfe{\underline{53.4}}\,(\tps{\underline{3.54}})\SZ & \nfe{\underline{52.4}}\,(\tps{\underline{3.49}})\SZ & \nfe{\underline{45.5}}\,(\tps{\textbf{3.70}})\SZ & \nfe{\underline{43.8}}\,(\tps{\textbf{3.91}})\SZ & \nfe{\underline{21.4}}\,(\tps{\textbf{8.47}})\SZ & \nfe{\underline{21.2}}\,(\tps{\textbf{8.19}})\SZ & \nfe{\underline{52.2}}\,(\tps{\textbf{3.71}})\SZ & \nfe{\underline{51.4}}\,(\tps{\underline{3.76}})\SZ \\
Spiffy graph & \nfe{63.8}\,(\tps{3.01})\SZ & \nfe{63.3}\,(\tps{2.75})\SZ & \nfe{53.7}\,(\tps{3.19})\SZ & \nfe{52.7}\,(\tps{3.03})\SZ & \nfe{26.1}\,(\tps{7.42})\SZ & \nfe{26.0}\,(\tps{6.95})\SZ & \nfe{61.7}\,(\tps{3.17})\SZ & \nfe{61.7}\,(\tps{2.88})\SZ \\
SSD free chain & \nfe{65.1}\,(\tps{2.99})\SZ & \nfe{62.7}\,(\tps{2.82})\SZ & \nfe{53.8}\,(\tps{3.24})\SZ & \nfe{50.1}\,(\tps{3.25})\SZ & \nfe{26.2}\,(\tps{7.36})\SZ & \nfe{25.9}\,(\tps{7.05})\SZ & \nfe{61.3}\,(\tps{3.31})\SZ & \nfe{61.0}\,(\tps{3.01})\SZ \\
BlockSpec tree & \nfe{68.3}\,(\tps{2.71})\SZ & \nfe{64.2}\,(\tps{2.54})\SZ & \nfe{57.7}\,(\tps{2.90})\SZ & \nfe{53.6}\,(\tps{2.75})\SZ & \nfe{26.5}\,(\tps{6.91})\SZ & \nfe{26.0}\,(\tps{6.57})\SZ & \nfe{64.9}\,(\tps{2.87})\SZ & \nfe{60.9}\,(\tps{2.74})\SZ \\
\bottomrule
\end{tabular}
}
\end{table}

\begin{table}[htbp]
\centering\small
\caption{LLaDA2.0-mini b32 at $T{=}1$ (MBPP, $n{=}100$, seed 0): \sys{} against a drafter tuned for this setting
alone; the full $T{=}1$ comparison is Table~\ref{tab:main-mini-t1}. \sys{} takes $4.4$--$8.1$ fewer forwards at
$1.05$--$1.08\times$ the tuned drafter's throughput. First machine, unpadded.}
\label{tab:mini-t1}
\begin{tabular}{lccc}
\toprule
 & \multicolumn{3}{c}{\nfe{NFE}$\downarrow$ (\tps{tok/s}$\uparrow$)} \\
\cmidrule(lr){2-4}
drafter & $D8$ & $D12$ & $D20$ \\
\midrule
\sys{} & \nfe{\textbf{100.8}}\,(\tps{\textbf{1.90}})\SZ & \nfe{\textbf{92.9}}\,(\tps{\textbf{2.04}})\SZ & \nfe{\textbf{80.5}}\,(\tps{\textbf{2.03}})\SZ \\
setting-specific drafter & \nfe{\underline{108.9}}\,(\tps{\underline{1.82}})\SZ & \nfe{\underline{97.7}}\,(\tps{\underline{1.88}})\SZ & \nfe{\underline{84.9}}\,(\tps{\underline{1.90}})\SZ \\
\bottomrule
\end{tabular}

\end{table}

\paragraph{The full transfer matrix.} Table~\ref{tab:transfer-full} points the SDAR-8B and LLaDA2.0-mini
checkpoints at all six targets, at $D8$. On SDAR-8B and LLaDA2.0-mini, the two checkpoints' own targets, the native drafter has the lowest NFE on
all four tasks. The transferred checkpoint costs $3.5$--$7.7\%$ more NFE on SDAR-8B and $3.0$--$11.3\%$ on LLaDA2.0-mini,
whose tokenizer differs from SDAR's; there the SDAR-8B drafter still takes the higher GSM8K throughput ($2.38$ against
$2.36\times$). On SDAR-4B and SDAR-1.7B, which neither checkpoint was trained on, the two land within $2.4\%$ of each other
in NFE on every task, the LLaDA2.0-mini one ahead on three of four. On LLaDA-8B and Dream-7B they land within $5\%$, the
SDAR-8B one ahead in NFE on all four tasks, at $1.57$--$2.56\times$ the throughput of draft-free decoding. What the drafter learns is largely a property of block-diffusion
decoding rather than of one target.

\begin{table}[htbp]
\centering\footnotesize
\caption{The SDAR-8B and LLaDA2.0-mini checkpoints on all six targets, $T{=}0$, b32, $D8$,
unpadded, seed $0$, $n{=}350$ per task (HumanEval $164$). $^{*}$marks the target's native drafter, the one trained on it; all other
rows are transferred with no retraining or recalibration. \tps{Ratio} is against that
target's own draft-free control, best per column within a target in bold. $^{\dagger}$no exact prefix cache, as in
Table~\ref{tab:transfer}.}
\label{tab:transfer-full}
\setlength{\tabcolsep}{4pt}
\resizebox{\linewidth}{!}{\begin{tabular}{llcccc}
\toprule
 & & \multicolumn{4}{c}{\nfe{NFE}$\downarrow$ (\tps{tok/s}$\uparrow$) at $D8$} \\
\cmidrule(lr){3-6}
target & drafter & GSM8K & MBPP & HumanEval & MATH500 \\
\midrule
\multirow{2}{*}{SDAR-8B} & SDAR-8B$^{*}$ & \nfe{\textbf{70.1}}\,(\tps{\textbf{3.00}})\SZ & \nfe{\textbf{56.7}}\,(\tps{\textbf{3.65}})\SZ & \nfe{\textbf{62.5}}\,(\tps{\textbf{3.44}})\SZ & \nfe{\textbf{84.6}}\,(\tps{\textbf{2.47}})\SZ \\
 & mini & \nfe{\underline{73.0}}\,(\tps{\underline{2.88}})\SZ & \nfe{\underline{58.7}}\,(\tps{\underline{3.54}})\SZ & \nfe{\underline{67.3}}\,(\tps{\underline{3.16}})\SZ & \nfe{\underline{88.1}}\,(\tps{\underline{2.41}})\SZ \\
\midrule
\multirow{2}{*}{SDAR-4B} & SDAR-8B & \nfe{\underline{69.2}}\,(\tps{\textbf{3.21}})\SZ & \nfe{\textbf{54.5}}\,(\tps{\textbf{3.82}})\SZ & \nfe{\underline{63.1}}\,(\tps{\underline{3.44}})\SZ & \nfe{\underline{83.8}}\,(\tps{\underline{2.52}})\SZ \\
 & mini & \nfe{\textbf{69.1}}\,(\tps{\underline{3.18}})\SZ & \nfe{\underline{55.6}}\,(\tps{\underline{3.82}})\SZ & \nfe{\textbf{61.6}}\,(\tps{\textbf{3.51}})\SZ & \nfe{\textbf{82.8}}\,(\tps{\textbf{2.58}})\SZ \\
\midrule
\multirow{2}{*}{SDAR-1.7B} & SDAR-8B & \nfe{\underline{76.1}}\,(\tps{\underline{2.76}})\SZ & \nfe{\textbf{55.3}}\,(\tps{\underline{3.72}})\SZ & \nfe{\underline{66.1}}\,(\tps{\underline{3.13}})\SZ & \nfe{\underline{86.3}}\,(\tps{\underline{2.44}})\SZ \\
 & mini & \nfe{\textbf{75.8}}\,(\tps{\textbf{2.77}})\SZ & \nfe{\underline{56.5}}\,(\tps{\textbf{3.75}})\SZ & \nfe{\textbf{64.9}}\,(\tps{\textbf{3.22}})\SZ & \nfe{\textbf{84.5}}\,(\tps{\textbf{2.53}})\SZ \\
\midrule
\multirow{2}{*}{LLaDA2.0-mini} & SDAR-8B & \nfe{\underline{85.9}}\,(\tps{\textbf{2.38}})\SZ & \nfe{\underline{101.7}}\,(\tps{\underline{1.99}})\SZ & \nfe{\underline{97.6}}\,(\tps{\underline{2.14}})\SZ & \nfe{\underline{87.1}}\,(\tps{\underline{2.23}})\SZ \\
 & mini$^{*}$ & \nfe{\textbf{82.6}}\,(\tps{\underline{2.36}})\SZ & \nfe{\textbf{94.1}}\,(\tps{\textbf{2.13}})\SZ & \nfe{\textbf{87.7}}\,(\tps{\textbf{2.22}})\SZ & \nfe{\textbf{84.6}}\,(\tps{\textbf{2.38}})\SZ \\
\midrule
\multirow{2}{*}{LLaDA-8B$^{\dagger}$} & SDAR-8B & \nfe{\textbf{79.9}}\,(\tps{\textbf{1.87}})\SZ & \nfe{\textbf{54.8}}\,(\tps{\textbf{2.39}})\SZ & \nfe{\textbf{83.4}}\,(\tps{\textbf{1.70}})\SZ & \nfe{\textbf{83.8}}\,(\tps{\underline{1.69}})\SZ \\
 & mini & \nfe{\underline{83.0}}\,(\tps{\underline{1.80}})\SZ & \nfe{\underline{57.3}}\,(\tps{\underline{2.17}})\SZ & \nfe{\underline{87.5}}\,(\tps{\underline{1.57}})\SZ & \nfe{\underline{86.7}}\,(\tps{\textbf{1.70}})\SZ \\
\midrule
\multirow{2}{*}{Dream-7B$^{\dagger}$} & SDAR-8B & \nfe{\textbf{66.3}}\,(\tps{\textbf{2.05}})\SZ & \nfe{\textbf{51.2}}\,(\tps{\underline{2.53}})\SZ & \nfe{\textbf{56.6}}\,(\tps{\textbf{2.37}})\SZ & \nfe{\textbf{77.1}}\,(\tps{\textbf{1.80}})\SZ \\
 & mini & \nfe{\underline{68.0}}\,(\tps{\underline{1.99}})\SZ & \nfe{\underline{52.5}}\,(\tps{\textbf{2.56}})\SZ & \nfe{\underline{58.3}}\,(\tps{\underline{2.30}})\SZ & \nfe{\underline{79.5}}\,(\tps{\underline{1.74}})\SZ \\
\bottomrule
\end{tabular}
}
\end{table}

\subsection{More block sizes}
\label{app:ext-blocks}
Table~\ref{tab:fourbed} gives the recipe transfer on MBPP across four target and block-width settings (SDAR-8B and
LLaDA2.0-mini at b32 and b16): with the recipe frozen
(same architecture, training procedure and calibration), the frozen-recipe \sys{} drafter
takes fewer forwards than a drafter tuned separately for each setting, on every setting and budget, by
$0.5$--$0.8$ NFE on SDAR b32, $1.1$--$1.2$ on SDAR b16, $0.8$--$2.4$ on mini b32 and $5.9$--$7.0$ on mini b16. Its
throughput is higher in $10$ of the $12$ cells; the tuned drafter is ahead on SDAR b32 at $D8$ ($0.07$) and on mini b32 at
$D12$ ($0.12$).

\begin{table}[htbp]
\centering\small
\caption{Recipe transfer across widths and targets (MBPP, $T{=}0$, $n{=}100$): the
single \sys{} recipe vs.\ a drafter tuned separately for each setting. Seed $0$, unpadded, first machine; each
ratio over the draft-free run of the same setting.}
\label{tab:fourbed}
\begin{tabular}{lccc}
\toprule
 & \multicolumn{3}{c}{\nfe{NFE}$\downarrow$ (\tps{tok/s}$\uparrow$)} \\
\cmidrule(lr){2-4}
setting & $D8$ & $D12$ & $D20$ \\
\midrule
SDAR-8B b32, \sys{} & \nfe{\textbf{55.6}}\,(\tps{\underline{3.68}})\SZ & \nfe{\textbf{47.6}}\,(\tps{\textbf{3.77}})\SZ & \nfe{\textbf{38.5}}\,(\tps{\textbf{3.24}})\SZ \\
SDAR-8B b32, tuned & \nfe{\underline{56.1}}\,(\tps{\textbf{3.74}})\SZ & \nfe{\underline{48.1}}\,(\tps{\underline{3.72}})\SZ & \nfe{\underline{39.3}}\,(\tps{\underline{3.20}})\SZ \\
\midrule
SDAR-8B b16, \sys{} & \nfe{\textbf{60.9}}\,(\tps{\textbf{3.43}})\SZ & \nfe{\textbf{58.9}}\,(\tps{\textbf{3.53}})\SZ & \nfe{\textbf{44.2}}\,(\tps{\textbf{4.43}})\SZ \\
SDAR-8B b16, tuned & \nfe{\underline{62.1}}\,(\tps{\underline{3.38}})\SZ & \nfe{\underline{60.0}}\,(\tps{\underline{3.42}})\SZ & \nfe{\underline{45.3}}\,(\tps{\underline{4.27}})\SZ \\
\midrule
mini b32, \sys{} & \nfe{\textbf{93.5}}\,(\tps{\textbf{2.11}})\SZ & \nfe{\textbf{85.5}}\,(\tps{\underline{2.16}})\SZ & \nfe{\textbf{74.2}}\,(\tps{\textbf{2.31}})\SZ \\
mini b32, tuned & \nfe{\underline{95.9}}\,(\tps{\underline{2.09}})\SZ & \nfe{\underline{86.5}}\,(\tps{\textbf{2.28}})\SZ & \nfe{\underline{75.0}}\,(\tps{\underline{2.25}})\SZ \\
\midrule
mini b16, \sys{} & \nfe{\textbf{91.6}}\,(\tps{\textbf{1.96}})\SZ & \nfe{\textbf{81.8}}\,(\tps{\textbf{2.05}})\SZ & \nfe{\textbf{72.2}}\,(\tps{\textbf{2.27}})\SZ \\
mini b16, tuned & \nfe{\underline{98.6}}\,(\tps{\underline{1.79}})\SZ & \nfe{\underline{87.8}}\,(\tps{\underline{1.93}})\SZ & \nfe{\underline{78.1}}\,(\tps{\underline{2.08}})\SZ \\
\bottomrule
\end{tabular}

\end{table}

\subsection{Generation length}
\label{app:ext-lengths}
Table~\ref{tab:genlen} repeats the $T{=}1$ MBPP comparison at generation lengths $64$
and $512$ beside the $256$ of the main text, each against its own draft-free control.
At length $256$, \sys{}'s throughput ratio over draft-free decoding at $D8$ is $3.48\times$ and
\sys{}-$\gamma$'s best ratio $3.90\times$ (at $D20$); against Spiffy at $D8$, \sys{} takes $30\%$ fewer
forwards at $1.27\times$ its throughput, and its NFE reduction over draft-free decoding is $4.5\times$.
At length $512$ and $D8$, \sys{} reaches $4.03\times$ and \sys{}-$\gamma$ $4.18\times$, and \sys{}'s NFE reduction grows to
$5.2\times$; accuracy stays within $0.03$ of the draft-free control at both lengths.
The gain at length $512$ comes from the end-of-text positions after the answer
(Appendix~\ref{app:stop-answer}): stopped at the answer's end, \sys{} reaches $1.96\times$ at $D8$ at both lengths.

At the short end, length $64$ is two blocks. \sys{} still has the lowest NFE in every column, but at $D20$
Spiffy is slightly faster ($1.71$ against $1.68\times$): with two blocks, the large $D20$ forwards cost more than their
extra commits save. \sys{}-$\gamma$, whose floor computes fewer verification positions, is flat in the budget ($2.24\times$
at $D8$, $2.21\times$ at $D20$) and the fastest row at both: at short generations the floored recipe is the one to
deploy. At length
$512$ no such caveat applies, and \sys{} has both the lowest NFE and the highest
throughput ratio against every baseline at both budgets ($3.49\times$ at $D20$, against $3.35\times$ for the free
chain with RS; \sys{}-$\gamma$ $4.68\times$).

\begin{table}[htbp]
\centering\footnotesize
\caption{Generation length, SDAR-8B b32 MBPP at $T{=}1$: \nfe{NFE per sample}
(\tps{throughput ratio} over that length's own draft-free control), best in bold.
The $256$ columns are the first machine's seed-$0$ runs of Table~\ref{tab:main-sdar}'s $T{=}1$ MBPP setting, so their ratios differ from that table's.}
\label{tab:genlen}
\setlength{\tabcolsep}{4pt}
\begin{tabular}{lcccccc}
\toprule
 & \multicolumn{2}{c}{length 64} & \multicolumn{2}{c}{length 256} & \multicolumn{2}{c}{length 512} \\
\cmidrule(lr){2-3}\cmidrule(lr){4-5}\cmidrule(lr){6-7}
method & $D8$ & $D20$ & $D8$ & $D20$ & $D8$ & $D20$ \\
\midrule
no draft ($D0$) & \multicolumn{2}{c}{\nfe{77.2}$\downarrow$ (\tps{1.00}$\uparrow$)\SZ} & \multicolumn{2}{c}{\nfe{269.2}$\downarrow$ (\tps{1.00}$\uparrow$)\SZ} & \multicolumn{2}{c}{\nfe{525.2}$\downarrow$ (\tps{1.00}$\uparrow$)\SZ} \\
\midrule
\rowcolor{oursrow}\sys{} (ours) & \nfe{\textbf{27.7}}\,(\tps{\underline{2.21}})\SZ & \nfe{\textbf{22.3}}\,(\tps{1.68})\SZ & \nfe{\textbf{60.3}}\,(\tps{\underline{3.48}})\SZ & \nfe{\textbf{42.7}}\,(\tps{\underline{2.96}})\SZ & \nfe{\textbf{100.6}}\,(\tps{\underline{4.03}})\SZ & \nfe{\textbf{67.0}}\,(\tps{\underline{3.49}})\SZ \\
\rowcolor{oursrow}\sys{}-$\gamma$ (ours) & \nfe{\underline{28.6}}\,(\tps{\textbf{2.24}})\SZ & \nfe{\underline{25.4}}\,(\tps{\textbf{2.21}})\SZ & \nfe{\underline{61.4}}\,(\tps{\textbf{3.59}})\SZ & \nfe{\underline{46.2}}\,(\tps{\textbf{3.90}})\SZ & \nfe{\underline{101.7}}\,(\tps{\textbf{4.18}})\SZ & \nfe{\underline{70.4}}\,(\tps{\textbf{4.68}})\SZ \\
Spiffy graph & \nfe{31.0}\,(\tps{2.20})\SZ & \nfe{27.0}\,(\tps{\underline{1.71}})\SZ & \nfe{86.5}\,(\tps{2.73})\SZ & \nfe{65.0}\,(\tps{2.50})\SZ & \nfe{158.6}\,(\tps{2.91})\SZ & \nfe{113.1}\,(\tps{2.76})\SZ \\
free chain, match & \nfe{30.9}\,(\tps{2.14})\SZ & \nfe{29.5}\,(\tps{1.38})\SZ & \nfe{87.4}\,(\tps{2.67})\SZ & \nfe{79.0}\,(\tps{1.86})\SZ & \nfe{159.5}\,(\tps{2.86})\SZ & \nfe{143.2}\,(\tps{1.88})\SZ \\
free chain, RS & \nfe{31.4}\,(\tps{1.62})\SZ & \nfe{30.2}\,(\tps{1.27})\SZ & \nfe{64.8}\,(\tps{2.79})\SZ & \nfe{53.0}\,(\tps{2.56})\SZ & \nfe{105.0}\,(\tps{3.33})\SZ & \nfe{77.1}\,(\tps{3.35})\SZ \\
BlockSpec tree & \nfe{41.9}\,(\tps{1.54})\SZ & \nfe{36.4}\,(\tps{1.07})\SZ & \nfe{138.2}\,(\tps{1.67})\SZ & \nfe{115.0}\,(\tps{1.17})\SZ & \nfe{266.2}\,(\tps{1.65})\SZ & \nfe{219.1}\,(\tps{1.13})\SZ \\
\bottomrule
\end{tabular}

\end{table}

\subsection{Stopping at the end of the answer}
\label{app:stop-answer}
Every table decodes the full generation length. After its answer the target fills the remaining positions with
end-of-text tokens, which every drafter predicts easily, and the fixed-length ratios count them. On SDAR-8B the answer
takes $162$ of the $256$ tokens on GSM8K, $55$ on MBPP, $112$ on HumanEval and $251$ on MATH500. Table~\ref{tab:stop-answer}
stops the one-commit runs of Table~\ref{tab:main-sdar} at the block that holds the first end-of-text token, from their
logged rounds; blocks decode left to right, so no committed token changes. The ratios fall most where the answers are
shortest, but \sys{} and \sys{}-$\gamma$ keep the lowest NFE and \sys{}-$\gamma$ the highest throughput in all $16$
settings, at $1.04$--$1.38\times$ Spiffy's.

\begin{table}[htbp]
\centering\footnotesize
\caption{The one-commit runs of Table~\ref{tab:main-sdar} (SDAR-8B b32, mean of seeds $0$--$2$, second machine,
$n{=}350$ per task, HumanEval $164$), stopped at the first end-of-text token: \nfe{NFE per sample} (\tps{throughput ratio}
over draft-free decoding stopped the same way), best in bold, second best underlined. ``answer'': the draft-free run's
mean answer length in tokens, of $256$.}
\label{tab:stop-answer}
\setlength{\tabcolsep}{3pt}
\resizebox{\linewidth}{!}{{\begin{tabular}{lcccccccc}
\toprule
\nfe{NFE}$\downarrow$ (\tps{tok/s}$\uparrow$) & \multicolumn{2}{c}{GSM8K} & \multicolumn{2}{c}{MBPP} & \multicolumn{2}{c}{HumanEval} & \multicolumn{2}{c}{MATH500} \\
\cmidrule(lr){2-3}\cmidrule(lr){4-5}\cmidrule(lr){6-7}\cmidrule(lr){8-9}
method & $D8$ & $D20$ & $D8$ & $D20$ & $D8$ & $D20$ & $D8$ & $D20$ \\
\midrule
\rowcolor[gray]{0.90}$T{=}0$ & \multicolumn{2}{c}{\footnotesize answer 162 tokens} & \multicolumn{2}{c}{\footnotesize answer 55 tokens} & \multicolumn{2}{c}{\footnotesize answer 112 tokens} & \multicolumn{2}{c}{\footnotesize answer 251 tokens} \\
no draft ($D0$) & \multicolumn{2}{c}{\nfe{178.2}$\downarrow$ (\tps{1.00}$\uparrow$)\SZ} & \multicolumn{2}{c}{\nfe{72.4}$\downarrow$ (\tps{1.00}$\uparrow$)\SZ} & \multicolumn{2}{c}{\nfe{128.0}$\downarrow$ (\tps{1.00}$\uparrow$)\SZ} & \multicolumn{2}{c}{\nfe{266.3}$\downarrow$ (\tps{1.00}$\uparrow$)\SZ} \\
\midrule
\rowcolor{oursrow}\textbf{\sys{} (ours)} & \nfe{\underline{55.6}}\,(\tps{\underline{2.48}})\SZ & \nfe{\textbf{43.6}}\,(\tps{\underline{1.90}})\SZ & \nfe{\textbf{25.9}}\,(\tps{\underline{2.02}})\SZ & \nfe{\textbf{21.1}}\,(\tps{1.30})\SZ & \nfe{\underline{40.0}}\,(\tps{\underline{2.52}})\SZ & \nfe{\textbf{31.0}}\,(\tps{\underline{1.93}})\SZ & \nfe{\textbf{83.7}}\,(\tps{\underline{2.48}})\SZ & \nfe{\textbf{65.5}}\,(\tps{\underline{1.84}})\SZ \\
\rowcolor{oursrow}\sys{}-$\gamma$ (ours) & \nfe{\textbf{55.6}}\,(\tps{\textbf{2.53}})\SZ & \nfe{\underline{45.8}}\,(\tps{\textbf{2.26}})\SZ & \nfe{\underline{26.2}}\,(\tps{\textbf{2.05}})\SZ & \nfe{\underline{23.4}}\,(\tps{\textbf{1.73}})\SZ & \nfe{\textbf{39.8}}\,(\tps{\textbf{2.54}})\SZ & \nfe{\underline{33.0}}\,(\tps{\textbf{2.26}})\SZ & \nfe{\underline{83.7}}\,(\tps{\textbf{2.50}})\SZ & \nfe{\underline{68.5}}\,(\tps{\textbf{2.23}})\SZ \\
Spiffy graph & \nfe{66.6}\,(\tps{2.29})\SZ & \nfe{54.0}\,(\tps{1.83})\SZ & \nfe{28.0}\,(\tps{1.95})\SZ & \nfe{24.5}\,(\tps{\underline{1.33}})\SZ & \nfe{47.3}\,(\tps{2.35})\SZ & \nfe{38.2}\,(\tps{1.81})\SZ & \nfe{99.0}\,(\tps{2.28})\SZ & \nfe{81.3}\,(\tps{1.78})\SZ \\
SSD free chain & \nfe{66.2}\,(\tps{2.26})\SZ & \nfe{61.5}\,(\tps{1.45})\SZ & \nfe{28.1}\,(\tps{1.95})\SZ & \nfe{26.8}\,(\tps{1.10})\SZ & \nfe{46.5}\,(\tps{2.34})\SZ & \nfe{43.8}\,(\tps{1.45})\SZ & \nfe{99.0}\,(\tps{2.27})\SZ & \nfe{92.6}\,(\tps{1.41})\SZ \\
BlockSpec tree & \nfe{90.6}\,(\tps{1.63})\SZ & \nfe{77.1}\,(\tps{1.09})\SZ & \nfe{36.7}\,(\tps{1.48})\SZ & \nfe{31.7}\,(\tps{0.88})\SZ & \nfe{64.5}\,(\tps{1.66})\SZ & \nfe{54.6}\,(\tps{1.11})\SZ & \nfe{134.0}\,(\tps{1.64})\SZ & \nfe{113.9}\,(\tps{1.09})\SZ \\
\specialrule{0.6pt}{2.5pt}{0pt}
\rowcolor[gray]{0.90}$T{=}1$ & \multicolumn{2}{c}{\footnotesize answer 164 tokens} & \multicolumn{2}{c}{\footnotesize answer 56 tokens} & \multicolumn{2}{c}{\footnotesize answer 117 tokens} & \multicolumn{2}{c}{\footnotesize answer 251 tokens} \\
no draft ($D0$) & \multicolumn{2}{c}{\nfe{181.1}$\downarrow$ (\tps{1.00}$\uparrow$)\SZ} & \multicolumn{2}{c}{\nfe{72.5}$\downarrow$ (\tps{1.00}$\uparrow$)\SZ} & \multicolumn{2}{c}{\nfe{132.4}$\downarrow$ (\tps{1.00}$\uparrow$)\SZ} & \multicolumn{2}{c}{\nfe{266.5}$\downarrow$ (\tps{1.00}$\uparrow$)\SZ} \\
\midrule
\rowcolor{oursrow}\textbf{\sys{} (ours)} & \nfe{\textbf{58.4}}\,(\tps{\underline{2.40}})\SZ & \nfe{\textbf{46.3}}\,(\tps{\underline{1.81}})\SZ & \nfe{\textbf{29.4}}\,(\tps{\underline{1.74}})\SZ & \nfe{\textbf{24.0}}\,(\tps{1.11})\SZ & \nfe{\textbf{43.3}}\,(\tps{\underline{2.37}})\SZ & \nfe{\textbf{34.2}}\,(\tps{\underline{1.79}})\SZ & \nfe{\textbf{86.7}}\,(\tps{\underline{2.36}})\SZ & \nfe{\textbf{68.3}}\,(\tps{\underline{1.81}})\SZ \\
\rowcolor{oursrow}\sys{}-$\gamma$ (ours) & \nfe{\underline{58.8}}\,(\tps{\textbf{2.44}})\SZ & \nfe{\underline{49.0}}\,(\tps{\textbf{2.18}})\SZ & \nfe{\underline{30.3}}\,(\tps{\textbf{1.81}})\SZ & \nfe{\underline{27.3}}\,(\tps{\textbf{1.59}})\SZ & \nfe{\underline{43.7}}\,(\tps{\textbf{2.38}})\SZ & \nfe{\underline{36.2}}\,(\tps{\textbf{2.16}})\SZ & \nfe{\underline{87.4}}\,(\tps{\textbf{2.40}})\SZ & \nfe{\underline{72.1}}\,(\tps{\textbf{2.17}})\SZ \\
Spiffy graph & \nfe{69.7}\,(\tps{2.20})\SZ & \nfe{58.4}\,(\tps{1.69})\SZ & \nfe{31.0}\,(\tps{1.73})\SZ & \nfe{27.9}\,(\tps{\underline{1.15}})\SZ & \nfe{50.5}\,(\tps{2.25})\SZ & \nfe{42.6}\,(\tps{1.75})\SZ & \nfe{103.0}\,(\tps{2.20})\SZ & \nfe{86.5}\,(\tps{1.67})\SZ \\
free chain, RS & \nfe{75.3}\,(\tps{1.45})\SZ & \nfe{75.0}\,(\tps{1.08})\SZ & \nfe{33.6}\,(\tps{1.24})\SZ & \nfe{33.7}\,(\tps{0.84})\SZ & \nfe{57.8}\,(\tps{1.47})\SZ & \nfe{57.3}\,(\tps{1.07})\SZ & \nfe{115.9}\,(\tps{1.39})\SZ & \nfe{115.2}\,(\tps{1.00})\SZ \\
BlockSpec tree & \nfe{94.2}\,(\tps{1.59})\SZ & \nfe{81.6}\,(\tps{1.06})\SZ & \nfe{39.9}\,(\tps{1.34})\SZ & \nfe{35.2}\,(\tps{0.78})\SZ & \nfe{68.4}\,(\tps{1.60})\SZ & \nfe{59.0}\,(\tps{1.04})\SZ & \nfe{137.9}\,(\tps{1.59})\SZ & \nfe{119.0}\,(\tps{1.05})\SZ \\
\bottomrule
\end{tabular}}
}
\end{table}

\subsection{More ablations and curves}
\label{app:ext-ablations}

\begin{table}[htbp]
\centering\footnotesize
\caption{Ablations on SDAR-8B MBPP b32 at $T{=}0$. Top: recipe components. Bottom: padding and the mass floor
$\gamma$, the two knobs that trade forwards for wall clock. The full $\gamma$ grid
and the calibration-temperature ablation are in the released records. Seed $0$, $n{=}350$; first machine, so the ratios differ from Table~\ref{tab:main-sdar}.}
\label{tab:ablation}
\setlength{\tabcolsep}{6pt}
\begin{tabular}{lcc}
\toprule
 & \multicolumn{2}{c}{\nfe{NFE}$\downarrow$ (\tps{tok/s}$\uparrow$)} \\
\cmidrule(lr){2-3}
configuration & $D8$ & $D20$ \\
\midrule
\textbf{final recipe} (tree, hedge $0.6$, no calibration temperature, $\gamma{=}0$, unpadded) & \nfe{56.7}\,(\tps{3.65})\SZ & \nfe{39.5}\,(\tps{3.19})\SZ \\
$-$ token hedge (gate off) & \nfe{56.9}\,(\tps{3.72})\SZ & \nfe{40.3}\,(\tps{3.19})\SZ \\
chain instead of tree & \nfe{63.7}\,(\tps{3.34})\SZ & \nfe{51.2}\,(\tps{3.09})\SZ \\
$-$ scaled training corpus & \nfe{57.3}\,(\tps{3.63})\SZ & \nfe{40.5}\,(\tps{3.18})\SZ \\
\midrule
padded, $\gamma{=}0$ & \nfe{56.7}\,(\tps{3.66})\SZ & \nfe{39.5}\,(\tps{3.10})\SZ \\
no padding, $\gamma{=}0.05$ & \nfe{58.1}\,(\tps{3.68})\SZ & \nfe{44.4}\,(\tps{4.16})\SZ \\
no padding, $\gamma{=}0.02$ (calibrated; \sys{}-$\gamma$) & \nfe{56.9}\,(\tps{3.72})\SZ & \nfe{41.8}\,(\tps{4.13})\SZ \\
padded, $\gamma{=}0.02$ & \nfe{56.9}\,(\tps{3.79})\SZ & \nfe{41.7}\,(\tps{3.32})\SZ \\
\bottomrule
\end{tabular}

\end{table}

\paragraph{Full budget grid.} Tables~\ref{tab:main-full} and~\ref{tab:main-full-b}
repeat Table~\ref{tab:main-sdar} with the $D6$ and $D12$ columns.

\begin{table}[htbp]
\centering\footnotesize
\caption{SDAR-8B b32, $T{=}0$, all four budgets: GSM8K and MBPP. Seed $0$, unpadded, $n{=}350$; first machine, so the ratios differ from Table~\ref{tab:main-sdar}'s (same NFE).}
\label{tab:main-full}
\setlength{\tabcolsep}{3pt}
\resizebox{\linewidth}{!}{\begin{tabular}{lcccccccc}
\toprule
\nfe{NFE}$\downarrow$ (\tps{tok/s}$\uparrow$) & \multicolumn{4}{c}{GSM8K} & \multicolumn{4}{c}{MBPP} \\
 & \multicolumn{4}{c}{\footnotesize acc.\ 0.89--0.90} & \multicolumn{4}{c}{\footnotesize acc.\ 0.42--0.43} \\
\cmidrule(lr){2-5}\cmidrule(lr){6-9}
method & $D6$ & $D8$ & $D12$ & $D20$ & $D6$ & $D8$ & $D12$ & $D20$ \\
\midrule
no draft ($D0$) & \multicolumn{4}{c}{\nfe{270.8}$\downarrow$ (\tps{1.00}$\uparrow$)\SZ} & \multicolumn{4}{c}{\nfe{269.2}$\downarrow$ (\tps{1.00}$\uparrow$)\SZ} \\
\midrule
\rowcolor{oursrow}\textbf{\sys{} (ours)} & \nfe{\textbf{77.8}}\,(\tps{\textbf{2.79}})\SZ & \nfe{\textbf{70.1}}\,(\tps{\underline{3.00}})\SZ & \nfe{\textbf{61.0}}\,(\tps{\underline{3.29}})\SZ & \nfe{\textbf{52.4}}\,(\tps{\underline{2.98}})\SZ & \nfe{\textbf{64.8}}\,(\tps{\underline{3.36}})\SZ & \nfe{\textbf{56.7}}\,(\tps{\underline{3.65}})\SZ & \nfe{\textbf{48.3}}\,(\tps{\underline{3.75}})\SZ & \nfe{\textbf{39.5}}\,(\tps{\underline{3.19}})\SZ \\
\rowcolor{oursrow}\sys{}-$\gamma$ (ours) & \nfe{\underline{77.9}}\,(\tps{\underline{2.79}})\SZ & \nfe{\underline{70.1}}\,(\tps{\textbf{3.09}})\SZ & \nfe{\underline{61.5}}\,(\tps{\textbf{3.36}})\SZ & \nfe{\underline{54.7}}\,(\tps{\textbf{3.46}})\SZ & \nfe{\underline{64.8}}\,(\tps{\textbf{3.38}})\SZ & \nfe{\underline{56.9}}\,(\tps{\textbf{3.72}})\SZ & \nfe{\underline{49.1}}\,(\tps{\textbf{4.00}})\SZ & \nfe{\underline{41.8}}\,(\tps{\textbf{4.13}})\SZ \\
Spiffy graph & \nfe{106.4}\,(\tps{2.21})\SZ & \nfe{92.5}\,(\tps{2.46})\SZ & \nfe{81.0}\,(\tps{2.85})\SZ & \nfe{71.6}\,(\tps{2.81})\SZ & \nfe{103.2}\,(\tps{2.30})\SZ & \nfe{83.6}\,(\tps{2.81})\SZ & \nfe{72.7}\,(\tps{2.74})\SZ & \nfe{61.2}\,(\tps{2.64})\SZ \\
SSD free chain & \nfe{104.8}\,(\tps{2.20})\SZ & \nfe{92.4}\,(\tps{2.54})\SZ & \nfe{89.8}\,(\tps{2.50})\SZ & \nfe{84.7}\,(\tps{2.08})\SZ & \nfe{102.6}\,(\tps{2.32})\SZ & \nfe{84.5}\,(\tps{2.73})\SZ & \nfe{83.5}\,(\tps{2.42})\SZ & \nfe{75.8}\,(\tps{1.89})\SZ \\
BlockSpec tree & \nfe{139.1}\,(\tps{1.67})\SZ & \nfe{136.5}\,(\tps{1.64})\SZ & \nfe{128.2}\,(\tps{1.71})\SZ & \nfe{114.6}\,(\tps{1.49})\SZ & \nfe{136.3}\,(\tps{1.70})\SZ & \nfe{135.1}\,(\tps{1.70})\SZ & \nfe{126.8}\,(\tps{1.49})\SZ & \nfe{111.5}\,(\tps{1.16})\SZ \\
\bottomrule
\end{tabular}
}
\end{table}

\begin{table}[htbp]
\centering\footnotesize
\caption{SDAR-8B b32, $T{=}0$, all four budgets: HumanEval and MATH500. Seed $0$, unpadded, $n{=}350$ (HumanEval $164$); first machine, so the ratios differ from Table~\ref{tab:main-sdar}'s (same NFE).}
\label{tab:main-full-b}
\setlength{\tabcolsep}{3pt}
\resizebox{\linewidth}{!}{\begin{tabular}{lcccccccc}
\toprule
\nfe{NFE}$\downarrow$ (\tps{tok/s}$\uparrow$) & \multicolumn{4}{c}{HumanEval} & \multicolumn{4}{c}{MATH500} \\
 & \multicolumn{4}{c}{\footnotesize acc.\ 0.71--0.74} & \multicolumn{4}{c}{\footnotesize acc.\ 0.10--0.12} \\
\cmidrule(lr){2-5}\cmidrule(lr){6-9}
method & $D6$ & $D8$ & $D12$ & $D20$ & $D6$ & $D8$ & $D12$ & $D20$ \\
\midrule
no draft ($D0$) & \multicolumn{4}{c}{\nfe{271.4}$\downarrow$ (\tps{1.00}$\uparrow$)\SZ} & \multicolumn{4}{c}{\nfe{271.8}$\downarrow$ (\tps{1.00}$\uparrow$)\SZ} \\
\midrule
\rowcolor{oursrow}\textbf{\sys{} (ours)} & \nfe{\textbf{70.9}}\,(\tps{\underline{3.13}})\SZ & \nfe{\textbf{62.5}}\,(\tps{\underline{3.44}})\SZ & \nfe{\textbf{53.6}}\,(\tps{\underline{3.79}})\SZ & \nfe{\textbf{44.8}}\,(\tps{\underline{3.54}})\SZ & \nfe{\textbf{92.6}}\,(\tps{\underline{2.35}})\SZ & \nfe{\textbf{84.6}}\,(\tps{\underline{2.47}})\SZ & \nfe{\textbf{75.3}}\,(\tps{\underline{2.70}})\SZ & \nfe{\textbf{66.0}}\,(\tps{\underline{2.46}})\SZ \\
\rowcolor{oursrow}\sys{}-$\gamma$ (ours) & \nfe{\underline{70.9}}\,(\tps{\textbf{3.14}})\SZ & \nfe{\underline{62.6}}\,(\tps{\textbf{3.46}})\SZ & \nfe{\underline{54.2}}\,(\tps{\textbf{3.94}})\SZ & \nfe{\underline{46.7}}\,(\tps{\textbf{4.02}})\SZ & \nfe{\underline{92.7}}\,(\tps{\textbf{2.37}})\SZ & \nfe{\underline{84.7}}\,(\tps{\textbf{2.56}})\SZ & \nfe{\underline{76.1}}\,(\tps{\textbf{2.79}})\SZ & \nfe{\underline{69.1}}\,(\tps{\textbf{2.74}})\SZ \\
Spiffy graph & \nfe{105.0}\,(\tps{2.32})\SZ & \nfe{87.6}\,(\tps{2.74})\SZ & \nfe{77.5}\,(\tps{2.97})\SZ & \nfe{65.5}\,(\tps{3.01})\SZ & \nfe{109.5}\,(\tps{2.18})\SZ & \nfe{100.6}\,(\tps{2.31})\SZ & \nfe{88.9}\,(\tps{2.57})\SZ & \nfe{82.5}\,(\tps{2.41})\SZ \\
SSD free chain & \nfe{103.7}\,(\tps{2.30})\SZ & \nfe{87.0}\,(\tps{2.71})\SZ & \nfe{85.4}\,(\tps{2.71})\SZ & \nfe{80.0}\,(\tps{2.23})\SZ & \nfe{107.9}\,(\tps{2.17})\SZ & \nfe{100.6}\,(\tps{2.35})\SZ & \nfe{97.1}\,(\tps{2.33})\SZ & \nfe{94.0}\,(\tps{1.93})\SZ \\
BlockSpec tree & \nfe{138.3}\,(\tps{1.72})\SZ & \nfe{136.0}\,(\tps{1.71})\SZ & \nfe{127.5}\,(\tps{1.74})\SZ & \nfe{112.7}\,(\tps{1.51})\SZ & \nfe{140.7}\,(\tps{1.65})\SZ & \nfe{136.8}\,(\tps{1.65})\SZ & \nfe{128.5}\,(\tps{1.71})\SZ & \nfe{116.3}\,(\tps{1.40})\SZ \\
\bottomrule
\end{tabular}
}
\end{table}

\paragraph{Budget curves.} Table~\ref{tab:budget} and Figure~\ref{fig:budget}
give \sys{}'s (unpadded, $\gamma{=}0$) response to the budget on the four SDAR-8B
tasks and on mini MBPP. NFE falls to $D20$ on every setting; the
throughput optimum sits at $D12$--$D16$ on every SDAR-8B task ($3.79\times$ HumanEval, $3.75\times$ MBPP
and $2.70\times$ MATH500 at $D12$, $3.31\times$ GSM8K at $D16$) and falls to
$2.46$--$3.54\times$ at $D20$, while on LLaDA2.0-mini the curve rises to $D16$ and stays there at $D20$ ($2.31\times$): the optimum's location is a
property of the target's forward economics. With the mass floor of
\sys{}-$\gamma$ the SDAR ratios rise to $D20$ on GSM8K, MBPP and HumanEval and level off on MATH500
(Tables~\ref{tab:main-full} and~\ref{tab:main-full-b}).

\begin{table}[htbp]
\centering\small
\caption{Budget response, $T{=}0$, $n{=}350$ (HumanEval 164; mini $100$), seed $0$. Bold: throughput optimum.}
\label{tab:budget}
\setlength{\tabcolsep}{4pt}
\resizebox{\linewidth}{!}{\begin{tabular}{llcccccccc}
\toprule
 & & $D0$ & $D3$ & $D5$ & $D6$ & $D8$ & $D12$ & $D16$ & $D20$ \\
\midrule
\multirow{2}{*}{\sys{}, GSM8K} & \nfe{NFE}$\downarrow$ & \nfe{270.8}\SZ & \nfe{102.4}\SZ & \nfe{84.3}\SZ & \nfe{77.8}\SZ & \nfe{70.1}\SZ & \nfe{61.0}\SZ & \nfe{54.8}\SZ & \nfe{52.4}\SZ \\
 & \tps{tok/s}$\uparrow$ & \tps{1.00}\SZ & \tps{2.20}\SZ & \tps{2.64}\SZ & \tps{2.79}\SZ & \tps{3.00}\SZ & \tps{3.29}\SZ & \tps{\textbf{3.31}}\SZ & \tps{2.98}\SZ \\
\multirow{2}{*}{\sys{}, MBPP} & \nfe{NFE}$\downarrow$ & \nfe{269.2}\SZ & \nfe{88.8}\SZ & \nfe{72.1}\SZ & \nfe{64.8}\SZ & \nfe{56.7}\SZ & \nfe{48.3}\SZ & \nfe{40.7}\SZ & \nfe{39.5}\SZ \\
 & \tps{tok/s}$\uparrow$ & \tps{1.00}\SZ & \tps{2.54}\SZ & \tps{3.04}\SZ & \tps{3.36}\SZ & \tps{3.65}\SZ & \tps{\textbf{3.75}}\SZ & \tps{3.71}\SZ & \tps{3.19}\SZ \\
\multirow{2}{*}{\sys{}, HumanEval} & \nfe{NFE}$\downarrow$ & \nfe{271.4}\SZ & \nfe{95.6}\SZ & \nfe{77.7}\SZ & \nfe{70.9}\SZ & \nfe{62.5}\SZ & \nfe{53.6}\SZ & \nfe{46.4}\SZ & \nfe{44.8}\SZ \\
 & \tps{tok/s}$\uparrow$ & \tps{1.00}\SZ & \tps{2.37}\SZ & \tps{2.90}\SZ & \tps{3.13}\SZ & \tps{3.44}\SZ & \tps{\textbf{3.79}}\SZ & \tps{3.78}\SZ & \tps{3.54}\SZ \\
\multirow{2}{*}{\sys{}, MATH500} & \nfe{NFE}$\downarrow$ & \nfe{271.8}\SZ & \nfe{117.0}\SZ & \nfe{98.4}\SZ & \nfe{92.6}\SZ & \nfe{84.6}\SZ & \nfe{75.3}\SZ & \nfe{69.2}\SZ & \nfe{66.0}\SZ \\
 & \tps{tok/s}$\uparrow$ & \tps{1.00}\SZ & \tps{1.96}\SZ & \tps{2.24}\SZ & \tps{2.35}\SZ & \tps{2.47}\SZ & \tps{\textbf{2.70}}\SZ & \tps{2.56}\SZ & \tps{2.46}\SZ \\
\midrule
\multirow{2}{*}{\sys{}, mini MBPP} & \nfe{NFE}$\downarrow$ & \nfe{270.4}\SZ & \nfe{125.1}\SZ & \nfe{106.8}\SZ & \nfe{102.2}\SZ & \nfe{93.5}\SZ & \nfe{85.5}\SZ & \nfe{77.2}\SZ & \nfe{74.2}\SZ \\
 & \tps{tok/s}$\uparrow$ & \tps{1.00}\SZ & \tps{1.79}\SZ & \tps{1.97}\SZ & \tps{2.03}\SZ & \tps{2.11}\SZ & \tps{2.16}\SZ & \tps{\textbf{2.31}}\SZ & \tps{2.31}\SZ \\
\bottomrule
\end{tabular}
}
\end{table}

\begin{figure}[htbp]
\centering
\includegraphics[width=0.95\linewidth]{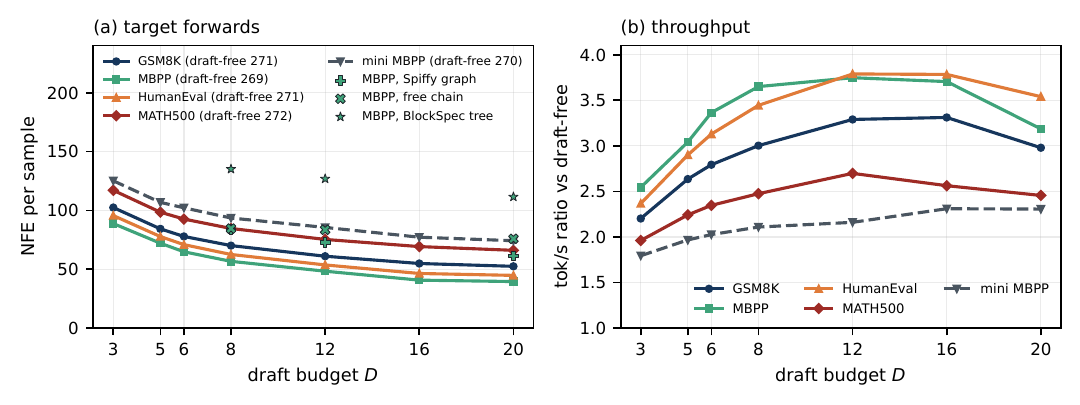}
\caption{Budget curves from Table~\ref{tab:budget} ($T{=}0$, seed $0$, unpadded; SDAR-8B solid, LLaDA2.0-mini MBPP
dashed). Left: NFE vs.\ draft budget for \sys{} and, on SDAR-8B MBPP, the fixed baselines at $D8/12/20$ (markers).
Right: throughput ratio over each setting's draft-free control for \sys{}, with its optimum at $D12$--$D16$ on SDAR-8B
and level from $D16$ to $D20$ on LLaDA2.0-mini.}
\label{fig:budget}
\end{figure}

\paragraph{Temperature sweep.} Table~\ref{tab:tsweep}, discussed in
Section~\ref{sec:obs-temp}, condenses the sweep to $D8$ and to \sys{} and Spiffy;
Table~\ref{tab:tsweep-full} is the complete grid: five temperatures, three drafters,
three verifiers, $D8$ and $D20$, with the draft-free row of each temperature
carrying the target's own accuracy, argmax fraction and mean committed-token
probability. Two rows worth quoting: at $T{=}1.5$, coupled vs uncoupled at $D20$ is $53.1$ vs $62.5$ NFE for
\sys{} and $72.6$ vs $84.2$ for Spiffy; the confidence tree, which ranks positions by
stale confidence without the noise, gains little from coupling ($121.6$ vs $126.0$
at the same budget). At
$T{=}2$, where the target's samples are near-uniform, uncoupled acceptance
all but disappears for every drafter (all three sit at $257$--$261$ NFE at $D8$ against
the $267.1$ of draft-free decoding in this grid) while
coupling still drafts, and Spiffy's noise-aware position ranking (stale confidence of
the \emph{drawn} token) overtakes the $T{=}0$-trained order model under coupling
($114.8$ vs $187.2$ NFE at $D8$), an ordering that holds nowhere the target
produces usable text.

\begin{table}[htbp]
\centering\small
\caption{Complete temperature sweep, SDAR-8B MBPP b32, $n{=}100$. \SEEDN{tsweep}. NFE
(tok/s ratio over the $T$-matched draft-free control) for every
drafter under the three equality verifiers at $D8$ and $D20$. $^{\dagger}$greedy match is not
distribution-preserving at $T>0$ (Proposition~\ref{prop:match} does not apply); its
NFE is not comparable to the exact regimes' and its emitted text drifts from the
target's distribution (Table~\ref{tab:tsweep}).}
\label{tab:tsweep-full}
\setlength{\tabcolsep}{4pt}
\resizebox{\linewidth}{!}{\begin{tabular}{llcccccc}
\toprule
 & & \multicolumn{2}{c}{coupled-noise match} & \multicolumn{2}{c}{uncoupled match} & \multicolumn{2}{c}{greedy match$^{\dagger}$} \\
\cmidrule(lr){3-4}\cmidrule(lr){5-6}\cmidrule(lr){7-8}
$T$ & drafter & $D8$ & $D20$ & $D8$ & $D20$ & $D8$ & $D20$ \\
\midrule
\multirow{4}{*}{0.5} & no draft ($D0$) & \multicolumn{6}{c}{267.1 NFE, accuracy 0.36, argmax fraction 0.998, mean committed $p$ 0.994} \\
 & \sys{} & 56.4\,(3.70) & 39.2\,(3.25) & 56.4\,(3.72) & 39.2\,(3.25) & 56.0\,(3.67) & 39.0\,(3.19) \\
 & Spiffy graph & 83.2\,(2.83) & 60.9\,(2.70) & 83.0\,(2.83) & 60.5\,(2.73) & 83.1\,(2.79) & 60.8\,(2.67) \\
 & BlockSpec-style tree & 135.4\,(1.69) & 112.2\,(1.18) & 134.9\,(1.71) & 111.4\,(1.17) & 135.4\,(1.69) & 112.1\,(1.15) \\
\midrule
\multirow{4}{*}{1.0} & no draft ($D0$) & \multicolumn{6}{c}{267.1 NFE, accuracy 0.42, argmax fraction 0.990, mean committed $p$ 0.968} \\
 & \sys{} & 58.5\,(3.56) & 41.1\,(3.09) & 58.6\,(3.56) & 41.0\,(3.11) & 56.1\,(3.68) & 39.3\,(3.15) \\
 & Spiffy graph & 85.3\,(2.75) & 63.5\,(2.56) & 85.5\,(2.72) & 63.1\,(2.58) & 85.0\,(2.74) & 63.3\,(2.54) \\
 & BlockSpec-style tree & 137.1\,(1.68) & 113.9\,(1.16) & 135.4\,(1.70) & 111.9\,(1.15) & 136.9\,(1.67) & 113.9\,(1.15) \\
\midrule
\multirow{4}{*}{1.5} & no draft ($D0$) & \multicolumn{6}{c}{267.1 NFE, accuracy 0.11, argmax fraction 0.893, mean committed $p$ 0.841} \\
 & \sys{} & 71.2\,(2.98) & 53.1\,(2.45) & 79.1\,(2.69) & 62.5\,(2.10) & 65.0\,(3.24) & 49.0\,(2.65) \\
 & Spiffy graph & 92.8\,(2.51) & 72.6\,(2.21) & 106.1\,(2.20) & 84.2\,(1.88) & 92.5\,(2.50) & 73.1\,(2.19) \\
 & BlockSpec-style tree & 145.3\,(1.57) & 121.6\,(1.06) & 148.3\,(1.54) & 126.0\,(1.03) & 145.1\,(1.56) & 121.6\,(1.07) \\
\midrule
\multirow{4}{*}{2.0} & no draft ($D0$) & \multicolumn{6}{c}{267.1 NFE, accuracy 0.00, argmax fraction 0.062, mean committed $p$ 0.054} \\
 & \sys{} & 187.2\,(1.21) & 146.1\,(0.89) & 257.5\,(0.88) & 252.7\,(0.51) & 183.7\,(1.23) & 144.2\,(0.90) \\
 & Spiffy graph & 114.8\,(2.03) & 103.1\,(1.47) & 259.4\,(0.91) & 256.6\,(0.59) & 114.8\,(2.04) & 103.0\,(1.44) \\
 & BlockSpec-style tree & 186.8\,(1.22) & 148.2\,(0.86) & 261.3\,(0.88) & 257.2\,(0.50) & 186.7\,(1.21) & 148.2\,(0.85) \\
\midrule
\multirow{4}{*}{2.5} & no draft ($D0$) & \multicolumn{6}{c}{267.1 NFE, accuracy 0.00, argmax fraction 0.002, mean committed $p$ 0.000} \\
 & \sys{} & 198.9\,(1.14) & 155.4\,(0.84) & 267.0\,(0.85) & 267.0\,(0.48) & 198.9\,(1.14) & 155.4\,(0.83) \\
 & Spiffy graph & 109.9\,(2.12) & 98.1\,(1.56) & 267.0\,(0.88) & 267.0\,(0.56) & 109.9\,(2.13) & 98.1\,(1.54) \\
 & BlockSpec-style tree & 184.7\,(1.24) & 145.1\,(0.89) & 267.0\,(0.86) & 267.0\,(0.49) & 184.7\,(1.23) & 145.1\,(0.87) \\
\bottomrule
\end{tabular}
}
\end{table}

\subsection{Additional verifier material}
\label{app:verifier-extra}
Table~\ref{tab:verifier} isolates verifier and structure on \sys{}'s own proposals
at $T{=}1$: the tree is worth $11$--$17$ NFE and $0.25$--$0.67$ in throughput ratio over the same drafter's chain under the
same verifier. On the chain, RS takes $1.9$--$3.0$ fewer NFE than match but is $0.23$--$0.25$ behind in throughput
ratio at every budget: the per-link acceptance test costs more than the forwards it saves.

\begin{table}[htbp]
\centering\footnotesize
\caption{Verifier and structure at $T{=}1$, SDAR-8B MBPP: \sys{}'s proposals under
each verifier (match: coupled-noise match; RS: rejection sampling, chains only).
Seed $0$.}
\label{tab:verifier}
\setlength{\tabcolsep}{3pt}
\begin{tabular}{llccc}
\toprule
 & & \multicolumn{3}{c}{\nfe{NFE}$\downarrow$ (\tps{tok/s}$\uparrow$)} \\
\cmidrule(lr){3-5}
drafter & verifier & $D8$ & $D12$ & $D20$ \\
\midrule
\sys{} & tree + match & \nfe{60.3}\,(\tps{3.48})\SZ & \nfe{51.6}\,(\tps{3.64})\SZ & \nfe{42.7}\,(\tps{2.96})\SZ \\
\sys{} & chain + match & \nfe{71.7}\,(\tps{3.01})\SZ & \nfe{65.3}\,(\tps{2.97})\SZ & \nfe{59.8}\,(\tps{2.71})\SZ \\
\sys{} & chain + RS & \nfe{68.9}\,(\tps{2.76})\SZ & \nfe{63.4}\,(\tps{2.74})\SZ & \nfe{56.8}\,(\tps{2.48})\SZ \\
\bottomrule
\end{tabular}

\end{table}

On chains, where both exact verifiers are defined, earlier runs of the same drafter, not part of the
benchmark tables, established a \emph{regime law}: per link, match transfers the drafter's argmax
skill through the shared noise, while RS degrades with $D_{TV}(\pi, q)$ rather
than with argmax stability, so RS wins where the stale conditional is flat (stale
confidence $\lesssim 0.6$) and match where it is sharp; temperature moves the mix
of regimes (the low-confidence share of rounds grows from ${\sim}6\%$ at
$T{=}0.5$ to ${\sim}46\%$ at $T{=}1.5$) without changing the mechanism, and a
sharper drafter raises match specifically. In the final benchmark RS takes
up to $3.0$ fewer forwards on \sys{}'s chain, but match stays ahead in throughput ratio at every budget, because RS
pays an acceptance test per link. A confidence-gated hybrid (match above $\sim 0.6$, RS
below) is the natural chain-only refinement and remains unimplemented, since
trees under match dominate every chain configuration.

\subsection{Mechanism: what happens in deep blocks}
\label{app:mechanism}
Averaged over the runs of Table~\ref{tab:main-sdar} at $T{=}0$, the chain and the tree accept $1.1$--$2.9$ tokens per
drafted round and \sys{} $2.5$--$7.5$; deep blocks accept more. Two measurements explain the gap structure of Table~\ref{tab:main-sdar} (SDAR-8B
MBPP b32, $D20$, seed 0, from an earlier run of this setting that logged per-block counts). First, stale-order knowledge decays across blocks: the
probability that the previous round's top-confidence position still commits falls
$0.74 / 0.68 / 0.34 / 0.16 / 0.11 / 0.08 / 0.07 / 0.04 / 0.03$ across blocks
$0$--$8$, while the stale argmax token stays right ($0.92 / 0.94 / 0.99 /
1.00 \dots$), deep blocks are near-deterministic in \emph{what} and unpredictable
in \emph{which position first}, from stale confidence alone. Second, accepted tokens
per drafted round by block:

\begin{center}\small
\begin{tabular}{lccccccccc}
\toprule
block $\to$ & 0 & 1 & 2 & 3 & 4 & 5 & 6 & 7 & 8 \\
\midrule
\textbf{\sys{}} & 2.29 & 2.67 & 5.83 & 9.54 & 12.30 & 13.76 & 14.50 & 14.50 & 14.50 \\
Spiffy graph & 1.88 & 2.23 & 3.45 & 4.48 & 4.74 & 5.05 & 5.18 & 5.14 & 5.15 \\
SSD-style free chain & 1.69 & 1.89 & 2.57 & 3.11 & 3.22 & 3.39 & 3.41 & 3.41 & 3.40 \\
BlockSpec-style tree & 1.55 & 1.39 & 1.50 & 1.55 & 1.57 & 1.58 & 1.58 & 1.58 & 1.58 \\
\bottomrule
\end{tabular}
\end{center}

Every drafter accepts more as the blocks become deterministic, but the amount is
set by how deep a \emph{correctly ordered} path each can propose: the learned,
prefix-conditional order model rides the transition law to $14.5$ accepted tokens
per round, the fixed 18-node graph is capped by its depth-6 chains at ${\sim}5$,
the free chain by its depth at ${\sim}3.4$, and the confidence tree spends its
budget on mutually exclusive first-level siblings (${\sim}15$ of 20 nodes) and
reaches depth $2$ or less in $90\%$ of drafted rounds. Figure~\ref{fig:mechanism} shows the same pair
graphically.

\begin{figure}[htbp]
\centering
\includegraphics[width=0.95\linewidth]{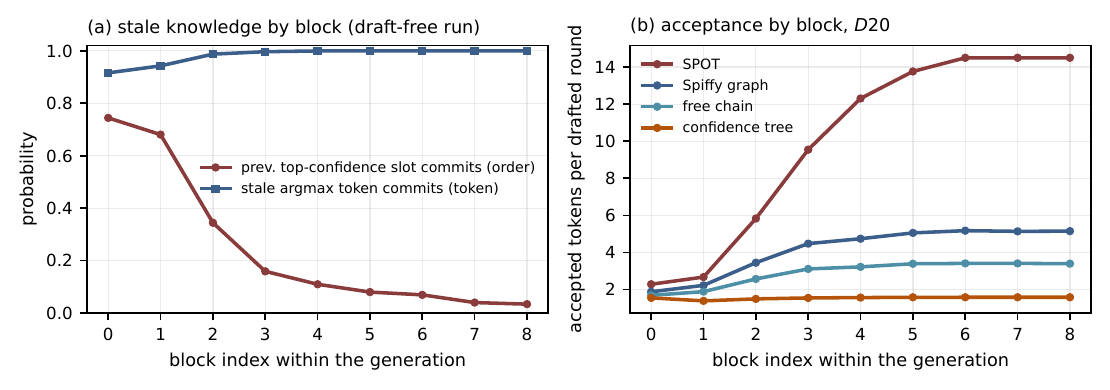}
\caption{Left: stale-order knowledge decays across blocks (previous round's
top-confidence position still commits) while the stale argmax token stays right.
Right: accepted tokens per drafted round by block, $D20$, SDAR-8B MBPP, from the same earlier run.}
\label{fig:mechanism}
\end{figure}

\subsection{Error analysis of the drafter}
\label{app:error}
Every rejection is attributable: at the first rejected link of a drafted round,
the fallback commit records the rank of the \emph{true} position under the order
model's proposal ordering and of the true token under the stale conditional. On
earlier full-set runs of the same drafter, not part of the benchmark tables, position misses account for
$62$--$81\%$ of first rejections and token misses for $3$--$6\%$; a growing share
of rejections at $D20$ ($32\%$ on MBPP, $90\%$ in the last block) occur with the
truth ranked \emph{first} by the order model, the tree ran out of nodes, which
is why NFE keeps improving with budget and why the mass floor $\gamma$ trades NFE
for wall clock rather than being free. The drafter's prefix-mass score is nearly
calibrated against realized acceptance ($96\%$ acceptance in the top mass bin,
$3\%$ in the bottom; Figure~\ref{fig:error}). These proportions are reported as a
diagnostic of the drafter and no claim in the main text rests on them.

\begin{figure}[htbp]
\centering
\includegraphics[width=0.95\linewidth]{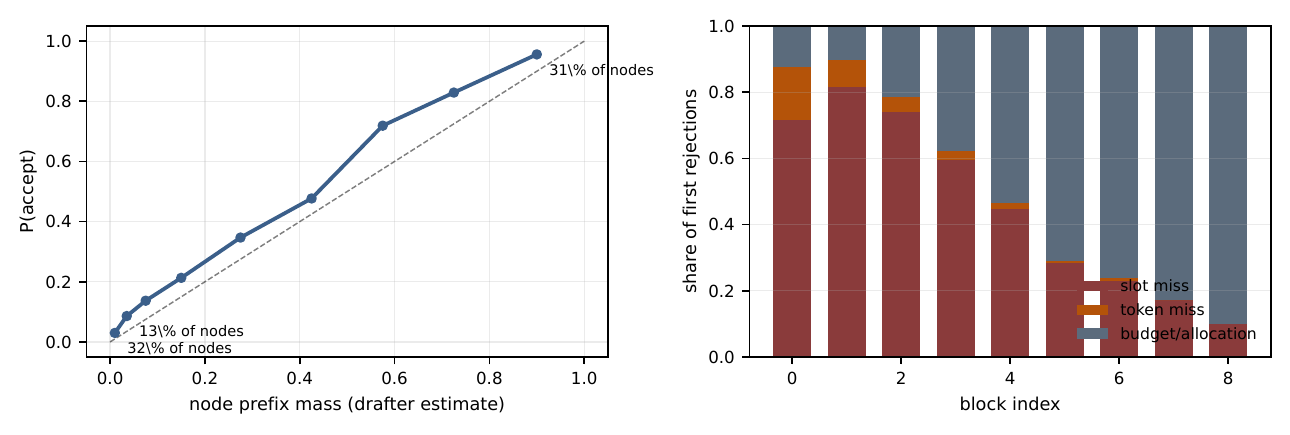}
\caption{Left: acceptance calibration of the drafter's own node score (realized
$P(\text{accept})$ per prefix-mass bin against the diagonal). Right:
first-rejection attribution by block ($D20$): position misses dominate early, deep
blocks are budget-bound. Earlier runs of the same drafter.}
\label{fig:error}
\end{figure}

\subsection{Drafter architecture}
\label{app:ext-arch}
An earlier offline capacity probe found the order model
feature-limited rather than capacity-limited: at fixed inputs, $8\times$ the
parameters changed held-out next-commit accuracy by $+0.004$, whereas broadening
the training data moved decoding. The released drafter is three layers,
$d_{\text{model}}{=}128$, $8$ heads. The probe is offline: we did not sweep depth or
width at decode time on the final benchmark, so the released size is a working
choice rather than a tuned optimum.

\section{Reproducibility}
\label{app:repro}

Code and checkpoints are available at \url{https://github.com/Eliezer318/SPOT-dLLM}.

\paragraph{Seeds and wall clock.} Table~\ref{tab:main-sdar} and Table~\ref{tab:tsweep} report the mean of
seeds $0$--$2$, each ratio over the draft-free run of the same seed; the other tables are seed $0$ unless their caption says otherwise. At $T{=}0$ the
decoder is deterministic: every sample has the same tokens and forwards under every seed, and only wall clock
moves. Over Table~\ref{tab:main-sdar}'s $40$ one-commit cells at $T{=}0$, the max--min spread of the throughput
ratio over seeds has median $0.05$ and maximum $0.15$; over its $40$ threshold cells, whose ratios are two to
three times larger, median $0.10$ and maximum $0.24$. At $T{=}1$ the target's own sampling moves the forwards as
well: the NFE standard deviation over seeds has median $0.33$ (one-commit) and $0.27$ (threshold) and maximum
$1.1$ and $2.1$, and the largest ratio spread, $1.22$, is on \sysSets{}'s HumanEval cells, the shortest task at
$164$ samples. The mean accuracy over seeds stays within $0.04$ of the control's in every one-commit cell. In
Table~\ref{tab:tsweep-full} the standard deviation of a ratio over seeds is at most $0.14$.

\paragraph{Oracle pairing.} Oracle runs are unpadded. Their tokens per second use the net wall clock,
with the oracle's own lookahead forwards subtracted, and every oracle row of a setting is divided by that
setting's draft-free run, measured under the same one-run-per-card protocol (Appendix~\ref{app:ext-oracle}).

\paragraph{Software.} PyTorch 2.7.1, CUDA 12.8, Transformers 4.55.4; one run per card, never co-located (machines in Section~\ref{sec:exp-setup}): the first
machine runs the one-commit \sys{} experiments, the second Table~\ref{tab:main-sdar}, the SDAR-4B $T{=}1$ table
(Table~\ref{tab:main-4b-t1}) and every threshold-decoding (\sysSets{}) run.
\section{\sysSets{} details}
\label{app:sets}
The main tables use \emph{drafter D}, trained on general chat prompts; \emph{drafter A}, the same drafter trained on task
prompts, appears only in this appendix (Appendix~\ref{app:sets-data}).

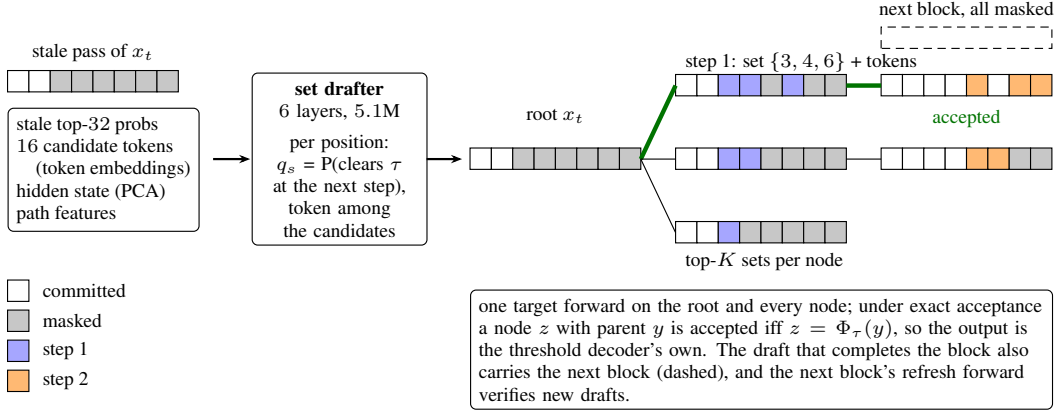
\begin{figure}[ht]
\centering
\newcommand{\blk}[3]{\foreach \c [count=\i from 0] in {#3}{\draw[fill=\c] (#1+\i*0.27,#2) rectangle ++(0.27,0.27);}}
\colorlet{cm}{black!22}\colorlet{cc}{white}\colorlet{ca}{blue!35}\colorlet{cb}{orange!55}
\resizebox{\linewidth}{!}{%
\begin{tikzpicture}[font=\scriptsize, arr/.style={-{Stealth[length=4pt]}, thick},
  box/.style={draw, rounded corners=2pt, align=center, inner sep=3pt}]
\node at (1.08,2.25) {stale pass of $x_t$};
\blk{0}{1.75}{cc,cc,cm,cm,cm,cm,cm,cm}
\node[box, anchor=north west, align=left] at (0,1.55)
  {stale top-$32$ probs\\ $16$ candidate tokens\\ \quad(token embeddings)\\ hidden state (PCA)\\ path features};
\draw[arr] (2.6,0.9) -- (3.07,0.9);
\begin{scope}[xshift=0.35cm]
\node[box, text width=2.0cm] at (3.85,0.9) {\textbf{set drafter}\\ $6$ layers, $5.1$M\\[3pt] per position:\\ $q_s$ = P(clears $\tau$\\ at the next step),\\ token among\\ the candidates};
\draw[arr] (4.95,0.9) -- (5.4,0.9);
\node at (6.58,1.45) {root $x_t$};
\blk{5.5}{0.77}{cc,cc,cm,cm,cm,cm,cm,cm}
\blk{8.1}{1.7}{cc,cc,ca,ca,cm,ca,cm,cm}
\blk{8.1}{0.77}{cc,cc,ca,ca,cm,cm,cm,cm}
\blk{8.1}{-0.16}{cc,cc,ca,cm,cm,cm,cm,cm}
\blk{10.7}{1.7}{cc,cc,cc,cc,cb,cc,cb,cb}
\blk{10.7}{0.77}{cc,cc,cc,cc,cb,cb,cm,cm}
\draw[densely dashed] (10.7,2.3) rectangle ++(2.16,0.27);
\node[anchor=south] at (11.78,2.57) {next block, all masked};
\draw[line width=1.6pt, green!45!black] (7.66,0.9) -- (8.1,1.83);
\draw (7.66,0.9) -- (8.1,0.9); \draw (7.66,0.9) -- (8.1,-0.03);
\draw[line width=1.6pt, green!45!black] (10.26,1.83) -- (10.7,1.83);
\draw (10.26,0.9) -- (10.7,0.9);
\node[anchor=west] at (8.1,2.12) {step 1: set $\{3,4,6\}$ + tokens};
\node[anchor=west] at (8.1,-0.4) {top-$K$ sets per node};
\node[anchor=north, green!45!black] at (11.78,1.65) {accepted};
\node[box, anchor=north west, text width=7.15cm, align=left] at (5.5,-0.75)
  {one target forward on the root and every node; under exact acceptance a node $z$ with parent $y$ is
   accepted iff $z=\Phi_\tau(y)$, so the output is the threshold decoder's own. The draft that
   completes the block also carries the next block (dashed), and the next block's refresh forward
   verifies new drafts.};
\end{scope}
\foreach \c/\t [count=\j from 0] in {cc/committed, cm/masked, ca/step 1, cb/step 2}{
  \draw[fill=\c] (0,-0.9-\j*0.38) rectangle ++(0.27,0.27); \node[anchor=west] at (0.32,-0.765-\j*0.38) {\t};}
\end{tikzpicture}}
\caption{One \sysSets{} round. The set drafter reads the stale pass and scores, for each masked
position, whether it clears the bar at the next target step and which candidate token it takes. A tree
node is a whole target step: the children of a node are its most probable commit sets, filled with the
predicted tokens. One target forward verifies every node by state equality.}
\label{fig:sets-round}
\end{figure}

\subsection{Drafter, training and tree}
\label{app:sets-drafter}
\paragraph{Threshold transition.} With $c_s(x)=\max_v p_{\theta,s}(v\mid x,C)$, the threshold decoder
commits $S_\tau(x)=\{s\in\mathcal M(x):c_s(x)\ge\tau\}$, or $\{s^\star(x)\}$ if that set is empty, each
with its argmax token at $T{=}0$. At $T{>}0$ each position draws $v_G(s)$ with the shared noise
and the bar tests $p_{\theta,s}(v_G(s)\mid x,C)$. Proposition~\ref{prop:greedy} and Steps~2 and~3 of the proof
of Proposition~\ref{prop:match} use only that the next state is a deterministic function of the current state (given
$G$), so they hold for $\Phi_\tau$ unchanged: for every realization of $G$, the speculative decoder reproduces the
shared-noise threshold decoder, as for the one-commit rule (Appendix~\ref{app:proof-match}).

\paragraph{Drafter and training.} The encoder is a six-layer transformer over the masked positions of a hypothetical
state ($d_{\mathrm{model}}{=}256$, $8$ heads), with PCA bases for the target's token embeddings and hidden states. A
\emph{commit} head gives $q_s$, a \emph{token} head scores the $16$ candidates, and a \emph{fallback} head scores the
position that commits when none clears the bar. The drafter trains on about $1$M prefix-conditional examples from
threshold-decoded trajectories ($\tau{=}0.9$, $T{=}0$), labelled with the set committed next (binary cross-entropy), the
fallback position and each committed token's index among the candidates (cross-entropy, weights $0.3$ and $0.5$); Adam,
learning rate $5{\times}10^{-4}$, $12$ epochs, about $30$ minutes on one card, and a temperature of $1.12$ calibrates $q$.
On held-out trajectories (drafter A, Appendix~\ref{app:sets-data}) the commit head reaches AUROC $0.95$, and the token head
picks the committed token for $87/73/62\%$ of the positions one, two and three steps ahead, against $83/60/44\%$ for the
stale argmax. The other targets' drafters use the same recipe.

\paragraph{Tree, block ends and floor.} Under the commit head the probability of committing exactly $S$ is
$\prod_{s\in S}q_s\prod_{s\notin S}(1-q_s)$, with the empty set replaced by the fallback position. The children of a
node are its top-$8$ sets, enumerated exactly, plus one sibling of the best set that swaps its least certain token for
the runner-up. A draft that completes the block also holds the next block, all masked, so the verifying forward returns
the next block's first step, and the next block's cache-refresh forward verifies drafts made from it. Each rule and
drafter has its own floor, set by the rule of Appendix~\ref{app:gamma-calib} on held-out MBPP samples, which picks
$\gamma{=}0.02$, or $\gamma{=}0$ for some rules on SDAR-1.7B and LLaDA2.0-mini. The drafter takes $5.6$\,ms per drafted round, $5.8\%$ of
the wall time on SDAR-8B. Table~\ref{tab:sets-ladder} gives the numbers behind Figure~\ref{fig:sets-ladder}.

\begin{table}[h]
\centering\small
\caption{What each part of \sysSets{} adds (Figure~\ref{fig:sets-ladder}): SDAR-8B, drafter D, $\tau{=}0.9$, $T{=}0$,
cap $8$, seed $0$, \THRN{final_sets_ladder}. Throughput ratio over standard decoding (NFE).}
\label{tab:sets-ladder}\label{tab:sets-ablation}
\setlength{\tabcolsep}{4pt}
\resizebox{\linewidth}{!}{\begin{tabular}{lccccc}
\toprule
tok/s ratio$\uparrow$ (NFE$\downarrow$) & GSM8K & MBPP & HumanEval & MATH500 & mean \\
\midrule
set drafter + set tree & 6.13 (35.6)\SZ & 8.15 (24.8)\SZ & 8.19 (26.8)\SZ & 4.43 (48.6)\SZ & 6.72\SZ \\
+ adaptive floor (calibrated $\gamma$) & 6.22 (36.5)\SZ & 8.37 (25.5)\SZ & 8.21 (27.6)\SZ & 4.61 (49.3)\SZ & 6.85\SZ \\
+ cross-block drafting & 6.63 (34.0)\SZ & 8.69 (24.4)\SZ & 8.94 (25.2)\SZ & 4.85 (46.7)\SZ & 7.28\SZ \\
+ draft refresh (\sysSets{}-$\gamma$, exact) & 6.94 (32.5)\SZ & 8.95 (23.8)\SZ & 9.28 (24.0)\SZ & 4.98 (45.1)\SZ & 7.54\SZ \\
+ max-commit, product (\sysSets{}-$\gamma$) & 7.33 (30.6)\SZ & 9.17 (22.8)\SZ & 9.80 (22.7)\SZ & 5.29 (42.3)\SZ & 7.90\SZ \\
\bottomrule
\end{tabular}
}
\end{table}

\subsection{Max-commit rules}
\label{app:sets-maxcommit}
State equality accepts a drafted step only if it equals the target's set. A relaxed rule also accepts a step whose
tokens are the target's top-$1$ tokens and whose positions pass a looser test, and max-commit takes the accepted path
with the most commits. From strict to loose, \emph{cut} accepts the set some higher bar $\tau'\ge\tau$ would commit,
\emph{subset} any part of the target's set, and \emph{product} any step whose confidences have a geometric mean of at
least $\tau$, so each rule accepts every step that a stricter one accepts. The less strict the rule, the more speed-up
(Table~\ref{tab:sets-maxcommit}). Product gains $4$--$6\%$ over exact on every target, and the one exception to the order
is cut on LLaDA2.0-mini ($3\%$ below exact). The cost is fidelity to the threshold decoder, as the share of identical
outputs falls from exact to product (exact itself gives identical outputs on $27$--$62\%$ of the prompts, the rest
differing through floating-point batch effects), while the
accuracy stays within $0.04$ of exact on every task. Drafter A gives the same order on every target.

\begin{table}[h]
\centering\small
\caption{Max-commit rules on \sysSets{}-$\gamma$ (drafter D, cap $8$, each rule with its own floor), $\tau{=}0.9$,
$T{=}0$, seed $0$, mean over the four tasks ($n{=}350$ per task, HumanEval $164$; LLaDA2.0-mini $100$). Cell: throughput
ratio over standard decoding / accuracy / outputs identical to the threshold decoder's.}
\label{tab:sets-maxcommit}
\setlength{\tabcolsep}{4pt}
\begin{tabular}{lcccc}
\toprule
rule & SDAR-8B & SDAR-4B & SDAR-1.7B & LLaDA2.0-mini \\
\midrule
exact & 7.54 / 0.52 / $59\%$\SZ & 7.81 / 0.46 / $62\%$\SZ & 6.61 / 0.37 / $51\%$\SZ & 5.09 / 0.38 / $27\%$\SZ \\
cut & 7.60 / 0.52 / $57\%$\SZ & 7.85 / 0.46 / $60\%$\SZ & 6.64 / 0.37 / $48\%$\SZ & 4.94 / 0.39 / $26\%$\SZ \\
subset & 7.73 / 0.53 / $55\%$\SZ & 7.94 / 0.46 / $58\%$\SZ & 6.73 / 0.38 / $43\%$\SZ & 5.29 / 0.39 / $22\%$\SZ \\
product & 7.90 / 0.52 / $51\%$\SZ & 8.26 / 0.46 / $54\%$\SZ & 6.88 / 0.37 / $41\%$\SZ & 5.31 / 0.39 / $20\%$\SZ \\
\bottomrule
\end{tabular}

\end{table}

\subsection{Training data}
\label{app:sets-data}
Our main drafter (D) trains on general chat prompts only, the first user turn of ShareGPT and UltraChat-200K
conversations, half from each, after a $13$-gram check against every test set. Drafter A instead trains on the GSM8K,
MATH and MBPP train splits and the first $4{,}000$ CodeAlpaca prompts. Everything else is the same (target, chat
template, architecture, hyperparameters, seed, number of examples and steps), all responses come from the target, and
each drafter fits its hidden-state basis on its own prompts. As expected, A is faster on the tasks it was trained for,
and the two are almost equal on HumanEval, which neither saw (Table~\ref{tab:sets-data}). D keeps $92$--$97\%$ of A's
speed-up over the threshold decoder on GSM8K, MATH500 and MBPP of the SDAR targets, $87$--$96\%$ on LLaDA2.0-mini, and
$97$--$99\%$ on HumanEval. On SDAR-8B, A reaches $8.2\times$ the throughput of standard decoding on average at $T{=}0$ and
$D8$, against $7.9\times$ for drafter D (Table~\ref{tab:sets-a}).

\begin{table}[h]
\centering\small
\caption{Drafter D (general chat) / drafter A (task prompts): speed-up over the threshold decoder without drafts,
\sysSets{}-$\gamma$, product rule, cap $8$, $\tau{=}0.9$, $T{=}0$, seed $0$, $n{=}350$ per task (HumanEval $164$;
LLaDA2.0-mini $100$). Neither drafter saw HumanEval.}
\label{tab:sets-data}
\setlength{\tabcolsep}{4pt}
{\begin{tabular}{lcccc}
\toprule
target & GSM8K & MATH500 & MBPP & HumanEval (unseen) \\
\midrule
SDAR-8B & 1.61 / 1.72\SZ & 1.60 / 1.65\SZ & 1.56 / 1.65\SZ & 1.60 / 1.61\SZ \\
SDAR-4B & 1.61 / 1.74\SZ & 1.60 / 1.69\SZ & 1.59 / 1.73\SZ & 1.51 / 1.55\SZ \\
SDAR-1.7B & 1.61 / 1.72\SZ & 1.49 / 1.58\SZ & 1.53 / 1.65\SZ & 1.52 / 1.56\SZ \\
LLaDA2.0-mini & 1.47 / 1.53\SZ & 1.51 / 1.57\SZ & 1.47 / 1.69\SZ & 1.14 / 1.18\SZ \\
\bottomrule
\end{tabular}}

\end{table}

\begin{table}[h]
\centering\scriptsize
\caption{Drafter D (general chat, as in Table~\ref{tab:main-sdar}) against drafter A (task prompts) on SDAR-8B b32
under the threshold decoder ($\tau{=}0.9$): \nfe{NFE per sample} (\tps{throughput ratio} over standard decoding), best in
bold, second best underlined. \SEEDN{sets_a_sdar8b}, \THRN{sets_a_sdar8b}, as Table~\ref{tab:main-sdar}.}
\label{tab:sets-a}
\setlength{\tabcolsep}{4pt}
\resizebox{\linewidth}{!}{\begin{tabular}{lcccccccc}
\toprule
\nfe{NFE}$\downarrow$ (\tps{tok/s}$\uparrow$) & \multicolumn{2}{c}{GSM8K} & \multicolumn{2}{c}{MBPP} & \multicolumn{2}{c}{HumanEval} & \multicolumn{2}{c}{MATH500} \\
\cmidrule(lr){2-3}\cmidrule(lr){4-5}\cmidrule(lr){6-7}\cmidrule(lr){8-9}
method & $D8$ & $D20$ & $D8$ & $D20$ & $D8$ & $D20$ & $D8$ & $D20$ \\
\midrule
\rowcolor[gray]{0.90}$T{=}0$ & \multicolumn{2}{c}{acc.\ 0.87--0.89} & \multicolumn{2}{c}{acc.\ 0.39--0.41} & \multicolumn{2}{c}{acc.\ 0.67--0.71} & \multicolumn{2}{c}{acc.\ 0.10--0.11} \\
no draft & \multicolumn{2}{c}{\nfe{58.8}$\downarrow$ (\tps{4.51}$\uparrow$)\SZ} & \multicolumn{2}{c}{\nfe{44.0}$\downarrow$ (\tps{5.89}$\uparrow$)\SZ} & \multicolumn{2}{c}{\nfe{43.4}$\downarrow$ (\tps{6.13}$\uparrow$)\SZ} & \multicolumn{2}{c}{\nfe{80.1}$\downarrow$ (\tps{3.34}$\uparrow$)\SZ} \\
\midrule
\rowcolor{oursrow}\sysSets{}-$\gamma$, general chat & \nfe{30.6}\,(\tps{7.30})\SZ & \nfe{30.2}\,(\tps{7.26})\SZ & \nfe{22.8}\,(\tps{9.28})\SZ & \nfe{22.4}\,(\tps{\underline{9.15}})\SZ & \nfe{\underline{22.7}}\,(\tps{\underline{9.85}})\SZ & \nfe{\underline{22.6}}\,(\tps{\underline{9.62}})\SZ & \nfe{\underline{42.3}}\,(\tps{\underline{5.24}})\SZ & \nfe{\underline{41.8}}\,(\tps{\underline{5.22}})\SZ \\
\rowcolor{oursrow}\sysSets{}-$\gamma$, exact, general chat & \nfe{32.5}\,(\tps{6.93})\SZ & \nfe{32.2}\,(\tps{6.86})\SZ & \nfe{23.8}\,(\tps{8.97})\SZ & \nfe{23.4}\,(\tps{8.72})\SZ & \nfe{24.0}\,(\tps{9.35})\SZ & \nfe{23.7}\,(\tps{9.22})\SZ & \nfe{45.1}\,(\tps{4.96})\SZ & \nfe{44.5}\,(\tps{4.92})\SZ \\
\rowcolor{oursrow}\sysSets{}-$\gamma$, task prompts & \nfe{\textbf{28.3}}\,(\tps{\textbf{7.83}})\SZ & \nfe{\textbf{27.9}}\,(\tps{\textbf{7.76}})\SZ & \nfe{\textbf{20.9}}\,(\tps{\textbf{9.79}})\SZ & \nfe{\textbf{20.2}}\,(\tps{\textbf{9.40}})\SZ & \nfe{\textbf{22.6}}\,(\tps{\textbf{9.90}})\SZ & \nfe{\textbf{22.2}}\,(\tps{\textbf{9.76}})\SZ & \nfe{\textbf{41.0}}\,(\tps{\textbf{5.44}})\SZ & \nfe{\textbf{40.3}}\,(\tps{\textbf{5.42}})\SZ \\
\rowcolor{oursrow}\sysSets{}-$\gamma$, exact, task prompts & \nfe{\underline{30.4}}\,(\tps{\underline{7.31}})\SZ & \nfe{\underline{29.8}}\,(\tps{\underline{7.31}})\SZ & \nfe{\underline{21.7}}\,(\tps{\underline{9.58}})\SZ & \nfe{\underline{21.0}}\,(\tps{9.05})\SZ & \nfe{23.6}\,(\tps{9.53})\SZ & \nfe{23.3}\,(\tps{9.42})\SZ & \nfe{43.9}\,(\tps{5.06})\SZ & \nfe{43.0}\,(\tps{5.13})\SZ \\
\specialrule{0.6pt}{2.5pt}{0pt}
\rowcolor[gray]{0.90}$T{=}1$ & \multicolumn{2}{c}{acc.\ 0.87--0.89} & \multicolumn{2}{c}{acc.\ 0.34--0.36} & \multicolumn{2}{c}{acc.\ 0.65--0.67} & \multicolumn{2}{c}{acc.\ 0.09--0.10} \\
no draft & \multicolumn{2}{c}{\nfe{61.5}$\downarrow$ (\tps{4.33}$\uparrow$)\SZ} & \multicolumn{2}{c}{\nfe{46.9}$\downarrow$ (\tps{5.54}$\uparrow$)\SZ} & \multicolumn{2}{c}{\nfe{44.4}$\downarrow$ (\tps{6.07}$\uparrow$)\SZ} & \multicolumn{2}{c}{\nfe{81.1}$\downarrow$ (\tps{3.31}$\uparrow$)\SZ} \\
\midrule
\rowcolor{oursrow}\sysSets{}-$\gamma$, general chat & \nfe{34.0}\,(\tps{6.44})\SZ & \nfe{33.5}\,(\tps{6.41})\SZ & \nfe{27.5}\,(\tps{7.53})\SZ & \nfe{27.2}\,(\tps{\textbf{7.25}})\SZ & \nfe{\underline{26.2}}\,(\tps{8.37})\SZ & \nfe{\underline{25.9}}\,(\tps{8.29})\SZ & \nfe{\underline{45.8}}\,(\tps{4.74})\SZ & \nfe{\underline{45.0}}\,(\tps{4.74})\SZ \\
\rowcolor{oursrow}\sysSets{}-$\gamma$, exact, general chat & \nfe{36.0}\,(\tps{6.21})\SZ & \nfe{35.4}\,(\tps{6.21})\SZ & \nfe{28.5}\,(\tps{7.40})\SZ & \nfe{28.1}\,(\tps{7.14})\SZ & \nfe{27.2}\,(\tps{8.30})\SZ & \nfe{26.8}\,(\tps{8.17})\SZ & \nfe{48.6}\,(\tps{4.58})\SZ & \nfe{47.7}\,(\tps{4.57})\SZ \\
\rowcolor{oursrow}\sysSets{}-$\gamma$, task prompts & \nfe{\textbf{32.0}}\,(\tps{\textbf{6.78}})\SZ & \nfe{\textbf{31.3}}\,(\tps{\textbf{6.78}})\SZ & \nfe{\textbf{26.0}}\,(\tps{\textbf{7.74}})\SZ & \nfe{\textbf{25.4}}\,(\tps{\underline{7.23}})\SZ & \nfe{\textbf{25.7}}\,(\tps{\textbf{8.61}})\SZ & \nfe{\textbf{25.4}}\,(\tps{\textbf{8.47}})\SZ & \nfe{\textbf{44.2}}\,(\tps{\textbf{4.93}})\SZ & \nfe{\textbf{43.8}}\,(\tps{\textbf{4.87}})\SZ \\
\rowcolor{oursrow}\sysSets{}-$\gamma$, exact, task prompts & \nfe{\underline{33.9}}\,(\tps{\underline{6.56}})\SZ & \nfe{\underline{33.1}}\,(\tps{\underline{6.60}})\SZ & \nfe{\underline{26.9}}\,(\tps{\underline{7.63}})\SZ & \nfe{\underline{26.1}}\,(\tps{7.18})\SZ & \nfe{26.9}\,(\tps{\underline{8.40}})\SZ & \nfe{26.4}\,(\tps{\underline{8.32}})\SZ & \nfe{46.8}\,(\tps{\underline{4.75}})\SZ & \nfe{46.3}\,(\tps{\underline{4.75}})\SZ \\
\bottomrule
\end{tabular}
}
\end{table}

\subsection{Transfer protocol}
\label{app:sets-transfer}
For each target we decode $99$ training prompts with the threshold rule (the first $25$ of the GSM8K, MATH, MBPP and
CodeAlpaca train splits, minus one on a decontamination list). They give the target's hidden-state basis, its token-embedding table when the tokenizer differs,
and the calibration trajectories. Each source drafter then refits only its input layers ($0.27$M of its $5.1$M
parameters) for $20$ epochs on these trajectories, with no validation-based selection, and keeps its source's floor; the
encoder and the heads stay as trained. LLaDA-8B and Dream-7B attend over the whole sequence, so, as for \sys{}'s rows,
drafted runs and the control use the approximate prefix cache of \citet{wu_fast-dllm_2026}. A transferred drafter keeps
$92$--$105\%$ of the native drafter's throughput (Table~\ref{tab:sets-transfer-full-d}). Its NFE is within $6\%$ of the
native drafter's on the SDAR targets, on either side, and $4$--$16\%$ higher on LLaDA2.0-mini, whose tokenizer differs.
On LLaDA-8B and Dream-7B, on which no drafter was trained, the four drafters land within $7\%$ of each other in NFE and reach
$2.5$--$7.3\times$ the throughput of standard decoding. Drafter A transfers less well, as a transferred A takes
$2$--$17\%$ more NFE than the native A on the SDAR targets and $4$--$33\%$ more on LLaDA2.0-mini.

\begin{table}[h]
\centering\scriptsize
\caption{Transfer matrix under threshold decoding, drafter D of each source ($\tau{=}0.9$, $T{=}0$, seed $0$;
$n{=}350$ per task, HumanEval $164$, LLaDA2.0-mini as target $100$). Cell: \sysSets{}-$\gamma$ (product rule, cap $8$) NFE
(throughput ratio over standard decoding). $^{*}$native drafter; $^{\dagger}$approximate prefix cache.}
\label{tab:sets-transfer-full-d}
\setlength{\tabcolsep}{3pt}
\resizebox{\linewidth}{!}{{\begin{tabular}{llcccc}
\toprule
target & drafter & GSM8K & MBPP & HumanEval & MATH500 \\
\midrule
\multirow{4}{*}{SDAR-8B} & SDAR-8B D$^{*}$ & 30.6\,(7.33)\SZ & 22.8\,(9.17)\SZ & 22.7\,(9.80)\SZ & 42.3\,(5.29)\SZ \\
 & SDAR-4B D & 31.3\,(7.14)\SZ & 22.2\,(9.23)\SZ & 23.7\,(9.39)\SZ & 43.9\,(5.10)\SZ \\
 & SDAR-1.7B D & 32.2\,(6.97)\SZ & 22.3\,(9.11)\SZ & 23.8\,(9.40)\SZ & 44.9\,(4.99)\SZ \\
 & mini D & 30.1\,(7.03)\SZ & 22.3\,(8.80)\SZ & 23.3\,(9.19)\SZ & 41.7\,(5.07)\SZ \\
\midrule
\multirow{4}{*}{SDAR-4B} & SDAR-8B D & 30.8\,(7.30)\SZ & 21.1\,(10.45)\SZ & 22.2\,(10.24)\SZ & 45.1\,(4.96)\SZ \\
 & SDAR-4B D$^{*}$ & 30.9\,(7.34)\SZ & 22.0\,(10.22)\SZ & 21.8\,(10.42)\SZ & 45.0\,(5.04)\SZ \\
 & SDAR-1.7B D & 31.7\,(7.08)\SZ & 21.2\,(10.40)\SZ & 22.1\,(10.12)\SZ & 46.4\,(4.87)\SZ \\
 & mini D & 30.6\,(7.03)\SZ & 21.3\,(10.13)\SZ & 22.0\,(9.95)\SZ & 43.7\,(4.91)\SZ \\
\midrule
\multirow{4}{*}{SDAR-1.7B} & SDAR-8B D & 43.9\,(5.00)\SZ & 21.8\,(9.97)\SZ & 27.3\,(8.16)\SZ & 52.7\,(4.17)\SZ \\
 & SDAR-4B D & 44.0\,(4.93)\SZ & 21.6\,(10.19)\SZ & 27.5\,(8.15)\SZ & 52.7\,(4.22)\SZ \\
 & SDAR-1.7B D$^{*}$ & 43.4\,(5.04)\SZ & 22.5\,(9.90)\SZ & 26.7\,(8.36)\SZ & 51.6\,(4.23)\SZ \\
 & mini D & 43.0\,(4.81)\SZ & 21.6\,(9.89)\SZ & 26.3\,(8.02)\SZ & 50.7\,(4.05)\SZ \\
\midrule
\multirow{4}{*}{LLaDA2.0-mini} & SDAR-8B D & 48.7\,(3.68)\SZ & 37.7\,(4.53)\SZ & 19.8\,(8.61)\SZ & 47.7\,(4.04)\SZ \\
 & SDAR-4B D & 48.3\,(3.73)\SZ & 39.0\,(4.34)\SZ & 19.7\,(8.84)\SZ & 48.0\,(4.09)\SZ \\
 & SDAR-1.7B D & 48.6\,(3.81)\SZ & 40.1\,(4.30)\SZ & 19.9\,(8.70)\SZ & 48.5\,(3.97)\SZ \\
 & mini D$^{*}$ & 44.0\,(3.96)\SZ & 34.7\,(4.68)\SZ & 19.0\,(8.41)\SZ & 43.7\,(4.19)\SZ \\
\midrule
\multirow{4}{*}{LLaDA-8B$^{\dagger}$} & SDAR-8B D & 40.6\,(3.44)\SZ & 22.4\,(5.97)\SZ & 36.6\,(3.92)\SZ & 48.7\,(2.64)\SZ \\
 & SDAR-4B D & 40.1\,(3.45)\SZ & 21.5\,(6.06)\SZ & 36.2\,(3.92)\SZ & 48.4\,(2.62)\SZ \\
 & SDAR-1.7B D & 41.2\,(3.41)\SZ & 21.7\,(5.95)\SZ & 36.5\,(3.88)\SZ & 48.9\,(2.71)\SZ \\
 & mini D & 38.7\,(3.07)\SZ & 21.6\,(5.89)\SZ & 34.8\,(3.66)\SZ & 46.7\,(2.52)\SZ \\
\midrule
\multirow{4}{*}{Dream-7B$^{\dagger}$} & SDAR-8B D & 30.2\,(4.82)\SZ & 19.2\,(7.29)\SZ & 23.6\,(6.40)\SZ & 48.5\,(2.94)\SZ \\
 & SDAR-4B D & 30.4\,(4.70)\SZ & 19.4\,(7.32)\SZ & 23.3\,(6.43)\SZ & 48.8\,(2.93)\SZ \\
 & SDAR-1.7B D & 30.4\,(4.92)\SZ & 18.9\,(7.30)\SZ & 23.4\,(6.40)\SZ & 49.2\,(2.79)\SZ \\
 & mini D & 30.0\,(4.46)\SZ & 19.0\,(7.29)\SZ & 23.1\,(6.14)\SZ & 49.0\,(2.69)\SZ \\
\bottomrule
\end{tabular}}
}
\end{table}

\section{Terms, common questions and released resources}
\label{app:readers}

This appendix gathers the paper's terms in one place, answers the questions a reader is likely to
bring to it, each with a pointer to the full account, and lists what is released.

\paragraph{Keywords.} Speculative decoding, diffusion language models (DLMs), block diffusion,
commit order, drafter, draft tree, coupled-noise verification, lossless acceleration, SDAR, LLaDA,
Dream, GSM8K, MBPP, HumanEval, MATH500, dynamic threshold decoding.

\subsection{Terms}

\begin{description}
\item[Stale conditionals.] The most recently available target conditionals at the masked positions
of a block. They are stale because hypothetical future commitments change the true target
distribution, but no additional target forward is performed while drafting
(Section~\ref{sec:method-drafter}).
\item[Commit order.] The sequence of positions the target commits next, one per forward under exact
decoding. The drafter predicts it as $q_\phi(s\mid y)$, a distribution over the target's next
committed position given a hypothetical state $y$ (Section~\ref{sec:method-drafter}).
\item[Draft tree.] Candidate block states grown by recursively querying the prefix-conditional
predictor, high-probability branches first, up to a budget of $D$ drafted states that the target
verifies in one batched forward (Section~\ref{sec:method-drafter}).
\item[Coupled-noise match.] Verification at $T>0$ in which the drafter applies the same Gumbel noise
as the target, so that, conditioned on the noise, the target transition is deterministic and
verification reduces to state equality (Section~\ref{sec:method-verification}).
\item[SPOT (Speculating Positions Order Trajectory).] The method: a $0.62$M-parameter order model,
a three-layer transformer encoder over the target's stale distributions at the masked positions,
predicts the target's next committed position; its predictions grow a draft tree that the target
verifies in one batched forward (Section~\ref{sec:method}).
\item[SPOT-$\gamma$.] SPOT with a prefix-mass floor $\gamma$ that discards candidates
below it, so that the realized verification budget adapts to drafter confidence; calibrated per target, $\gamma=0.02$ on the SDAR targets and $0$ on LLaDA2.0-mini
(Section~\ref{sec:method-drafter}, Appendix~\ref{app:gamma-calib}).
\item[NFE and throughput ratio.] NFE is the number of target forward passes per sample; the
throughput ratio is measured against the draft-free run of the same task, model and temperature on the same
machine; within $5\%$ we say that two methods match (Section~\ref{sec:exp-setup}, Appendix~\ref{app:accounting}).
\item[Dynamic threshold decoding.] The decoding rule of Fast-dLLM~\citep{wu_fast-dllm_2026}: each forward commits every
masked position whose confidence is at least $\tau$, or the most confident one if none is
(Section~\ref{sec:method-sets}).
\item[Commit set.] The positions one threshold step commits; a \sysSets{} draft node is a whole step.
\item[\sysSets{}-$\gamma$.] \sysSets{} unpadded, with a calibrated mass floor $\gamma$ and a cap of $D$ drafts,
drafting across blocks, draft refresh, and max-commit with the product rule; ``\sysSets{}-$\gamma$, exact'' keeps
state-equality acceptance. Its drafter is trained on general chat prompts.
\item[Max-commit.] A relaxed acceptance rule for threshold decoding that may accept a drafted step that
is not exactly the target's set; it changes the output.
\end{description}

\subsection{Common questions}

\paragraph{What problem does SPOT solve?}
SPOT accelerates diffusion language models without changing their output distribution. Exact
decoding fills in only one masked position per model pass, while faster multi-token schemes are
lossy. Speculative decoding is the standard lossless acceleration for autoregressive models
\citep{leviathan_fast_2023,chen_accelerating_2023}, but a DLM must choose both what to generate and
where. SPOT drafts several future commitments with a small learned model, and the target verifies
all of them in a single pass (Section~\ref{sec:intro}).

\paragraph{Why is commit order, not the token, the bottleneck?}
The tokens are largely known in advance; the order is not. The stale argmax matches the true token
in $92\%$ of first-block commits and $\ge 90\%$ thereafter, while the probability that the previous
round's top-confidence position commits next drops to $3\%$ by the ninth block. Live oracles show that
true positions are worth more than true tokens on every task (Section~\ref{sec:obs-oracle},
Figure~\ref{fig:oracle}).

\paragraph{How does SPOT draft and verify?}
SPOT reads the target's stale top-$32$ probabilities at each masked position, predicts the next
committed position, and repeats the prediction after each hypothetical commitment to grow a tree of
block states. The target evaluates the current and drafted states in one batched forward and
accepts the longest prefix consistent with its own trajectory. The drafter is trained entirely from
draft-free target trajectories and uses no target hidden states, token embeddings or text
(Section~\ref{sec:method}).

\paragraph{Is the output exact when the target samples?}
Yes, for targets with an exact prefix cache. At $T=0$, state-equality verification reproduces the draft-free target trajectory for any
drafter (Proposition~\ref{prop:greedy}). At $T>0$, the drafter and the target share per-position
Gumbel noise, following the coupling of \citet{daliri_coupling_2025}; conditioned on the noise, the
target transition is deterministic, and the drafted output equals the draft-free output for every draw of the
noise (Proposition~\ref{prop:match}). Draft quality therefore affects speed, never correctness
(Section~\ref{sec:method-verification}).

\paragraph{How does SPOT differ from Spiffy, SSD and BlockSpec?}
Spiffy calibrates a static draft graph offline \citep{agrawal_structuring_2026}, SSD orders
positions by confidence into a self-drafted chain \citep{gao_self_2025}, and a BlockSpec-style tree
ranks positions by the target's stale confidences \citep{pan_blockspec_2025}; SPOT learns the commit
order instead. Under one decoder and verifier, SPOT uses $16$--$36\%$ fewer target forwards than
Spiffy's calibrated graph at $1.04$--$1.41\times$ its throughput on SDAR-8B, in every task and budget
(Section~\ref{sec:exp-main}).

\paragraph{Which models and benchmarks does the paper use?}
The headline target is SDAR-8B \citep{cheng_sdar_2026}, a dense block-causal DLM at block width
$32$. Drafters are also trained for SDAR-4B, SDAR-1.7B and the mixture-of-experts LLaDA2.0-mini
\citep{bie_llada20_2025}, and transfer is measured on LLaDA-8B \citep{nie_large_2025} and Dream-7B
\citep{ye_dream_2025}. The benchmarks are GSM8K \citep{cobbe_training_2021}, MBPP
\citep{austin_program_2021}, HumanEval \citep{chen_evaluating_2021} and MATH500
\citep{hendrycks_measuring_2021}, with $n=350$ samples (all $164$ for HumanEval) at generation length $256$,
except where a caption says otherwise; Table~\ref{tab:main-sdar} and Table~\ref{tab:tsweep} are the mean of seeds $0$--$2$, the
other tables seed $0$ unless their caption says otherwise (Section~\ref{sec:exp-setup}).

\paragraph{What are the main results?}
On SDAR-8B, SPOT needs $3.2$--$4.7\times$ fewer target forwards than draft-free decoding at budget
$8$ and $4.1$--$6.8\times$ at budget $20$, at $2.5$--$3.4\times$ the throughput at budget $8$. At
$T=1$ it keeps $16$--$35\%$ fewer forwards than Spiffy at $1.08$--$1.37\times$ its throughput. On top of
dynamic threshold decoding, \sysSets{} reaches $5.2$--$9.9\times$ the throughput of standard decoding at $T=0$ and budget $8$
($7.9\times$ on average), against $3.3$--$6.1\times$ ($5.0\times$) for the threshold decoder alone. SPOT beats or
matches every baseline in every setting at both temperatures
(Sections~\ref{sec:exp-main}, \ref{sec:exp-temp} and~\ref{sec:exp-transfer}).

\paragraph{Does SPOT work beyond SDAR-8B?}
Yes, across six DLMs. The same recipe gives the lowest NFE and highest throughput in every
comparison on SDAR-4B; on SDAR-1.7B, and on LLaDA2.0-mini, which has a different tokenizer, it gives the lowest NFE in
every comparison and beats or matches every baseline in throughput. The LLaDA2.0-mini checkpoint, pointed at SDAR-8B
without retraining, costs $4$--$8\%$ more forwards and still beats every baseline calibrated
there. Checkpoints that never saw LLaDA-8B or Dream-7B cut their $256$ forwards to
$55$--$88$ and $51$--$80$, measured with an approximate prefix cache, so these rows measure
reach rather than losslessness (Section~\ref{sec:exp-transfer}).

\paragraph{What does \sysSets{} add, and is it exact?}
Faster DLM decoders commit every position whose confidence clears a bar $\tau$ (dynamic threshold decoding). There a
drafted step is a whole set of positions with a token for each, so \sysSets{} drafts commit sets and their tokens from
the target's stale distributions, its hidden state and the candidate tokens' embeddings, and grows a tree whose nodes
are whole steps. With state-equality acceptance (``exact'') the output is the threshold decoder's own, up to
floating-point batch effects. Max-commit with
the product rule also accepts a drafted step whose tokens are the target's top choices and whose confidences have a
geometric mean of at least $\tau$; it commits more per forward and changes the output, so its accuracy is reported
next to the exact rows (Section~\ref{sec:method-sets}, Appendix~\ref{app:sets}).
On SDAR-8B at $T=0$ and budget $8$, \sysSets{}-$\gamma$ takes $47$--$48\%$ fewer forwards than the threshold decoder at
$1.57$--$1.62\times$ its throughput; its exact variant still reaches $1.48$--$1.54\times$ (Table~\ref{tab:main-sdar}). Its drafter is
trained on general chat prompts. A drafter trained on the task prompts is faster on those tasks ($8.2\times$ on average on
SDAR-8B); the general chat drafter keeps $87$--$97\%$ of that drafter's speed-up on GSM8K, MATH500 and MBPP, and $97$--$99\%$
on HumanEval, which neither drafter saw (Tables~\ref{tab:sets-data} and~\ref{tab:sets-a}).

\paragraph{What are the limitations?}
The gains diminish when verification is costly or acceptance is low. Verification becomes the
bottleneck at wide blocks, and exact tree packing remains unresolved under bidirectional attention.
Gains also narrow on low-acceptance tasks, and at $T\ge 2$ the order model trained at $T=0$ is
outperformed by noise-aware ranking. \sysSets{} needs a short calibration on each new target, its larger drafter
($5.1$M parameters) matters more when the target forward is cheap, and its main rows use max-commit, which changes the
output (Section~\ref{sec:conclusion}).

\paragraph{Where are the code and the drafters?}
The code, the trained drafters and the instructions to reproduce the experiments are at
\url{https://github.com/Eliezer318/SPOT-dLLM} under the MIT license, and the drafters are mirrored
on Hugging Face at \url{https://huggingface.co/EliezerS/SPOT-dLLM} (Appendix~\ref{app:repro}).

\subsection{Where this work sits}

SPOT brings speculative decoding \citep{leviathan_fast_2023,chen_accelerating_2023} to diffusion
language models \citep{nie_large_2025,ye_dream_2025}, in particular to block-diffusion models that
denoise within a block and generate causally across blocks \citep{arriola_block_2025,cheng_sdar_2026}.
It is compared against speculative decoding for DLMs with self-drafting, draft graphs and confidence
trees \citep{gao_self_2025,agrawal_structuring_2026,pan_blockspec_2025}, and verifies with Gumbel
coupling \citep{daliri_coupling_2025}.

\end{document}